\documentclass{article}

\usepackage[final]{neurips_2022}
\usepackage{usual_macros}
\usepackage{tikz}
\usepackage[labelfont=bf]{caption}
\usetikzlibrary{arrows,shapes,positioning}
\usetikzlibrary{calc,decorations.markings}
\tikzstyle{block} = [rectangle, draw, fill=blue!10, rounded corners, text centered,     text width = 7em, minimum height = 2em]
\tikzstyle{line} = [draw, -latex']
\usepackage{comment}

\title{Gradient flow dynamics of shallow ReLU networks for square loss and orthogonal inputs}

\author{%
  Etienne Boursier \\
  TML, EPFL, Switzerland\\
  \texttt{etienne.boursier@epfl.ch} \\
   \And
Loucas Pillaud-Vivien \\
  TML, EPFL, Switzerland\\
  \texttt{loucas.pillaud-vivien@epfl.ch} \\
   \And
  Nicolas Flammarion \\
  TML, EPFL, Switzerland\\
 \texttt{nicolas.flammarion@epfl.ch}
}

\begin{document}

\maketitle

\setcounter{tocdepth}{3}

\doparttoc 
\faketableofcontents 


\begin{abstract}
    The training of neural networks by gradient descent methods is a cornerstone of the deep learning revolution. Yet, despite some recent progress, a complete theory explaining its success is still missing. This article presents, for orthogonal input vectors, a precise description of the gradient flow dynamics of training one-hidden layer ReLU neural networks for the mean squared error at small initialisation. In this setting, despite non-convexity, we show that the gradient flow converges to zero loss and characterise its implicit bias towards minimum variation norm. Furthermore, some interesting phenomena are highlighted: a quantitative description of the initial alignment phenomenon and a proof that the process follows a specific saddle to saddle dynamics.

\end{abstract}

\section{Introduction}\label{sec:intro}
Artificial neural networks are nowadays trained successfully to solve a large variety of learning tasks. However, a large number of fundamental questions surround their impressive success. Among them, the convergence to global minima of their non-convex training dynamics and their ability to generalise well despite fitting perfectly the dataset have challenged traditional machine learning belief. While a complete theory is still lacking, the machine learning community has recently come up with 
key steps that allow to tame the complexity of the problem: proving the convergence of gradient flow to zero loss~\citep{mei2018mean,chizat2018global,sirignano2020mean,rotskoff2022}, investigating the algorithmic selection of a specific global minimum, often referred as the \textit{implicit bias} of an algorithm~\citep{neyshabur2014search,zhang2021understanding}; while paying attention to the importance of the initialisation~\citep{woodworth2020kernel,chizat2019lazy}. The aim of this article is to analyse precisely these three points for regression problems. This is done in a specific setting:  for orthogonal inputs,
we provide a complete characterisation of the gradient flows dynamics of training one-hidden layer ReLU neural networks with the square loss at small initialisation. We show that this non-convex optimisation dynamics  
captures most of the complexity mentioned above
and thus could be a first step towards analysing more general setups.

\paragraph{Global convergence of training loss for neural networks.} 
Showing convergence of the gradient flow to a global minimum is  an open and important question.
Beyond the lazy regime (see next paragraph), only a few results were proven in the regression setting.
The most promising route might be the link with Wassertein gradient flows for infinite neural networks.
In that case, global convergence happens under mild conditions~\citep{chizat2018global,wojtowytsch2020convergence}. 
Other works focus on local convergence~\citep{zhou2021local, safran2021effects}, or general criteria that eventually fail to encompass practical setups~\citep{chatterjee2022convergence,chen2022feature}.
These latter works rest on Polyak-\L{}ojasiewicz inequalities that in fact cannot be satisfied through the whole process if the dynamics travels near saddle points~\citep{liu2022loss}, as empirically observed~\citep{dauphin2014identifying}.
%
On the contrary, the present paper proves global convergence without resorting to large overparameterisation, dealing carefully with saddles.
%
%
%
%

\paragraph{Feature learning and small initialisation.} The scale of initialisation plays an essential role in the behavior of the training dynamics. Indeed, an important example is that, at large initialisation, known as the \textit{lazy regime}~\citep{chizat2019lazy}, the neurons move relatively slightly implying that the dynamics is nearly convex and described by an effective kernel method with respect to the \textit{Neural Tangent Kernel}~\citep{jacot2018neural,allen2019learning,arora2019fine}. Instead, we are interested in another regime where the initialisation scale is small. This regime is known to be richer as it performs \textit{feature learning}~\citep{yang2021tensor} but is also more challenging to analyse as it follows a truly non-convex dynamics (see details in Section~\ref{sec:dynamics}).

\paragraph{Implicit bias of gradient methods training.} There are many global minima to the mean squared error, i.e. ReLU neural networks that perfectly interpolate the dataset. 
An important question is to understand which one is selected by the gradient flow for a given initialisation~\citep{neyshabur2014search}. 
For linear neural networks, this question has been answered thoroughly~\citep{arora2019fine, yun2021unifying,min2021on} with a discussion on the role of initialisation~\citep{woodworth2020kernel} and noise~\citep{pesme2021implicit}.
For non-linear activations such as ReLU, no clear implicit bias criteria have been ever exhibited for the square loss besides a conjecture of a quantisation effect~\citep{maennel2018gradient}.
Finally, note that in the classification setting, the favorable behavior of iterates going to infinity simplifies the analysis to prove implicit biases such as: max-margin for the $\ell_2$ norm in case of linear models~\citep{soudry2018implicit,ji2019implicit}, alignment of inner layers for linear neural networks~\citep{ji2018gradient} and max-margin for the variation norm induced by neural networks~\citep{kurkova2001variation} in the case of one-hidden layer neural networks~\citep{lyu2019gradient,chizat2020implicit}.\\
Beyond the convergence results, the implicit bias characterisation anticipates the generalisation properties of the returned estimate as discussed in \cref{subsec:discussion}.

\paragraph{Dynamics of training for neural network.}
In the regression case, the starting point governs where the flow converges. 
This observation suggests that a complete analysis of the trajectory may be required when one wants to understand the implicit bias in this case.  Such descriptions have been undertaken by \citet{maennel2018gradient}, who describe the initial alignment phase at small initialisation, and \citet{li2020towards,jacot2021deep} who conjecture that the dynamics travels from saddle to saddle. These papers provide intuitive content that we prove rigorously in the orthogonal setup.\\
Finally, closest to our work are the following results on the classification of orthogonally separable data~\citep{phuong2020inductive,wang2021convex} and linearly separable, symmetric data~\citep{lyu2021gradient}. The classification setup provides easier tools to analyse the problem: indeed, after the initial alignment phase, the network has already  perfectly classified the data points in these settings. From there, it is known that the training loss converges to zero and that the parameters direction is biased towards KKT points of the max-margin problem~\citep{lyu2019gradient,ji2020directional}. Such tools cannot be applied after the alignment phase for regression, and we resort to a refined analysis of the trajectory to show both global convergence and implicit bias. 
On the other hand, \citet{lyu2021gradient} require a precise description of the dynamics to ensure convergence towards specific KKT points of the max-margin problem. Yet, the analysis of the dynamics is simplified by their symmetry assumption: the trajectory does not go through intermediate saddles and all the labels are simultaneously fitted. On the contrary, the dynamics we describe travels near an intermediate saddle point which separates two distinct fitting phases. This behaviour largely complicates the analysis, besides being more representative of the saddle to saddle dynamics observed in general settings.


\vspace*{-0.2em}

\subsection{Main contributions}

\vspace*{-0.15em}
%
We make the following contributions.
\begin{itemize}
[itemsep=-0.1em, topsep=-0.05em,leftmargin=0.8cm]
    \item 
     We prove the convergence of the gradient flow towards a global minimum of the non-convex training loss for small enough initialisation and finite width. 
    \item We characterise the global optimum retrieved for infinitesimal initialisation as a minimum $\ell_2$ norm interpolator, which implies a minimum variation norm in terms of prediction function. 
     \item As important as the convergence result, the dynamics is portrayed in \cref{sec:dynamics}: we quantitatively detail its different phases (alignment and fitting) and show it follows a saddle to saddle dynamics.
\end{itemize}

\subsection{Notations}

We denote by $\iind{A}$ the function equal to $1$ if $A$ is true and $0$ otherwise. $\mathcal{U}(S)$ is the uniform distribution over the set $S$ and $\mathcal{N}(\mu, \Sigma)$ is a Gaussian of mean $\mu$ and covariance $\Sigma$. We denote $\nabla_\theta h_\theta(x)$ the gradient of $\theta \mapsto h_\theta(x)$ at fixed $x$. For any $n \in \N^*$, $\llbracket n \rrbracket$  denotes the tuple of integers between $1$ and $n$. 
The scalar product between $x,y \in \R^d$ is denoted by $\langle x,y \rangle$ and the Euclidean norm is denoted by $\|\cdot \|$ and called~$\ell_2$. $\S_{d-1}$ denotes the sphere of $\R^d$ for the Euclidean norm. $B(\theta,r)$ is the Euclidean ball of center $\theta$ and radius $r$. All the detailed proofs of the claimed results are deferred to the Appendix.

\section{Setup and preliminaries}
\label{sec:setup_preliminaries}

\subsection{One-hidden layer neural network and training loss}
\label{subsec:model_and_gradient_flow}

\paragraph{Model.} Let us fix an integer $n \in \mathbb{N}^*$ as well as input data $(x_1, \hdots, x_n) \in (\R^d)^n$ and outputs $(y_1, \hdots, y_n) \in \R^n$. We are interested in the minimisation of the mean squared error:
\begin{equation}
\label{eq:erm}
L(\theta):= \frac{1}{2 n } \sum_{k = 1}^n \left(h_\theta (x_k) -y_k\right)^2, \qquad \text{where} \ \, h_\theta (x) :=  \sum_{j=1}^m a_j \sigma(\langle w_j, x \rangle)
\end{equation}
is a one-hidden layer neural network of width $m$ defined with parameters $\theta = (a, W)\in \R^m \times \R^{m \times d}$. The vector $a \in \R^m$ stands for the weights of the last layer and  $W^\top  = \begin{bmatrix} w_1 \cdots\, w_m \end{bmatrix} \in \R^{d \times m}$, where each $w_j \in \R^{d}$ represents a hidden neuron. To encompass the effect of the bias, an additional component can be added to the inputs $x^\top \leftarrow [x^\top, 1]$ without changing our results.  Finally, the activation function $\sigma$ is the ReLU: $\sigma(x) \coloneqq\max \, \{0, x\}$.


We introduce here the main assumptions on the data inputs.
\begin{ass}
\label{ass:orthonormal_family}
The input points form an orthonormal family, i.e. $\forall k,k' \in \llbracket n \rrbracket$, $\langle  x_k, x_{k'} \rangle = \iind{k = k'}$.  
\end{ass}
\vspace*{-0.2cm}
The data are assumed to be normalized only for convenience---the real limitation being that they are pairwise orthogonal. 
This assumption is exhaustively discussed in \cref{subsec:discussion}. 
\begin{ass}
\label{ass:non_degeneracy_y}
For all $k \in \llbracket n \rrbracket$, $y_k \neq 0$ and $ \sum_{k \mid y_k > 0} y_k^2 \neq \sum_{k \mid y_k < 0} y_k^2$.  
\end{ass}
\vspace*{-0.2cm}
This assumption on the data output is mild, e.g. has zero Lebesgue measure, and only permits to exclude degenerate situations. 


\paragraph{Gradient flow.} 
As the limiting dynamics of the (stochastic) gradient descent with infinitesimal step-sizes~\citep{li2019stochastic}, we study the following gradient flow
\begin{equation}
\label{eq:gradient_flow_theta}
\frac{\dd \theta^t}{\dd t} = - \nabla L(\theta^t) = - \frac{1}{n} \sum_{k = 1}^n \left(h_{\theta^t} (x_k) -y_k\right) \nabla_\theta h_{\theta^t}(x_k),
\end{equation}
initialised at $\theta^0 := (a^0,W^0)$. Since the ReLU is not differentiable at $0$, the dynamics should be defined as a subgradient inclusion flow~\citep{bolte2010characterizations}. However, we show in \cref{app:subgradient_flows} that the \textit{only} ReLU subgradient that guarantees the existence of a global solution is $\sigma'(x)= \iind{x > 0}$.
Hence, we stick with this choice throughout the paper.
%
Another important difficulty of this non-differentiability is that Cauchy-Lipschitz theorem does not apply and uniqueness is not ensured. There have been attempts to define the solution of this Ordinary Diffential Equation (ODE) unequivocally~\citep{eberle2021existence} as well as ways to circumvent this difficulty by resorting to smooth activations or additional data assumptions~\citep{wojtowytsch2020convergence,chizat2020implicit}. Yet, we do not follow this line and demonstrate our results \textit{for all the gradient flows} satisfying \cref{eq:gradient_flow_theta}.

\subsection{Preliminary properties and initialisation}

Let us derive here some preliminary properties of the gradient flows. If we rewrite explicitly the dynamics of \cref{eq:gradient_flow_theta} on each layer separately, we have straightforwardly that for all $j \in \llbracket m \rrbracket$, 
\begin{equation}
\label{eq:ode_a_w}
\frac{\dd a^t_j}{\dd t} = \langle D^{\theta^t}_{j}, w_j^t \rangle \qquad \text{and} \qquad
\frac{\dd w^t_j}{\dd t} =  D^{\theta^t}_{j} a_j^t,
\end{equation}
where $  D^{\theta^t}_{j} \coloneqq - \frac{1}{ n } \sum_{k=1}^n \iind{{\langle w^t_j,\, x_k\rangle  > 0}}\left(h_{\theta^t} (x_k)-y_k\right) x_k$ is a vector of $\R^{d}$. This vector solely depends on $\theta^t$ through the prediction function $h_{\theta^t}$ and on the neuron $j$ through its activation vector $\mathsf{A}(w_j^t)$; where, for a vector $w \in \R^{d}$, $\mathsf{A}(w):= (\iind{\langle w,\, x_1\rangle  > 0},\, \hdots,\, \iind{\langle w,\, x_n\rangle  > 0})  \in \{0, 1\}^n$.
%
From~\cref{eq:ode_a_w}, we deduce the following balancedness property~\citep{arora2019fine}.
\begin{lem}
\label{lem:balance}
For all $t\geq0$ and all $j\in \llbracket m\rrbracket$, $\displaystyle (a_j^t)^2 - \|w_j^t\|^2 = (a_j^0)^2 - \|w_j^0\|^2$.
Assume furthermore that for all $j\in \llbracket m\rrbracket$, the initialisation is balanced and non-zero: $|a_j^0| = \|w_j^0\|>0$. Then $|a_j^t| = \|w_j^t\| > 0$ and letting $\s = \mathrm{sign}(a^0) \in \{1, -1\}^{m}$, for all $t\geq 0$, we have that $a_j^t = \s_j \|w_j^t\|$.
\end{lem}
%
%
Importantly, by Lemma~\ref{lem:balance}, the study of \cref{eq:ode_a_w} reduces to the hidden layer $W$ solely. We consider the following balanced initialisation:
\begin{equation}\label{eq:init}
    \theta^0 = (a^0,W^0) \quad \text{with} \quad \begin{cases}
    w_j^0 = \lambda \, g_j \text{ where } g_j \stackrel{\mathrm{i.i.d.}}{\sim} \mathcal{N}(0, I_{d}),\\
    a_j^0 = \s_j \|w_j^0\| \text{ where } \s_j \stackrel{\mathrm{i.i.d.}}{\sim} \mathcal{U}(\{-1,1\}).
    \end{cases}
\end{equation}
As already stated, we are interested in the regime where the initialisation scale $\lambda>0$ is small. 
We also introduce the following sets of neurons that are crucial in the fitting process
\begin{align}\label{eq:S+1}
    S_{+,1} &\coloneqq \left\{ j\in\llbracket m \rrbracket\mid \s_j = +1 \ \ \  \text{ and \ \ \ for all } k \text{ such that } y_k>0, \langle w_j^0, x_k \rangle \geq 0 \right\},   \\
    S_{-,1} &\coloneqq \left\{ j\in\llbracket m \rrbracket\mid \s_j = -1  \ \ \ \text{ and \ \ \  for all } k \text{ such that } y_k<0, \langle w_j^0, x_k \rangle \geq 0 \right\}.\label{eq:S-1}
\end{align}

\begin{ass}\label{ass:init}
The sets $S_{+,1}$ and $S_{-,1}$ are both non-empty.
\end{ass}
\cref{ass:init} states that there are some neurons in two given cones at initialisation. It holds with probability $1$ when the support of initialisation covers all directions and the width $m$ of the network goes to infinity. 
This is thus a weaker condition than the \textit{omni-directionality} of neurons at initialisation~\citep{wojtowytsch2020convergence}, which is instrumental to show convergence in the mean field regime~\citep{chizat2018global}. 
On the other hand, it is stronger than the alignment condition of~\citet{abbe2022initial}, which is known to be necessary for weak learning but might not lead to the implicit bias described in the next section.


\section{Convergence and implicit bias characterisation}
\label{sec:main_result}


\subsection{Main result}\label{subsec:main}

\cref{thm:main} below states our main result on the convergence and implicit bias of one-hidden layer ReLU networks for regression tasks with orthogonal data.

\begin{theo}\label{thm:main}
Under \cref{ass:orthonormal_family,ass:non_degeneracy_y,ass:init}, there exists $\lambda^*>0$ depending only on the data and the width such that, if $\lambda\leq \lambda^*$, the gradient flow initialised according to \cref{eq:init} converges almost surely to some $\theta_\lambda^\infty$ of zero training loss, i.e. $L(\theta_\lambda^\infty)=0$. Furthermore, there exists $\displaystyle \theta^* $ such that 
\begin{equation}
\label{eq:minimum_norm_convergence}
\lim_{\lambda \to 0} \lim_{t \to \infty} \theta^t = \theta^* \in \underset{L(\theta) = 0}{\mathrm{argmin}} \ \|\theta\|^2.
\end{equation}
\end{theo}
The significance of this result is thoroughly discussed in \cref{subsec:discussion}.
Note that a quantitative and non-asymptotic version of \cref{thm:main}, both in time and $\lambda$, is stated in \cref{lemma:theta*} (\cref{app:proof}). Roughly, it states that the dynamics has already nearly converged after a time of order $-\ln(\lambda)$ and then that the convergence happens at exponential speed. 
Note also that the neural network need not be overparametrised for the result to hold: the only sufficient and necessary requirement on the width $m$ stems from \cref{ass:init}. \looseness=-1

\paragraph{Sketch of proof.}
The proof of \cref{thm:main} rests on a precise description of the training dynamics, which is divided into four different phases. We here only sketch it at a very high level and a more thorough description, with quantitative intermediate lemmas, is given in \cref{sec:dynamics}. 

During the first phase, hidden neurons align to a few representative directions, while remaining close to $0$ in norm. In particular, all hidden neurons in $S_{+,1}$ (resp. $S_{-,1}$) align with some key vector $D_+$ (resp. $-D_-$) defined in \cref{subsec:addnotations}.
During the second phase, the neurons aligned with $D_+$ grow in norm, while staying aligned with $D_+$, until fitting all the positive labels of the dataset (up to some error scaling with $\lambda$). Meanwhile, all the other neurons stay idle.
Then similarly, the neurons aligned with $-D_-$ grow in norm during the third phase, until nearly fitting all the negative labels. Meanwhile, these neurons remain aligned with $-D_-$ and all other neurons remain idle. The precise description of these three phases is obtained by analysing the solutions of the limit ODEs when $\lambda = 0$. The approximation errors that occur from dealing with non-zero $\lambda$ are then carefully handled via Gr\"onwall comparison arguments. Due to the large time scales (of order $-\ln(\lambda)$), the error can propagate on such large time spans. Handling these error terms is the main challenge of our proof and remains intricate despite the orthogonality assumption. 

After these three phases (which last a time $-\ln(\lambda)/\|D_-\|$), the parameters vector is close to some minimal $\ell_2$-norm interpolator. From there we show, exploiting a local Polyak-\L{}ojasiewicz condition, that the dynamics converges at exponential speed to a global minimum close to this interpolator.

\subsection{Discussion}\label{subsec:discussion}

Even if the orthogonal setting we consider is quite restrictive, it carries several characteristics that may be generic, either because they have been observed empirically, shown in related contexts or simply conjectured. 
We discuss these important points below.

\paragraph{Convergence to zero loss.} \cref{thm:main} states that the gradient flow converges to zero loss. Such a result is simple to show when the loss satisfies a Polyak-\L{}ojasiewicz (PL) inequality~\citep{bolte2007lojasiewicz}: $\|\nabla L\|^2 \geq c L $ for $c>0$. However, here, as the dynamics travels near saddles, this inequality is not verified through all the process. Circumventing this global argument, it is yet possible to formulate a refined analysis and show convergence if the dynamics arrives in a region where a local PL stands with a large enough constant. This refined analysis, inspired by the recent work of~\citet{chatterjee2022convergence}\footnote{Note that the argument is certainly not new, but the cited article has the benefit of clearly presenting it.}, 
allows to characterise properly the last phase of the dynamics. We believe that this approach may help in showing convergence in other non-convex gradient flow/descent. 

\paragraph{On the implicit bias.}
Additionally, \cref{thm:main} states that the gradient flow at infinitesimally small initialisation \textit{selects} global minimisers with the smallest $\ell_2$ parameter norm. 
To our knowledge, this is the first characterisation of the implicit bias for regression with non-linear neural networks. 
Although it might not hold for some degenerate situations~\citep{vardi2021implicit}, we believe it to be true beyond the orthogonal case.

This regularisation is \textit{implicit}, meaning that this effect does not result from any explicit regularisation (e.g. weight decay) performed during training~\citep{shevchenko2021mean,parhi2022kinds}. This is only a consequence of the inner structure of the gradient flow and the scale of initialisation.


Furthermore, the implicit bias in parameter space can be translated in function space. 
Indeed, if we introduce formally the space of (infinite) neural networks, i.e. functions written as $f(x):= \int \sigma(\langle \theta, x \rangle) \dd \mu (\theta)$,  where $\mu$ is a signed finite measure on $\S_{d-1}$.
Then, we can define the \textit{variation norm},  $\|f\|_{\mathcal{F}_1}$, as the infimum of $|\mu|(\S_{d-1})$ over such representations~\citep{kurkova2001variation,bach2017breaking}. 
We have the following link between the two formulations
\begin{equation}
\label{eq:equivalence_ell2_F1}
\min_{L(\theta) = 0} \frac{1}{2}\|\theta\|_2^2 = \min_{L(f) = 0} \|f\|_{\mathcal{F}_1},
\end{equation}
with a slight abuse of notation when defining $L(f)$. Note that the result in terms of the $\ell_2$-norm of the parameters is strictly stronger than that of the $\mathcal{F}_1$-norm of the function~\citep{neyshabur2014search}.

Following \cref{eq:equivalence_ell2_F1}, note the striking parallel between the inductive bias of infinitesimally small initialisation for regression and that of the classification problem with the logistic loss as a max-margin problem with respect to the $\mathcal{F}_1$-norm~\citep{chizat2020implicit}. As already observed in the linear case~\citep{woodworth2020kernel}, in contrast with classification, infinitesimally small initialisation is instrumental in regression to be biased towards small $\mathcal{F}_1$-norm functions. 
The role of initialisation is illustrated empirically in \cref{app:expe}.

Finally, let us stress that we did not address the question of what functions solve \cref{eq:equivalence_ell2_F1}, nor the question of the generalisation implied by such a bias. Related works on the first point come from a functional description of norms related to $\mathcal{F}_1$~\citep{savarese2019infinite,ongie2019function,debarre2022sparsest}. 
For the generalisation properties of small $\mathcal{F}_1$ norm functions, we refer to~\citet{kurkova2001variation,bach2017breaking}. Importantly, we recall that the question of how well low $\mathcal{F}_1$-norm functions generalise depends heavily on the \textit{a priori} we have on the ground-truth~\citep{petrini2022learning}.

\paragraph{The initial alignment phenomenon.} 
An important characteristics of the loss landscape is that the origin is a saddle point. Hence, as the dynamics is initialised at small scale $\lambda$, the radial movement is slow and neurons move out of the saddle after time scale $- \ln \lambda$. Meanwhile, the tangential movement of the neurons rules the dynamics and aligns their directions towards specific vectors. This has been first explained by~\citet{maennel2018gradient} and referred as the \textit{quantisation} phenomenon, because neural networks weights collapse to a small finite number of directions. We emphasise that this phase happens generically when initialisation is near the origin and that this part of our analysis can be directly extended to the general (i.e. non-orthognal) case. \citet{phuong2020inductive,lyu2021gradient} analysed a similar early alignment for classification with specific data structures.

\paragraph{The saddle to saddle dynamics.} When initialising the dynamics of a gradient flow near a saddle point of the loss,  it is expected (but hard to prove generically) that the dynamics will alternate slow movements near saddles and rapid junctions between them. Such a behavior has been conjectured for linear neural networks~\citep{li2020towards,jacot2021deep} initialised near the origin. We precisely prove that such a phenomenon occurs: after initialisation, the dynamics visits one strict saddle. See \cref{sec:dynamics},~\cref{fact:saddle_to_saddle} for more details.

\paragraph{Limitations and possible relaxations.} As its main limitation, \cref{thm:main} assumes orthogonal data points $x_k$. The orthogonality assumption disentangles the analysis as the different phases, where either the neurons align towards some direction or grow in norm, are well separated in that case. More precisely, the neurons do not change in direction once they have a non-zero norm in the case of orthogonal data. 
This separation between alignment and norm growth does not hold in the general case, as observed empirically in \cref{app:expe}. Extending our result to more general data thus remains a major challenge and requires additional theoretical tools. 
Nonetheless, as it can be the case in high dimension, our analysis can easily be extended to nearly orthogonal data where $|\langle x_k, x_{k'} \rangle|\leq \delta$, with $\delta$ of order $\lambda$. If however $\delta$ is much larger than the initialisation scale, the dynamics is drastically different and becomes as hard as the general case to analyse. In \cref{app:expe}, we observe similar dynamics for high dimensional data, where the loss converges towards $0$, goes through an intermediate saddle point and the final solution is close to a $2$ neurons network.

A minor assumption is the balanced initialisation, i.e. $\|w_j^0\|=|a_j^0|$. If instead we initialise $a^0$ as a Gaussian scaling with $\lambda$, the initialisation would be  nearly balanced for small $\lambda$. This assumption is thus mostly used for simplicity and our analysis can be extended to unbalanced initialisations.

It is unclear whether our analysis can be extended to any homogeneous activation function. The training trajectory might indeed not be biased towards minimal $\ell_2$-norm for \textit{leaky ReLU} activations~\citep[][Theorem 6.2]{lyu2021gradient}. Contrary to some beliefs, it suggests that the $\ell_2$ implicit bias phenomenon does not occur for any homogeneous activation function, but might instead be specific to the ReLU. 

\paragraph{The overparameterisation regime.} \cref{ass:init} states a deterministic condition to guarantee convergence towards a minimal norm interpolator. This condition is not only sufficient, but also \textit{necessary} for implicit bias towards minimal $\ell_2$ norm.
For isotropic initialisations, the width $m$ needs to be exponential in the number of data points $n$ for \cref{ass:init} to hold with high probability. 

With a smaller (e.g. polynomial in $n$) number of neurons, \cref{ass:init} does not hold anymore. In that case, the training loss should still converge to $0$, but the $\ell_2$-norm of the parameters will not be minimal. More precisely, the estimated function will have more than two kinks. However, an adapted analysis might still show some sparsity in the number of kinks (and thus a weak bias) of the final solution. The training trajectory would then go through multiple saddle points (one saddle per kink).

\paragraph{Scale of initialisation.} The exact value of $\lambda^*$ is omitted for exposition's clarity. Roughly, it can be inferred from the analysis that $\lambda^*$ scales as $\frac{\Theta(1)}{\sqrt{m}}e^{-\Theta(n)}$. Interestingly, the $\frac{1}{\sqrt{m}}$ term is reminiscent of the mean field regime, which is known to induce implicit bias \citep{chizat2020implicit,lyu2021gradient}. On the other hand, the exponential dependency in $n$ is common in the implicit bias literature \citep{woodworth2020kernel}. For larger values of $\lambda$ (but still in the mean field regime), the parameters empirically seem to also converge towards a minimal norm interpolator. The analysis yet becomes more intricate and we do not observe any separation between Phase 2 and Phase 3, i.e. there is no intermediate saddle in the trajectory.


\section{Fine dynamics description: alignments and saddles}\label{sec:dynamics}

This section describes thoroughly the training dynamics of the gradient flow. 
It presents and discusses quantitative lemmas on the state of the neural network at the end of each different phase. 
In particular, mathematical formulations of the early alignment and saddle to saddle phenomena 
are provided.

\subsection{Additional notations}\label{subsec:addnotations}

First, we need to introduce additional notations for this section. We define vectors $D_+$ and $D_-$ that are the two directions towards which the neurons align
\begin{equation*}
    D_+ \coloneqq \frac{1}{n}\sum_{k\mid y_k>0} y_k x_k \qquad \text{and} \qquad D_- \coloneqq \frac{1}{n}\sum_{k\mid y_k<0} y_k x_k.
\end{equation*}
%
%
We also need to define $c\coloneqq\max_{j\in\llbracket m \rrbracket} \|w_j^0\|/\lambda$ and $r\coloneqq \|D_+\|/\|D_-\|$. \cref{ass:non_degeneracy_y} implies that $r\neq 1$, and by symmetry we can assume $r>1$ without any loss of generality.
%
We additionally fix constants $\lambda_*, \varepsilon>0$, small enough and depending only on the dataset and the width $m$.

\paragraph{Spherical coordinates.}
As radial and tangential movements are almost decoupled during the dynamics, it is natural to introduce the spherical coordinates of the neurons: for all $j\in\llbracket m \rrbracket$, denote $w_j = \e^{\rho_j} \cdot \w_j $, where $\rho_j = \ln \|w_j\| \in \R$ and  $\w_j = w_j/\|w_j\| \in \S_{d-1}$. In these adapted coordinates, the system of ODEs~\eqref{eq:ode_a_w} reduces to:
\begin{equation}
\label{eq:ode_norm_angle}
\frac{\dd \rho^t_j}{\dd t} = \s_j \langle D^{\theta^t}_{j}, \w_j^t \rangle \qquad \text{and} \qquad \frac{\dd \w^t_j}{\dd t} = \s_j \left(D^{\theta^t}_{j} - \langle D^{\theta^t}_{j}, \w_j^t \rangle \w_j^t \right).
\end{equation}

\subsection{Training dynamics}
This section precisely describes the phases of the dynamics, summarised in \cref{fig:frise}.
\begin{figure*}[htbp]
    \centering
\begin{tikzpicture}[xscale=1.375]\usetikzlibrary{patterns}
  \def\h{1}
  \def\b{0.725}
  \def\o{0.25}
  {\footnotesize	
  \draw[thick,blue!80!black,fill=blue!90!black,opacity=0.2]
    (0,0) rectangle (2,\h);
  \node at (1,\b) {Early alignment};
  \node at (1,\o) {(\cref{lemma:phase1})};
  \draw[thick,blue!80!black,fill=blue!90!black,opacity=0.2]
    (2,0) rectangle (4.5,\h);
    \node at (3.25,\b) {Fitting positive labels};
  \node at (3.25,\o) {(\cref{lemma:phase2informal})};
\draw[thick,blue!80!black,fill=blue!90!black,opacity=0.2]
    (8,0) rectangle (9.8,\h);
    \node at (8.9,\b) {Final convergence};
  \node at (8.9,\o) {(\cref{lemma:PL})};
  \fill[opacity=1,pattern=north east lines,pattern color=green!80!black]    (4.5,0) rectangle (6.85,\h);
  \draw[thick,blue!80!black,fill=blue!90!black,opacity=0.2]
    (4.5,0) rectangle (8,\h);
    \node at (6.25,\b) {Fitting negative labels};
  \node at (6.25,\o) {(\cref{lemma:phase3informal})};
  \draw[->,line width=2] (0,0) -- (10,0) node[below right] {\large $t$};
  
   \draw (0 cm,2pt) -- (0 cm,-2pt) node[anchor=north] {\tiny $0$};
   \draw (2 cm,2pt) -- (2 cm,-2pt) node[anchor=north] {\tiny $\frac{-\varepsilon \ln \lambda}{\|D_-\|}$};
   \draw (4.5 cm,2pt) -- (4.5 cm,-2pt) node[anchor=north] {\tiny $\frac{-(1+3\varepsilon) \ln \lambda}{\|D_+\|}$};
   \draw[color=green!50!blue] (6.85 cm,2pt) -- (6.85 cm,-2pt) node[anchor=north,color=green!60!blue]  {\tiny $\frac{-(1-\varepsilon) \ln \lambda}{\|D_-\|}$};
   \draw (8 cm,2pt) -- (8 cm,-2pt) node[anchor=north]  {\tiny $\frac{-(1+ 3r\varepsilon) \ln \lambda}{\|D_-\|}$};
   \draw (9.8 cm,2pt) -- (9.8 cm,-2pt) node[anchor=north]  {\tiny $\infty$};
   \draw[<->,line width=1.5, color=green!50!blue] (4.5 cm, 1.2*\h cm) -- (6.85 cm, 1.2*\h cm) node[above, align= left, color=green!60!blue]  {\hspace{-3.25cm} \scriptsize Intermediate saddle (\cref{fact:saddle_to_saddle})};
   }
\end{tikzpicture}
    \vspace{-1.5em}
    \caption{Timeline of the training dynamics.}
    \label{fig:frise}
      \vspace{.6em}
\end{figure*}
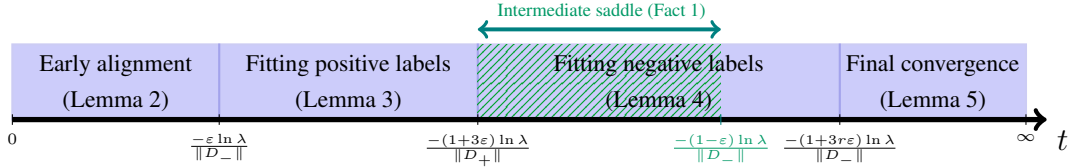

\paragraph{Neuron alignment phase.} During the first phase, all the neurons remain small in norm, while moving tangentially (i.e. in directions). The neurons align according to several key directions: an initial clustering of neurons' directions happens in this early phase, as observed by \citet{maennel2018gradient}.
As the neurons have small norm, $h_{\theta^t}\approx 0$ for this phase and \cref{eq:ode_norm_angle} approximates
\begin{equation}
\label{eq:tangential_movement}
\frac{\dd \w^t_j}{\dd t} \approx \s_j \left(D^{0}_{j} - \langle D^{0}_{j}, \w_j^t \rangle \w_j^t \right).
\end{equation}
This ODE corresponds to the descent/ascent gradient flow (depending on the sign of $\s_j$) on the sphere with objective $\langle D^{0}_{j}, \w_j \rangle$. All neurons end up minimizing or maximizing their scalar product with $D^{0}_{j}$, which only depends on the activation of $\w_j$. As a consequence, neurons with similar activations align towards the same vector, leading to some quantisation of the neurons' directions. This alignment happens in a relatively short time, so that the neurons cannot largely grow in norm. 
\cref{lemma:phase1} below quantifies this effect for neurons in $S_{+1,}$ and $S_{-,1}$, which are crucial to the training dynamics. Since all other neurons remain small in norm during the whole process, we do not focus on their direction.
%
%
\begin{lem}[First phase]\label{lemma:phase1}
For $\lambda \leq \lambda^*$, we have the following inequalities for $t_1=\frac{-\varepsilon \ln(\lambda)}{\|D_-\|}$:
\begin{enumerate}[itemsep=-0.5em, topsep=-1em, label=(\roman*)]
    \item neurons in $S_{+,1}$ are aligned with $D_+$: \quad  \phantom{+}$\forall j \in S_{+,1}, \langle \w_j^{t_1},D_+ \rangle \geq (1-2\lambda^{\varepsilon})\|D_+\|$,
    \item neurons in $S_{-,1}$ are aligned with $-D_-$: \quad  $\forall j \in S_{-,1}, \langle \w_j^{t_1},-D_- \rangle \geq (1-2\lambda^{\varepsilon})\|D_-\|$,
        \item all neurons have small norm: \quad $\forall j \in \llbracket m \rrbracket, \|w_j^{t_1}\|\leq 2c\lambda^{1-r\varepsilon}$. 
\end{enumerate}
\end{lem}

\paragraph{Fitting positive labels.} During the second phase, the norm of the neurons in $S_{+,1}$ (which are aligned with $D_+$) grows until fitting all positive labels. Meanwhile, all the other neurons do not move significantly. The key approximate ODE of this phase is given for $u_+(t) := \sum_{j\in S_{+,1}} \|w_j^t\|^2$ by
\begin{equation*}
    \frac{\dd u_+(t)}{\dd t} \approx 2\|D_+\|\left(1-\frac{u_+(t)}{n\|D_+\|}\right)u_+(t).
\end{equation*}
This equation implies that $u_+(t)$, the sum of the squared norms of neurons in $S_{+,1}$, eventually converges to $n\|D_+\|$ within a time ${-\ln(\lambda)}/{\|D_+\|}$ . Meanwhile, it needs to be shown that these neurons remain aligned with $D_+$ and that the other neurons remain small in norm. This fine control is technical and relies on the orthogonality assumption. If data were not orthogonal, neurons could indeed realign while growing in norm as illustrated by \cref{fig:nonortho3} in \cref{app:expe}.
\cref{lemma:phase2informal} below describes the state of the network at the end of the second phase.
\begin{lem}[Second phase]\label{lemma:phase2informal}
If $\lambda \leq \lambda^*$, then for some time $t_2\leq -\frac{1+3\varepsilon}{\|D_+\|}\ln(\lambda)$:
\begin{enumerate}[itemsep=-0.5em, topsep=-1em, label=(\roman*)]
    \item neurons in $S_{+,1}$ are aligned with $D_+$: \quad $\forall j \in S_{+,1}, \langle \w_j^{t_2},D_+ \rangle \geq \|D_+\| - \lambda^{\frac{\varepsilon}{2}}$,
    \item neurons in $S_{+,1}$ have a large norm: \quad $\sum_{j\in S_{+,1}} \|w_j^{t_2}\|^2=n\|D_+\|-\lambda^{\frac{\varepsilon}{5}}$, 
        \item other neurons have small norm: \quad $\forall j \in \llbracket m \rrbracket\setminus S_{+,1}, \|w_j^{t_2}\|\leq 2c\lambda^{\varepsilon}$.
\end{enumerate}
\end{lem}
These three points directly imply that the loss is of order $\lambda^{\frac{\varepsilon}{5}}$ on the positive labels at time $t_2$. 

\paragraph{Saddle to saddle dynamics.} As explained above, the positive labels are almost fitted by the action of the neurons belonging to $S_{+,1}$ at the end of the second phase, whereas the other neurons still have infinitesimally small norm. At this point, the dynamics has reached the vicinity of a strict saddle point and requires a long time to escape it. The analysis actually leads to the following fact:
\begin{fact}
\label{fact:saddle_to_saddle}
There exists a (strict) saddle point $\theta_S\neq 0$ of $L$ such that if $\lambda\leq\lambda^*$:
\begin{equation*}
    \hspace*{-1cm}\forall t\in\left[-\frac{1+3\varepsilon}{\|D_+\|}\ln(\lambda), -\frac{1-\varepsilon}{\|D_-\|}\ln(\lambda)\right], \quad \text{we have }\  \|\theta^t - \theta_S\| \leq \lambda^{\frac{\varepsilon}{5}}.
\end{equation*}
\end{fact}
The training trajectory thus starts at the saddle point $0$ and passes through a second non-trivial saddle point at the end of the second phase. This lemma illustrates the phenomenon of \textit{saddle to saddle dynamics}  discussed in \cref{subsec:discussion} and conjectured for linear models by \citet{li2020towards,jacot2021deep}.
This intermediate saddle point is escaped when the norms of the neurons in $S_{-,1}$ have significantly grown (i.e. become non-zero), which happens during a third phase described below.

\paragraph{Fitting negative labels.}
The norm of the neurons in $S_{-,1}$ (which are aligned with $-D_-$) grows until fitting all negative labels during the third phase. Meanwhile, all other neurons do not move significantly.
The additional difficulty in the analysis of this phase compared to the second one is that of controlling the possible movements of neurons in $S_{+,1}$. Their norm is indeed large during the whole phase, but they do not change consequently, because the positive labels are nearly perfectly fitted.  
\looseness=-1
\begin{lem}[Third phase]\label{lemma:phase3informal}
If $\lambda\leq\lambda^*$, then for some time $t_3\leq -\frac{1+3r\varepsilon}{\|D_-\|}\ln(\lambda)$:
\begin{enumerate}[itemsep=-0.5em, topsep=-1em, label=(\roman*)]
    \item neurons in $S_{-,1}$ are aligned with $-D_-$: \quad $\forall j \in S_{-,1}, \langle \w_j^{t_3},-D_- \rangle \geq \|D_-\| - \lambda^{\frac{\varepsilon}{14}}$,
    \item neurons in $S_{-,1}$ have a large norm: \quad $\sum_{j\in S_{-,1}} \|w_j^{t_3}\|^2=n\|D_-\|-\lambda^{\frac{\varepsilon}{29}}$,
    \item neurons in $S_{+,1}$ did not move since phase $2$: \quad $\forall j\in S_{+,1}, \|w_j^{t_2}-w_j^{t_3}\|\leq \lambda^{\frac{\varepsilon}{15}}$,
        \item other neurons have small norm: \quad $\forall j \in \llbracket m \rrbracket\setminus \left(S_{+,1}\cup S_{-,1}\right), \|w_j^{t_3}\|\leq 3c\lambda^{\varepsilon}$.
\end{enumerate}
\end{lem}
Thanks to the orthogonality assumption, the set of minimal $\ell_2$-norm interpolators can be exactly described by \cref{prop:KKT} in \cref{app:KKT}. The minimal interpolators are actually \textit{equivalent} to a neural network of width $2$: the first hidden neuron is collinear with $D_+$ and the second one is collinear with $-D_-$.
\cref{lemma:phase3informal} then ensures at the end of the third phase that the parameter vector is $\lambda^{\frac{\varepsilon}{29}}$-close to an interpolator $\theta^*$ of minimal $\ell_2$-norm, that does not depend on $\lambda$ (see \cref{lemma:theta*}). It remains to show that the training trajectory converges to a point close to this interpolator at infinity.

\paragraph{Convergence phase.} To show this final convergence, we use a local PL condition given by \cref{lemma:PL}.
\begin{lem}[Local PL condition]\label{lemma:PL}
For $\lambda\leq\lambda^*$, we have the following lower bound on the PL constant
\begin{equation*}
\inf_{\theta \in B(\theta^*,\ \lambda^{\frac{\varepsilon}{240}}) \cap \Theta} \frac{\|\nabla L (\theta)\|^2}{L(\theta)} \geq \|D_-\|,
\end{equation*}
where $\Theta$ is the set of parameters verifying the balancedness property.
\end{lem}
Adapting arguments from the recent work by \citet{chatterjee2022convergence}, this implies that the training trajectory converges to an interpolator and stays in the aforementioned ball. It thus converges exponentially at a rate $\|D_-\|$ to a point close to a minimal norm interpolator, and the distance to this point goes to $0$ when $\lambda$ goes to $0$, hence implying \cref{thm:main}. This exponential rate is only asymptotical: the dynamics still require a large time $-{\ln(\lambda)}/{\|D_-\|}$ to escape the two first saddles.

\begin{rk}[Local PL in the ReLU case]
Note that, strictly speaking, Theorem 1.1 of \citet{chatterjee2022convergence} cannot be applied directly to our setup because (i) the infimum is taken on the intersection of a ball and the set $\Theta$ of balanced weights, (ii) the ReLU and hence the loss are not $C^2$ as required. We thank Spencer Frei for bringing us this point after the paper acceptance.
Yet, the reformulation of \citet{chatterjee2022convergence} in our case is straightforward and does not provide any more insight. To avoid confusion, we decide to omit the proof here and postpone its writing in an independent note for later. This will also be the occasion to put emphasis on the fact that, in some cases, such a result can be applied in the ReLU case without much difficulty.
\end{rk}

\section{Experiments}\label{sec:expe}

This section confirms empirically the dynamics described in \cref{sec:dynamics} on an orthogonal toy example. 
The code and animated versions of the figures are available in \url{github.com/eboursier/GFdynamics}.
Additional experiments can be found in \cref{app:expe}; they illustrate the necessity of small initialisation for implicit bias and present similar experiments on non-orthogonal toy data. For the latter, we observe some similar training phenomena, but major differences appearing in the dynamics highlight the difficulty of dealing with non-orthogonality.

We consider the following two-point dataset: $x_1,y_1 = (-0.5, 1), -1$ and $x_2,y_2 = (2,1), 1$. It corresponds to unidimensional data with a second $1$ coordinate for the bias term. We choose unidimensional data for a simpler visualisation. However, it restricts the number of observations to $n=2$ to maintain orthogonality. Also, the inputs' norms are not $1$ here, but we recall that our analysis is not specific to this case.
The width of the neural network is $m=60$. We choose a balanced initialisation at scale 
$\lambda=10^{-6}/\sqrt{m}$. We then run gradient descent with a step size $10^{-3}$ to approximate the gradient flow trajectory.

   \begin{figure}[htbp]
        \centering
        \begin{subfigure}[b]{0.45\textwidth}
            \centering
            \includegraphics[trim=10pt 0pt 10pt 30pt, clip, width=\textwidth]{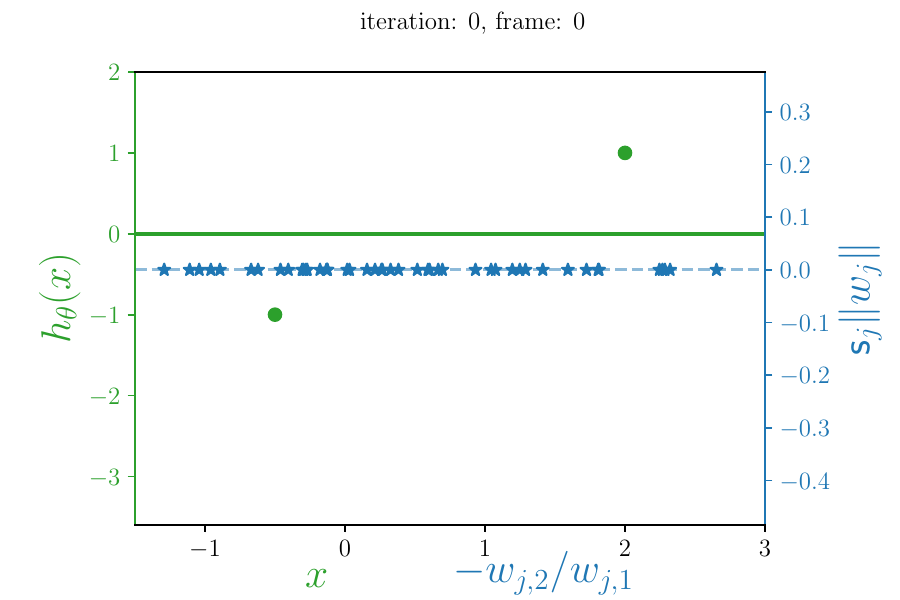}
            \caption{\label{fig:ortho0}{\small Initialisation (Iteration $0$)}}  
        \end{subfigure}
        \hfill
        \begin{subfigure}[b]{0.45\textwidth}  
            \centering 
            \includegraphics[trim=10pt 0pt 10pt 30pt, clip, width=\textwidth]{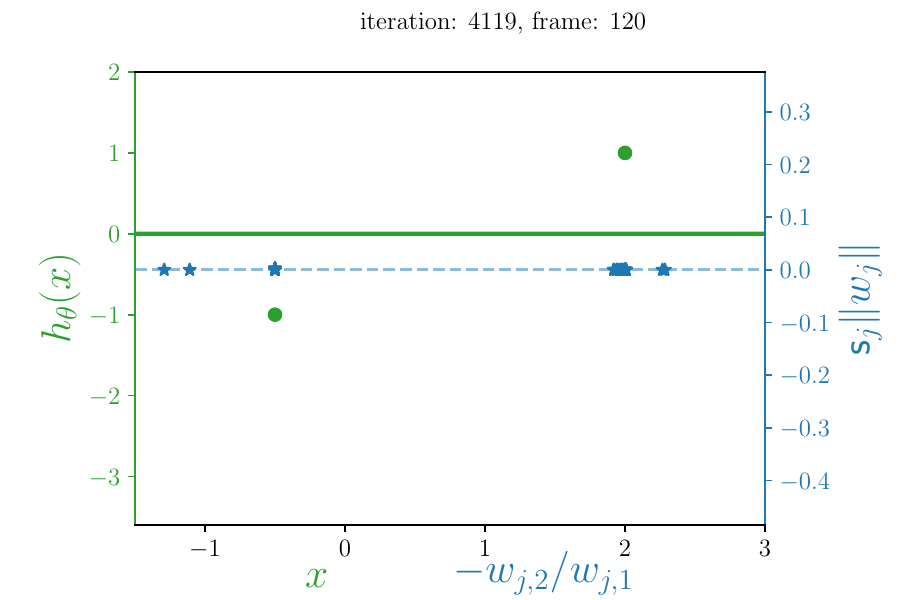}
            \caption{\label{fig:ortho1}{\small End of first phase (Iteration $4000$)}}    
        \end{subfigure}
        \vskip\baselineskip
        \begin{subfigure}[b]{0.45\textwidth}   
            \centering 
            \includegraphics[trim=10pt 0pt 10pt 30pt, clip, width=\textwidth]{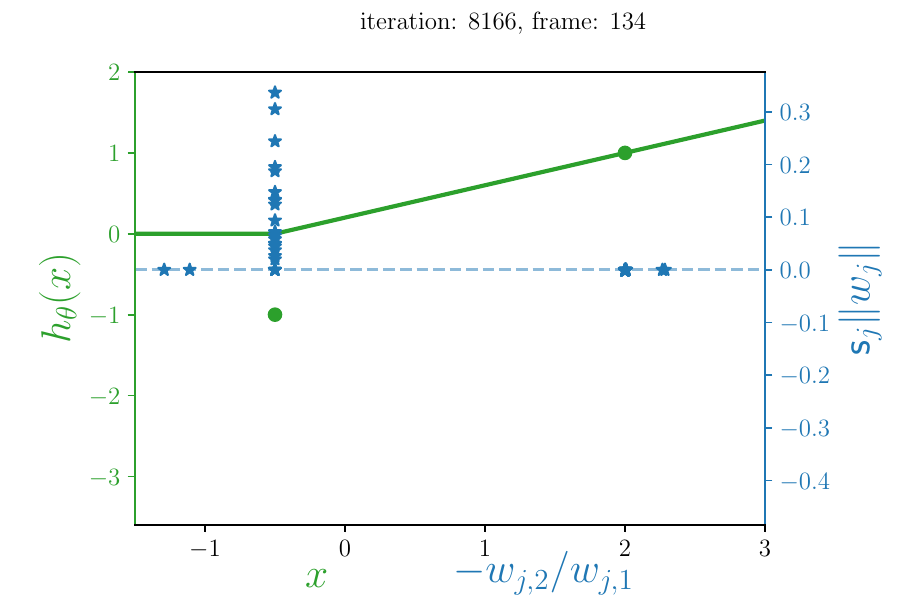}
            \caption{\label{fig:ortho2}{\small End of second phase (Iteration $8000$)}}    
        \end{subfigure}
        \hfill
        \begin{subfigure}[b]{0.45\textwidth}   
            \centering 
            \includegraphics[trim=10pt 0pt 10pt 30pt, clip, width=\textwidth]{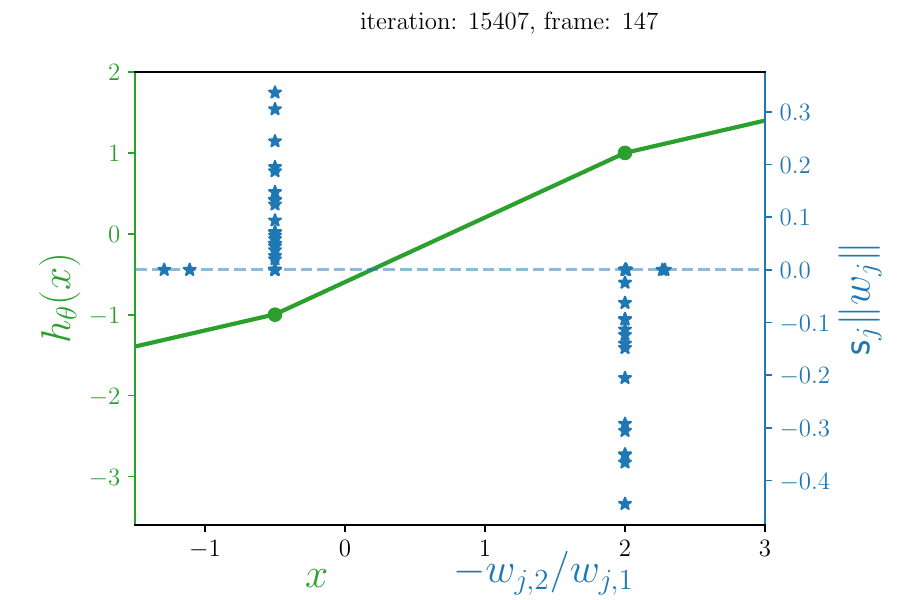}
            \caption{\label{fig:ortho3}{\small End of third phase (Iteration $15000$)}}    
        \end{subfigure}
        \caption{\label{fig:dynamics} State of training at different stages. The green dots correspond to the data, while the green line is the estimated function $h_\theta$. Each blue star represents a neuron $w_j$: its $x$-axis value is given by $-w_{j,2}/w_{j,1}$, which coincides with the position of the kink of its associated ReLU; its $y$-axis value is given by $\s_j\|w_j\|$, which we recall is the associated value of the output layer.} 
              \vspace{.6em}
    \end{figure}
    
    \cref{fig:dynamics} shows the training dynamics on this example. In particular, the state of the network is shown at different steps. In \cref{fig:ortho0}, all the neurons are close to $0$ at initialisation. 
    \cref{fig:ortho1} shows the end of the first phase, where the neurons are aligned towards two key directions.
    After the second phase, shown in \cref{fig:ortho2}, all the neurons aligned with $D_+$ have grown in norm and the positive label is perfectly fitted. 
    Similarly at the end of third phase in \cref{fig:ortho3}, all neurons aligned with $-D_-$ have grown in norm and the negative label is fitted.
    
    At the end of training, the loss is $0$ and the estimated function is \textit{simple}. In particular, it only has two kinks, which illustrates the sparsity induced by the implicit bias. Also, the final estimated function might be counter-intuitive. Previous works on implicit bias indeed conjectured that the learned estimator is linear if the data can be linearly fitted~\citep{kalimeris2019sgd,lyu2021gradient}. However, the learned function in \cref{fig:ortho3} has a smaller $\mathcal{F}_1$-norm than the linear interpolator. 

    \begin{minipage}{0.51\linewidth}
    \cref{fig:saddle} shows the evolution of the loss during training. The saddle to saddle dynamics is well observed here: the parameters vector starts from the $0$ saddle point at initialisation and needs $5000$ iterations to leave this first saddle. A second saddle is then encountered at the end of the second phase and the trajectory only leaves this saddle around iteration $11000$, once the norm of the neurons in $S_{-,1}$ start being significant during the third phase. All these different experiments confirm \cref{thm:main} and the precise dynamics described in \cref{sec:dynamics}. Moreover, such training phenomena are not specific to the orthogonal data case, as observed in \cref{app:expe}.
    \end{minipage}\hspace{10pt}
    \begin{minipage}{0.45\linewidth}
    \begin{figure}[H]
        \centering
        \includegraphics[trim=5pt 15pt 5pt 15pt, clip, width=\textwidth]{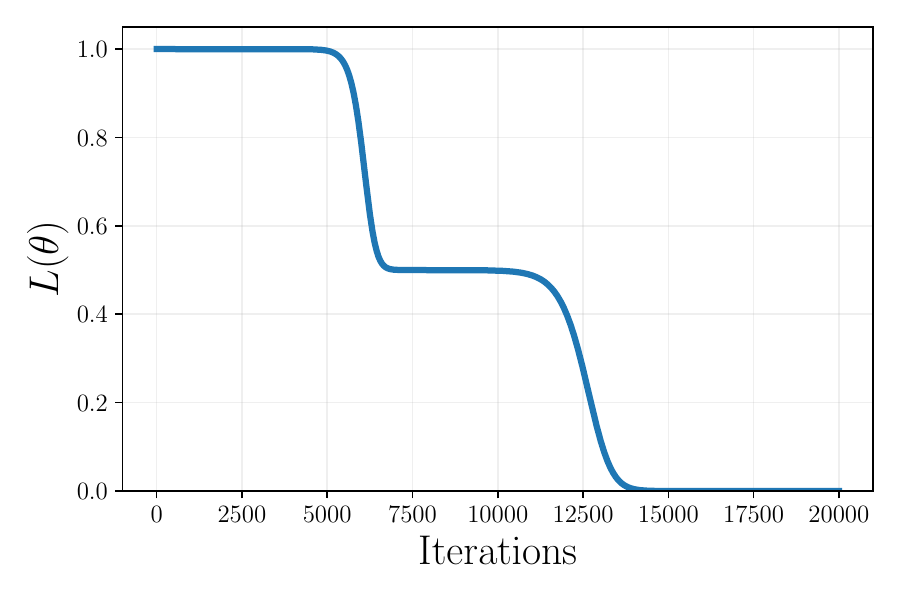}
        \caption{Evolution of the training loss.}
        \label{fig:saddle}
    \end{figure}
    \end{minipage}
\section{Conclusion and perspectives}

We have shown that the training of non-linear neural networks on orthogonal data presents a rich dynamics with a small and omnidirectional intialisation. Convergence holds generically despite a truly non-convex landscape and the limit enjoys an implicit bias as a minimum $\ell_2$ parameter norm. Obviously, removing the orthogonal assumption on the inputs, while keeping a fine level of description is a major, but difficult, perspective for future work. 
Another key point to better understand the good generalisation of neural networks is to analyse the properties of the functions solving the minimum variation norm problem stated in \cref{eq:equivalence_ell2_F1}. 

%
%
%
%
%
\clearpage
%
%
%
%
%
%


\bibliographystyle{plainnat}
\bibliography{neurips_2022}

\clearpage

\appendix

\addcontentsline{toc}{section}{Appendix} 
\part{Appendix} 
\parttoc 

\section{Additional experiments}\label{app:expe}

\subsection{Unidimensional data}

This section presents additional experiments in the general setting where data are non-orthogonal. To be able to visualize the results, similarly to \cref{sec:expe}, we consider unidimensional data with a bias term, but with $5$ data points. Here again, we consider a neural network of width $m=60$ and run gradient descent with step size $10^{-3}$.

   \begin{figure}[htbp]
        \centering
        \begin{subfigure}[b]{0.475\textwidth}
            \centering
            \includegraphics[trim=10pt 0pt 10pt 30pt, clip, width=\textwidth]{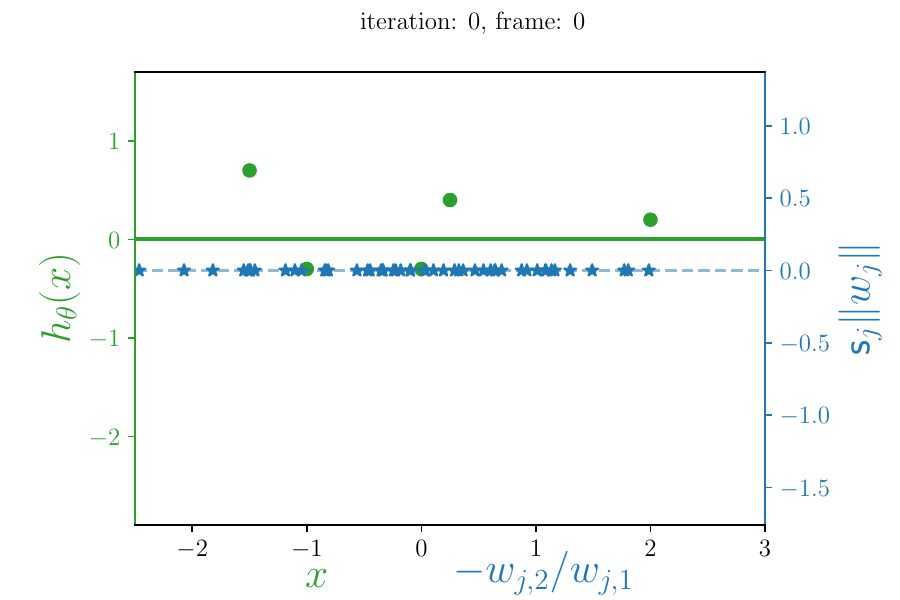}
            \caption{\label{fig:nonortho0}{\small Initialisation (Iteration $0$)}}  
        \end{subfigure}
        \hfill
        \begin{subfigure}[b]{0.475\textwidth}  
            \centering 
            \includegraphics[trim=10pt 0pt 10pt 30pt, clip, width=\textwidth]{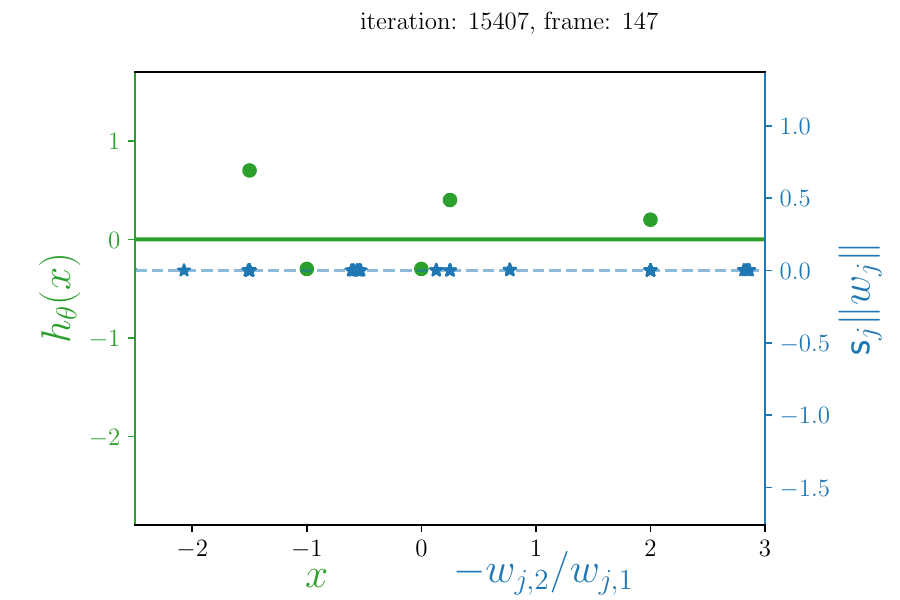}
            \caption{\label{fig:nonortho1}{\small End of early alignment (Iteration $15000$)}}    
        \end{subfigure}
        \vskip\baselineskip
        \begin{subfigure}[b]{0.475\textwidth}   
            \centering 
            \includegraphics[trim=10pt 0pt 10pt 30pt, clip, width=\textwidth]{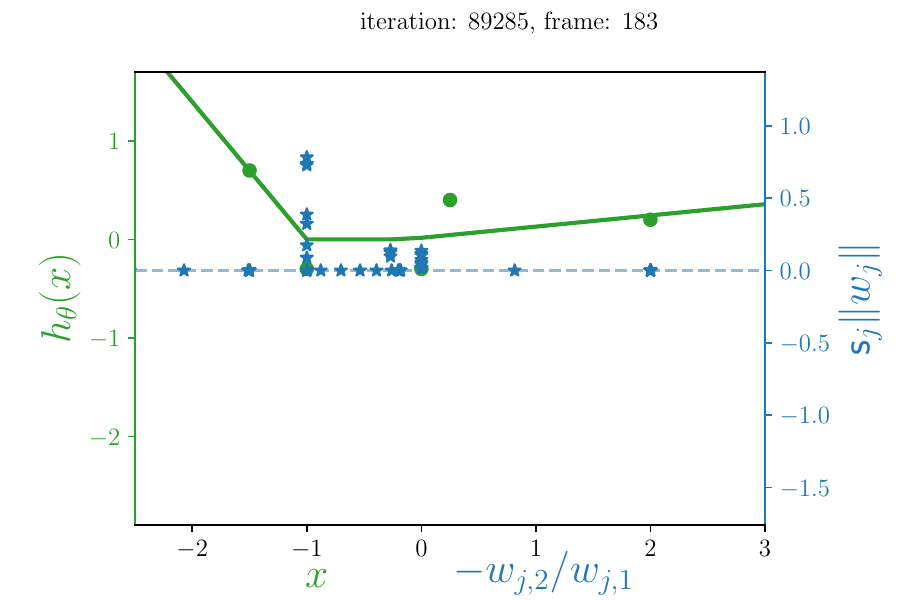}
            \caption{\label{fig:nonortho2}{\small Intermediate saddle (Iteration $94000$)}}    
        \end{subfigure}
        \hfill
        \begin{subfigure}[b]{0.475\textwidth}   
            \centering 
            \includegraphics[trim=10pt 0pt 10pt 30pt, clip, width=\textwidth]{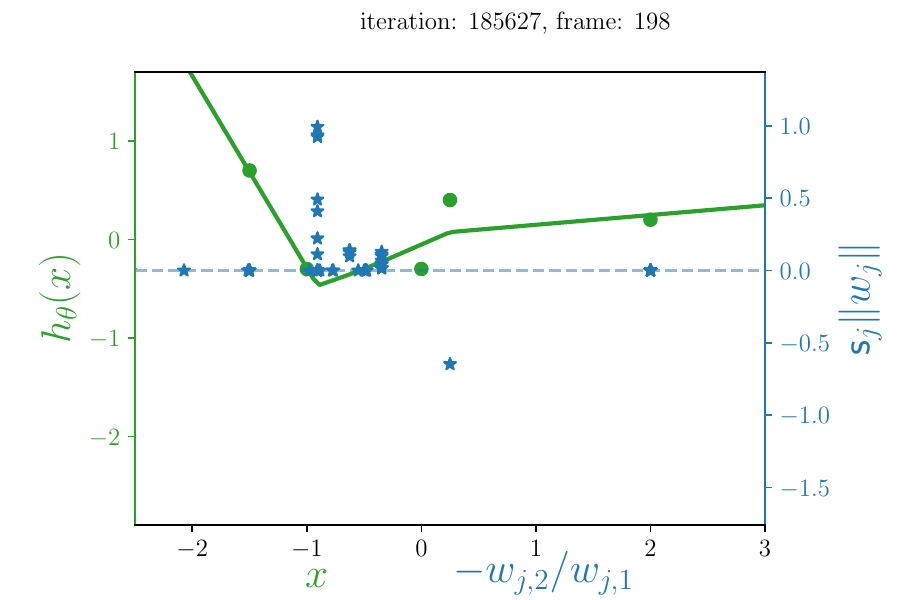}
            \caption{\label{fig:nonortho3}{\small Simultaneous norm growth and realignment (Iteration $150000$)}}    
        \end{subfigure}
        \vskip\baselineskip
        \begin{subfigure}[b]{0.475\textwidth}   
            \centering 
            \includegraphics[trim=10pt 0pt 10pt 30pt, clip, width=\textwidth]{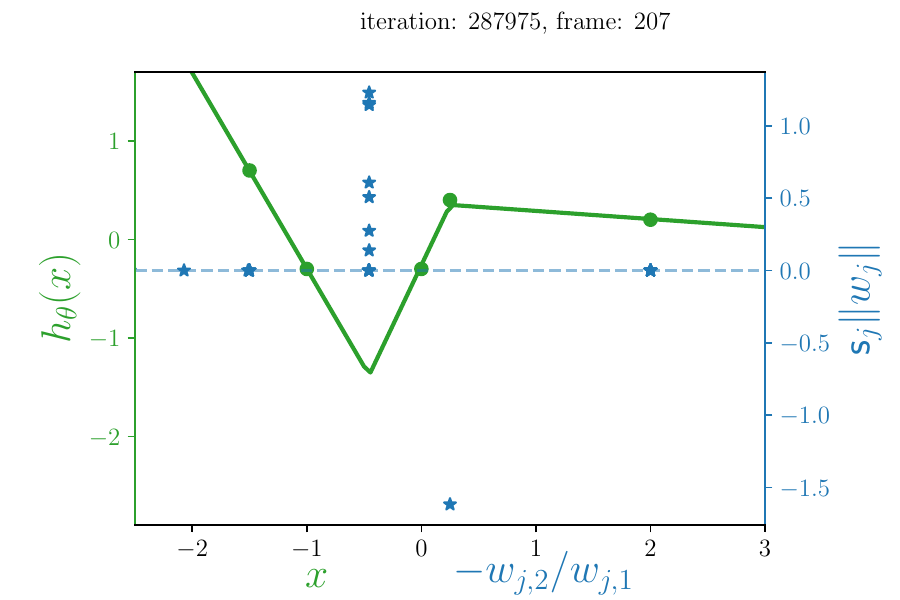}
            \caption{\label{fig:nonortho4}{\small Final state (Iteration $300000$)}}    
        \end{subfigure}
        \hfill
        \begin{subfigure}[b]{0.475\textwidth}   
            \centering 
            \includegraphics[width=\textwidth]{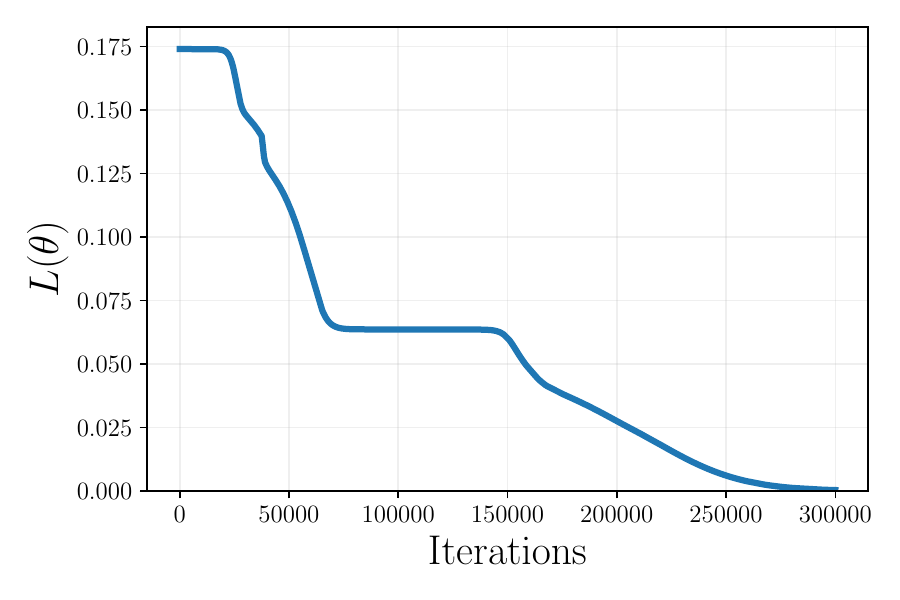}
            \caption{\label{fig:nonorthoprofile}{\small Loss profile}}    
        \end{subfigure}
        \caption{\label{fig:non-ortho} State of training at different stages and loss profile. } 
    \end{figure}

Similarly to \cref{fig:dynamics}, \cref{fig:non-ortho} presents the training dynamics with a small initialisation ($\lambda=\frac{10^{-4}}{\sqrt{m}}$). As in the orthogonal case, we first observe an early alignment phase in \cref{fig:nonortho1}. Afterwards, two groups of neurons (against a single group in the orthogonal case) grow in norms until reaching an intermediate saddle point in \cref{fig:nonortho2}. Note that during this norm-growth phase, the group of neurons does not remain aligned with a fixed direction, but changed in direction. A similar phenomenon happens when leaving the intermediate saddle point in \cref{fig:nonortho3}: the norm of these neurons still grow but they also change their direction. This behavior is what makes the general case fundamentally harder to analyse than the orthogonal one, where such a behavior does not happen. We believe that controlling this type of behavior is the key towards dealing with the general set up.

In \cref{fig:nonortho4}, we have a zero training loss and as in the orthogonal case, the estimated function is sparse in its number of kinks. \cref{fig:nonorthoprofile} shows that a saddle to saddle dynamics also happens in this general setup.

   \begin{figure}[htbp]
        \centering
        \begin{subfigure}[b]{0.475\textwidth}
            \centering
            \includegraphics[trim=10pt 0pt 10pt 30pt, clip, width=\textwidth]{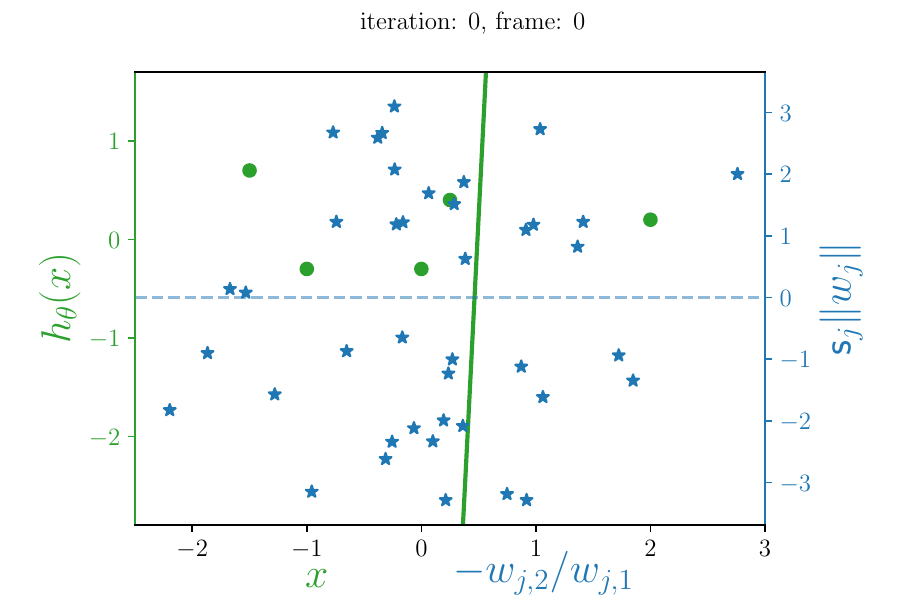}
            \caption{\label{fig:largeinit0}{\small Initialisation (Iteration $0$)}}  
        \end{subfigure}
        \hfill
        \begin{subfigure}[b]{0.475\textwidth}  
            \centering 
            \includegraphics[trim=10pt 0pt 10pt 30pt, clip, width=\textwidth]{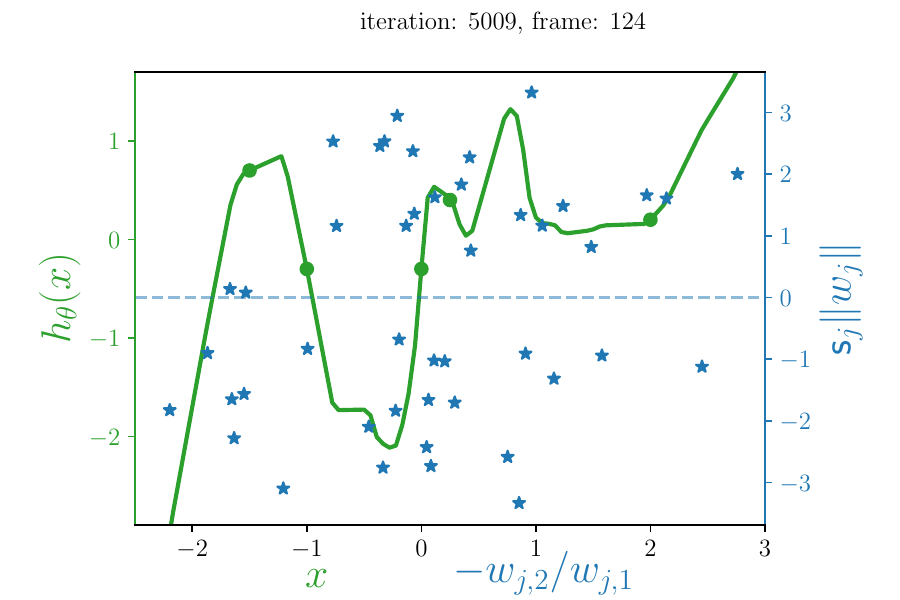}
            \caption{\label{fig:largeinit1}{\small End of training (Iteration $5000$)}}    
        \end{subfigure}
   \caption{\label{fig:largeinit} Training dynamics for large initialisation.} 
    \end{figure}
    
\cref{fig:largeinit} on the other hand studies the impact of the initialisation scale. In particular, it shows the training dynamics for a large scale of initialisation ($\lambda=\frac{10}{\sqrt{m}}$). By comparing \cref{fig:largeinit0,fig:largeinit1}, we indeed observe the \textit{lazy regime}~\citep{chizat2019lazy}: the neuron weights do not significantly move between the initialisation and the end of training. 

The final estimated function is not as simple as for small initialisation: it is approximately the interpolator coming from the Neural Tangent Kernel at initialisation, whose associated RKHS is described in~\citet{bietti2019inductive}. The strong biased induced by initialisation and the poor generalising properties of this kernel interpolator illustrate the benefits of the \textit{rich regime} obtained for small initialisation. However, this large initialisation has the advantage of converging towards $0$ loss much faster: the trajectory indeed does not go through saddle points that might significantly slow down the learning process. Similar results are observed for large initialisations when either considering unbalanced initialisation or orthogonal data.

\subsection{High dimensional data}\label{app:expehighdim}

 \begin{figure}[htbp]
        \centering
        \begin{subfigure}[b]{0.45\textwidth}
            \centering
            \includegraphics[trim=90pt 10pt 80pt 15pt, clip, width=\textwidth]{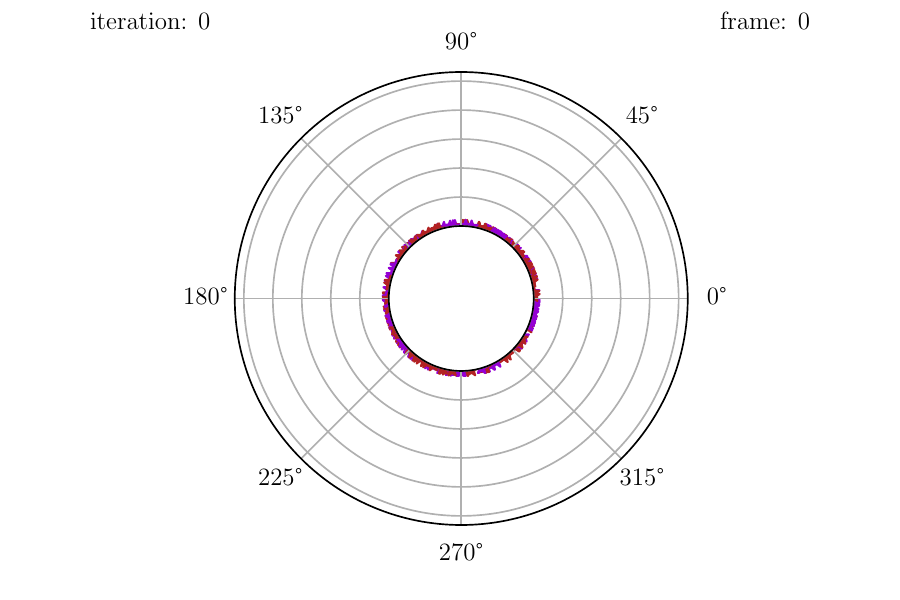}
            \caption{\label{fig:highdim0}{\small Initialisation (Iteration $0$)}}  
        \end{subfigure}
        \hfill
        \begin{subfigure}[b]{0.45\textwidth}  
            \centering 
            \includegraphics[trim=90pt 10pt 80pt 15pt, clip, width=\textwidth]{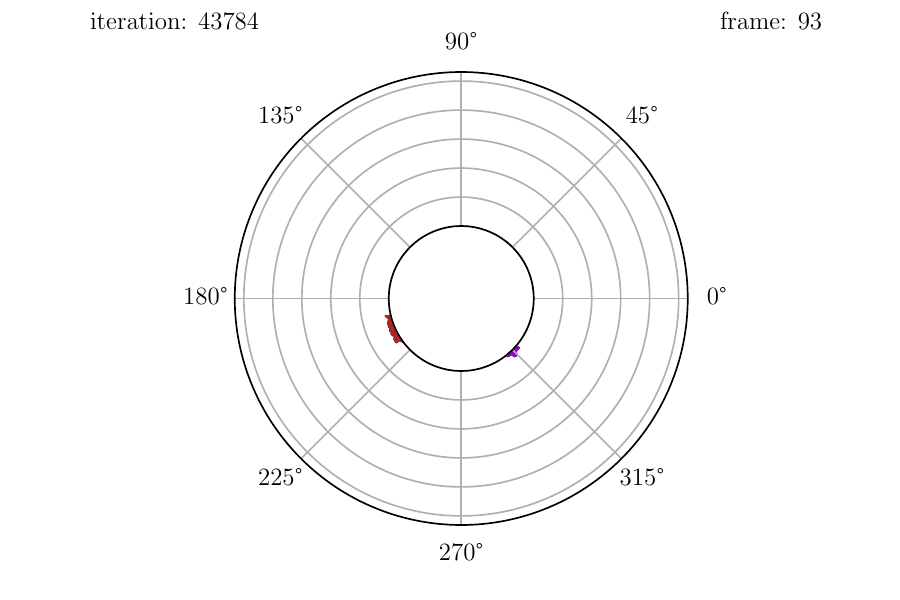}
            \caption{\label{fig:highdim1}{\small End of alignment phase (Iteration $40000$)}}    
        \end{subfigure}
        \vskip\baselineskip
        \begin{subfigure}[b]{0.45\textwidth}   
            \centering 
            \includegraphics[trim=90pt 10pt 80pt 15pt, clip, width=\textwidth]{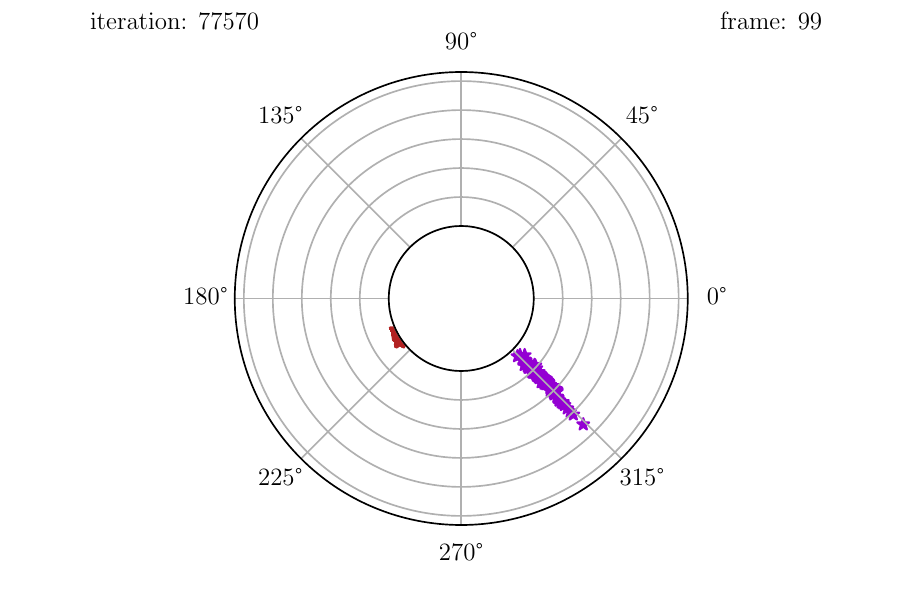}
            \caption{\label{fig:highdim2}{\small Growth of first cluster (Iteration $80000$)}}    
        \end{subfigure}
        \hfill
        \begin{subfigure}[b]{0.45\textwidth}   
            \centering 
            \includegraphics[trim=90pt 10pt 80pt 15pt, clip, width=\textwidth]{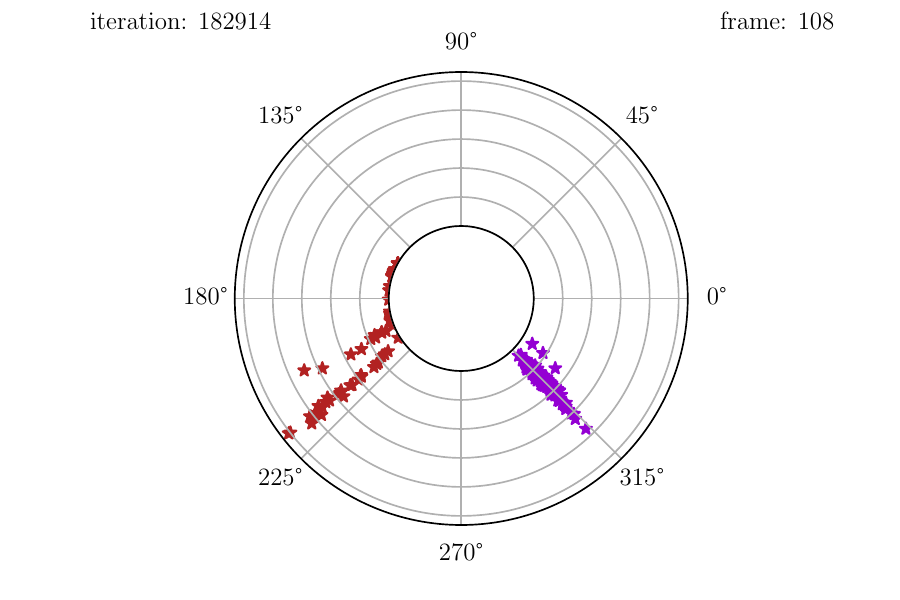}
            \caption{\label{fig:highdim3}{\small Final state (Iteration $200000$)}}    
        \end{subfigure}
        \caption{\label{fig:highdim}State of training at different stages. Each red (resp. purple) star represents a single neuron with $\s_j=-1$ (resp. $\s_j=1$): it shows (in polar coordinates) the projection of the hidden layer weight onto the $2$ dimensional space spanned by the two principal components of the hidden layer weights at the final state of training. The inner circle corresponds to $0$ norm vectors, whose direction is given by the angle.} 
              \vspace{.6em}
    \end{figure}
    
       \begin{figure}[htbp]
        \centering
        \begin{subfigure}[b]{0.475\textwidth}
            \centering
            \includegraphics[trim=10pt 10pt 10pt 10pt, clip, width=\textwidth]{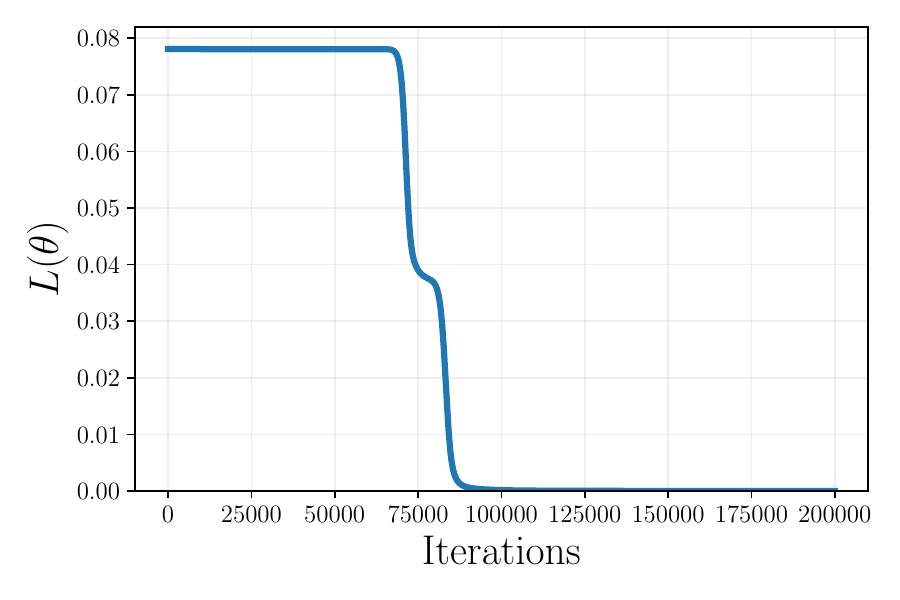}
            \caption{\label{fig:highdimloss}{\small Loss profile}}  
        \end{subfigure}
        \hfill
        \begin{subfigure}[b]{0.475\textwidth}  
            \centering 
            \includegraphics[trim=10pt 10pt 10pt 10pt, clip, width=\textwidth]{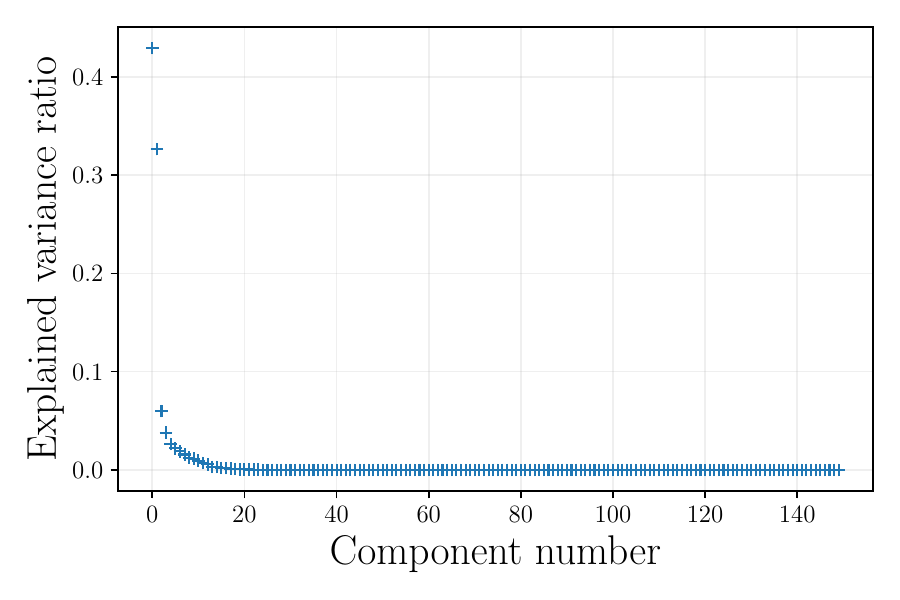}
            \caption{\label{fig:highdimratio}{\small Explained variance ratios of the principal components analysis of the final hidden weights}}    
        \end{subfigure}
   \caption{\label{fig:highdimsupp}Additional information on the high dimensional experiment.}
    \end{figure}

This section presents additional experiments for high dimensional, nearly orthogonal data. We generated $n=75$ data points $x_i$ drawn independently at random according to a standard Gaussian distribution of dimension $d = 150$. It is then known that such points are almost orthogonal with large probability. 
The labels $y_i$ are then given by a $6$-neurons teacher network, whose weights were drawn at random following a Gaussian distribution. 

From there, we trained a neural network with width $m=200$ (without bias terms), an initialisation scale $\lambda = \frac{10^{-20}}{\sqrt{md}}$ and a gradient step size of $10^{-3}$. \cref{fig:highdim} illustrates the training dynamics of the parameters when projected onto the $2$ dimensional space of the two principal components of the $200\times150$ matrix associated to the hidden layer of the network.

The training loss profile is given by \cref{fig:highdimloss}. \cref{fig:highdimratio} finally shows the explained variance ratio of the principal components of the PCA used in \cref{fig:highdim}.

\medskip

Behaviors close to the orthogonal case can easily be observed here. First, we can see in \cref{fig:highdimloss} that the training loss converges to $0$ and that the dynamics goes through an intermediate saddle point around iteration $75000$. Also, thanks to~\cref{fig:highdim}, we see the training dynamics (at least when projected onto the $2$ dimensional space of the two principal components) follows similar phases. At first, we observe an early alignment phase where the neurons align towards two key directions. During a second phase, a first cluster of neurons grows in norm while keeping a fixed direction. During a third phase, the same happens for the other cluster of neurons. Only after these three phases, the neurons will slightly move from these two key directions and reach the final state of \cref{fig:highdim3}, for which we see that the neurons are not exactly aligned with two directions.

This last state is what differs from the exactly orthogonal case, where the final solution consists of a $2$ neurons network. This is obviously not the case here, as can be seen from \cref{fig:highdim3}, but also from the explained variance ratios given in \cref{fig:highdimratio}. Indeed, we can there see that $80\%$ of the final state neurons are explained by the two principal components. This implies that the estimated interpolator is close to a $2$ neurons network, but still far from being only represented by $2$ direction (around $20\%$ of its variance is explained by the remaining directions).

\medskip

Note that we here chose a very small initialisation scale (of order $10^{-20}$). As explained in \cref{subsec:discussion}, this confirms the exponential dependence of $\lambda$ in the number of data points and is merely needed to observe a clear saddle point in the dynamics, but similar final states of training are observed for much larger values initialisation scales. 

\section{Main proofs}\label{app:proof}

In this section, we prove the main theorem. \cref{sec:notations,app:sec_initialisation} provide additional notations and recall the assumption on the initialisation. Then, each subsection corresponds to the study of the different phases of the dynamics:
\cref{app:sec_phase1,app:sec_phase2,app:sec_phase3,app:sec_phaseconvergence} prove respectively the alignment phase, the positive, then negative label fitting phases and finally the convergence phase.

\subsection{Additional notations}
\label{sec:notations}
First, we introduce the following additional notations. We need to define the two dynamical vectors that encode respectively the fitting of positive and negative labels
\begin{equation*}
D_+^t = -\frac{1}{n}\sum_{k \mid y_k>0}(h_{\theta^t}(x_k)-y_k) x_k \quad \textrm{ and }\quad  D_-^t = -\frac{1}{n}\sum_{k \mid y_k<0}(h_{\theta^t}(x_k)-y_k) x_k.
\end{equation*}
We also define the following open subsets of $\R^d$: 
\begin{equation*}
S_+ = \{ w \in \R^d \mid \forall k, \iind{\langle w, x_k\rangle\geq0} \geq  \iind{y_k>0}\} \quad \textrm{ and } \quad  S_- = \{ w \in \R^d \mid \forall k, \iind{\langle w, x_k\rangle\geq0} \geq  \iind{y_k<0}\}.
\end{equation*}
Note that $D_+ \in S_+$ and $-D_- \in S_-$ 
%
%
and that neurons defined in \cref{eq:S+1,eq:S-1} correspond to
\begin{equation*}
S_{+,1} = \{ j \in \llbracket m \rrbracket \mid w_j^0 \in S_+ \text{ and } \s_j = 1\} \quad \textrm{ and } \quad  S_{-,1} = \{ j \in \llbracket m \rrbracket \mid w_j^0 \in S_- \text{ and } \s_j = -1\}.
\end{equation*}

\subsection{Initialisation and assumptions on the dataset}
\label{app:sec_initialisation}

We recall here, for the sake of completeness, the setup at initialisation. We initialise the dynamics on a balanced fashion: $a_j^0 = \s_j \|w_j^0\|$. We take $\lambda >  0$ and  assume that $w_j^0 = \lambda g_j$, where each $g_j \sim \mathcal{N}(0, I_d)$ is an independent standard Gaussian. We assume that both $S_{+,1}$ and $S_{-,1}$ are non-empty and that for all $k$, $y_k \neq 0$. We also assume without loss of generality that $\| D_+\| > \| D_-\|$. This makes $r = \| D_+\| / \| D_-\| > 1$. 
We fix $\varepsilon>0$ small enough so that
\begin{gather*}
    (1+3\varepsilon)\max\left(1-\alpha, \frac{\|D_-\|}{\|D_+\|}\right)\leq 1-\varepsilon \quad \text{and} \quad (1+3r\varepsilon)(1-\beta)\leq 1-\varepsilon\\
    \text{where} \quad \alpha = \min_{k\mid y_k>0}\frac{y_k^2}{2\|D_+\|^2}, \quad \beta = \min_{k\mid y_k<0}\frac{y_k^2}{2\|D_-\|^2}.
\end{gather*}
We finally introduce an arbitrarily small $\lambda_*>0$ that only depends on the training set $(x_k,y_k)_k$ and set $c=\max_{j\in\llbracket m\rrbracket} \|w_j^0\| / \lambda = \max_{j\in\llbracket m\rrbracket} \|g_j\|$.

Because of the non-differentiability of the ReLU activation, the gradient flow is not uniquely defined. In the orthogonal case, we have in particular the following ODE
\begin{equation}\label{eq:multipleflows}
    \frac{\dd \langle w_j^t, x_k \rangle}{\dd t} = -\frac{h_{\theta^t}(x_k)-y_k}{n}\|w_j^t\| \iind{\langle w_j^t, x_k\rangle>0},
\end{equation}
which leads to the following key lemma.
\begin{lem}\label{lemma:sectors}
For any $t'\geq t$, $j\in\llbracket m \rrbracket$ and $k\in[n]$:
\begin{equation*}
\iind{\langle w_j^t,x_k \rangle< 0} \implies \iind{\langle w_j^{t'},x_k \rangle< 0}.
\end{equation*}
\end{lem}
\begin{proof}
This is a direct consequence of \cref{eq:multipleflows}.
\end{proof}

\subsection{Phase 1: proof of Lemma~\ref{lemma:phase1}}
\label{app:sec_phase1}

During the first phase, the neurons remain small in norm and they move tangentially to align with the vectors $D_+$ and $D_-$. The duration of this movement is typically sub-logarithmic in $\lambda$, the initialisation scale. As $\|D_+\| > \|D_-\|$, neurons in $S_{+,1}$ move slightly faster to align with $D_+$ than the ones in  $S_{-,1}$ that align with $-D_-$. We begin by describing the dynamics of neurons in $S_{+,1}$ in \cref{lemma:phase_plus_1} and derive similar results for the one of $S_{-,1}$ in \cref{lemma:phase_moins_1}. These two Lemmas constitute together \cref{lemma:phase1} of the main text.

First, we define the following ending time of the phase $+1$:
\begin{equation}
t_{+1} = \frac{- \varepsilon \ln(\lambda)}{\|D_+\|}. \label{eq:t+1} 
\end{equation}
We have the following lemma.

\begin{lem}[Phase $+1$]\label{lemma:phase_plus_1}
If $\lambda\leq \lambda_*$, then we have the following inequalities,
\begin{enumerate}[itemsep=-0.5em, topsep=-1em, label=(\roman*)]
\item $\forall j\in\llbracket m \rrbracket, \|w_j^{t_{+1}} \| \leq 2c \lambda^{1-\varepsilon}$
\item $\forall j\in  S_{+,1},\langle \w_j^{t_{+1}},D_+\rangle \geq (1-2\lambda^{\varepsilon})\|D_+\|$.
\item For all $j \in S_{+,1}$, let $k$ such that $y_k < 0$, then $\langle w_j^{t_{+1}}, x_k \rangle = 0$.
\end{enumerate}
\end{lem}
The condition (i) of the lemma means that the neurons do not grow so much in this first phase, whereas (ii) states that they have aligned with vectors $D_+$. Finally (iii) shows that neurons in $S_{+,1}$ deactivate along negative labels during this first phase.

\begin{proof} We divide the proof in three steps. First one is to control the growth of the neurons norm until $t_{+1}$. The second step shows that the tangential movement is faster, while the third one shows that neurons in $S_{+,1}$ ``unalign'' with the directions of negative labels.

\noindent \underline{First step}: we show (i), i.e. that for $t\leq t_{+1}$, we have $\|w_j^t \| \leq  2c \lambda^{1- \varepsilon}$. Note first that, thanks to the balancedness, $a_j^t = \s_j\|w_j^t \|$. Let $\tau_\lambda:=  \inf\{ t \geq 0 \mid \exists j\in\llbracket m \rrbracket, \|w_j^t \| > 2c \lambda^{1-\varepsilon}\}$, then for all $t\leq \tau_\lambda$, we have $\|w_j^t \| \leq  2c \lambda^{1-\varepsilon}$ and hence,
\begin{equation*}
\left|h_{\theta^t}(x_k)\right| \leq \sum_{j=1}^m \|w_j^t \| |\sigma(\langle w_j, x_k \rangle)| \leq \sum_{j=1}^m \|w_j^t \|^2 \leq 4mc^2\lambda^{2(1-\varepsilon)}.
\end{equation*}
Then, for $\lambda_*$ such that $\lambda_*^{2(1-\varepsilon)}\leq \frac{\min_k |y_k|}{4mc^2}$, $y_k - h_{\theta^t}(x_k)$ and $y_k$ have the same sign for any $t\leq \tau_\lambda$. As a consequence, for $j$ such that $\s_j=1$, \cref{eq:ode_a_w} yields
\begin{align*}
\frac{\dd \|w_j^t \|}{\dd t} &= \frac{1}{n}\sum_{\substack{k| \langle x_k, w_j^t \rangle >0}}\left(y_k-h_{\theta^t}(x_k)\right)\langle x_k, w_j^t \rangle \leq \frac{1}{n}\sum_{\substack{k \mid \langle x_k, w_j^t \rangle >0}}\left(y_k-h_{\theta^t}(x_k)\right)\langle x_k, w_j^t \rangle \iind{y_k>0}.
\end{align*}
Now, denote $\displaystyle D^t_{+,j} := \frac{1}{n}\sum_{\substack{k \mid \langle x_k, w_j^t \rangle >0}}\left(y_k-h_{\theta^t}(x_k)\right) x_k \iind{y_k>0}$, we have 
\begin{align*}
\langle D^t_{+,j}, w_j^t \rangle &\leq \|D^t_{+,j}\| \|w_j^t\| \leq \|D^t_{+}\| \|w_j^t\| \leq (\|D_+\| + \|\frac{1}{n}\sum_{k,y_k>0}h_{\theta^t}(x_k)x_k\|)\|w_j^t\| \\
&\leq (\|D_+\| + 4mc^2\lambda^{2(1-\varepsilon)})\|w_j^t\|. 
\end{align*}
The same series of inequalities holds in the case $\s_j = -1$ for $\|D_-^t\|$. Overall,
\begin{align*}
\frac{\dd \|w_j^t \|}{\dd t} & \leq \left(\max\left\{\|D_+\|,\|D_-\|\right\} + 4mc^2\lambda^{2(1-\varepsilon)}\right) \|w_j^t \|\\
& \leq \left(\|D_+\| + 4mc^2\lambda^{2(1-\varepsilon)}\right)\|w_j^t\|.
\end{align*}
By Gr\"onwall's lemma, this gives for any $t\leq \tau_\lambda$, $\|w_j^t \|\leq \|w_j^0 \|e^{\left(\|D_+ \|+4mc^2\lambda^{2(1-\varepsilon)}\right)t}$. This shows that for $t\leq \min(\tau_\lambda, t_{+1})$,
\begin{equation}\label{eq:norm1}
\|w_j^t \| \leq c\lambda^{1-\varepsilon} e^{\frac{-4\varepsilon mc^2\lambda^{2(1-\varepsilon)}}{\|D_+\|}\ln(\lambda)} < 2c\lambda^{1-\varepsilon},
\end{equation}
where $\lambda^*$ has been taken small enough. Hence $\tau_\lambda \geq t_{+1}$.

\medskip

\noindent \underline{Second step}: we show condition (ii). 
Indeed, let us now choose any $j \in S_{+,1}$. We have the following decomposition:
\begin{equation*}
D_j^t = D_{+} - \frac{1}{n}\sum_{\substack{k \mid \langle x_k, w_j^t \rangle >0}}h_{\theta^t}(x_k) x_k \iind{y_k>0} + D_{-,j}^t, 
\end{equation*}
so that as all vectors in the $D_{-,j}^t$ are different from the one of $D_{+}$, by orthogonality we have,
\begin{equation*}
\langle D_j^t,D_{+}\rangle = \|D_{+}\|^2 - \frac{1}{n}\sum_{\substack{k \mid \langle x_k, w_j^t \rangle >0}}h_{\theta^t}(x_k) \langle x_k, D_{+}\rangle \iind{y_k>0}. 
\end{equation*}
Now, let us analyse the tangential movement for these $j$'s: for all $t\leq t_{+1}$, \cref{eq:ode_norm_angle} leads to the following growth comparison
\begin{align*}
\frac{\dd \langle \w_j^t, D_+ \rangle}{\dd t} &= \langle D^{t}_{j}, D_+\rangle  - \langle D^{t}_{j}, \w_j^t \rangle \langle \w_j^t, D_+\rangle   \\ 
&= \|D_+\|^2 - \langle \w_j^t, D_+ \rangle^2 - \frac{1}{n}\sum_{\substack{k \mid \langle x_k, w_j^t \rangle >0}}h_{\theta^t}(x_k) \langle x_k, D_{+} -  \langle D_+, \w_j^t\rangle  \w_j^t \rangle \iind{y_k>0} - \langle D^{t}_{j,-}, \w_j^t \rangle \langle \w_j^t, D_+\rangle, 
\end{align*}
and as we have, $- \langle D^{t}_{j,-}, \w_j^t \rangle \langle \w_j^t, D_+\rangle \geq 0$ and the third term lower bounded by $ - 4 m c^2 \|D_+\| \lambda^{2 (1 -\varepsilon)}$, it yields
\begin{align}\label{eq:odecomparison}
\frac{\dd \langle \w_j^t, D_+ \rangle}{\dd t} &\geq \|D_+\|^2 - 4mc^2 \|D_+\| \lambda^{2(1-\varepsilon)} - \langle \w_j^t, D_+ \rangle^2.
\end{align}
Solutions of the ODE $f'(t) = a^2 - f^2(t)$ with value in $(-a,a)$ are of the form $f(t)=a\tanh(a(t+t_0))$ for some $t_0$. Note in the remaining of the proof $a^2 = \|D_+\|^2 - 4mc^2 \|D_+\| \lambda^{2(1-\varepsilon)}$. In our case, define $t_0$ such that $\langle \w_j^0, D_+ \rangle = a\tanh(at_0)$, then, by Gr\"onwall comparison, it yields:
\begin{equation*}
\langle \w_j^t, D_+ \rangle \geq a \tanh(a(t+t_0)).
\end{equation*}
Note that $\langle \w_j^0, D_+ \rangle\geq 0$, so that $t_0 \geq 0$. As a consequence, we simply have
\begin{equation*}
\langle \w_j^t, D_+ \rangle \geq a \tanh(at).
\end{equation*}
Now, using the inequality $\tanh(x)\geq 1-2e^{-2x}$, we have
\begin{align*}
\langle \w_j^{t_{+1}}, D_+ \rangle & \geq a (1-2e^{2a\frac{\varepsilon}{\| D_+\|}\ln(\lambda)}).
\end{align*}
We have the following inequalities on $a$ when choosing $\lambda_*$ small enough:
\begin{align*}
a = \sqrt{ \|D_+\|^2 - 4mc^2 \|D_+\| \lambda^{2(1-\varepsilon)}} \geq \frac{3\|D_+\|}{4}.
\end{align*}
This leads for any $j\in S_{+,1}$ and $\lambda^*$ small enough to
\begin{align*}
 \langle \w_j^{t_{+1}}, D_+ \rangle & > (\|D_+\| -  2c\sqrt{m \| D_+\|}\lambda^{1-\varepsilon}) \left(1-2\lambda^{\frac{3\varepsilon}{2}} \right) \\
 & > \|D_+\|\left(1-2\lambda^{\frac{3\varepsilon}{2}} \right) - 2c\sqrt{m \| D_+\|}\lambda^{1-\varepsilon} \\
  & > \|D_+\|\left(1-2\lambda^{\varepsilon} \right),
\end{align*}
which proves condition (ii).

\medskip

\noindent \underline{Third step}: we show (iii). Consider $j \in S_{+,1}$ and $k$ such that $y_k < 0$. If $\langle w_j^{0}, x_k\rangle < 0$, then by \cref{lemma:sectors}, there is nothing to prove. Otherwise, as for $t\leq t_{+1}$ we have $\|w_j^t \| \leq  2c \lambda^{1- \varepsilon}$, it yields (as long as $\langle\w_j^{t}, x_k\rangle>0$)
\begin{align*}
\frac{\dd}{\dd t} \langle \w_j^{t}, x_k\rangle &=  \langle D_j^{t}, x_k\rangle - \langle D_j^{t}, \w_j^{t}\rangle \langle \w_j^{t}, x_k\rangle \\
&= \frac{1}{n} ( y_k - h_\theta(x_k)) -\frac{1}{n}  \langle \w_j^{t}, x_k\rangle \sum_{l \mid \langle \w_j^{t}, x_l\rangle > 0} (y_l - h_{\theta}(x_l)) \langle \w_j^{t}, x_l\rangle < 0.
\end{align*}
Hence, if for some time $\tau \leq t_{+1}$, we have $\langle \w_j^{\tau}, x_k\rangle = 0$, then for all times $t \in [\tau, t_{+1}]$, this quantity remains $0$. We now continue to upperbound the derivative of $\langle \w_j^{\tau}, x_k\rangle$:
\begin{align*}
\frac{\dd}{\dd t} \langle \w_j^{t}, x_k\rangle&\leq \frac{1}{n} (y_k + 4mc^2 \lambda^{2(1-\varepsilon)}) -\frac{1}{n}  \langle \w_j^{t}, x_k\rangle \sum_{l \mid \langle \w_j^{t}, x_l\rangle > 0} ( y_l - 4mc^2 \lambda^{2(1-\varepsilon)}) \langle \w_j^{t}, x_l\rangle \\
&\leq \frac{y_k}{n} - \frac{1}{n}  \langle \w_j^{t}, x_k\rangle \sum_{l \mid \langle \w_j^{t}, x_l\rangle > 0} y_l \langle \w_j^{t}, x_l\rangle \iind{y_l<0} + 5mc^2 \lambda^{2(1-\varepsilon)}.
\end{align*}
From this, we sum all the $y_k< 0$, and noting $\displaystyle f(t):= - \frac{1}{n}\!\!\sum_{k \mid \langle \w_j^{t}, x_k\rangle > 0}\!\!\!\! y_k \langle \w_j^{t}, x_k\rangle \iind{y_k<0} $, we have
\begin{align*}
\frac{\dd}{\dd t} f(t) \leq -\frac{1}{n^2} \sum_{k\mid y_k<0} y^2_k   + f(t)^2 - 5mc^2 \lambda^{2(1-\varepsilon)} \frac{1}{n} \sum_{k\mid y_k<0} y_k.
\end{align*}
If we let $a^2 := \frac{1}{n^2} \sum_{k\mid y_k<0} y^2_k + 5mc^2 \lambda^{2(1-\varepsilon)} \frac{1}{n} \sum_{k\mid y_k<0} y_k>0$, we have $f'(t) \leq -a^2 + f(t)^2$, and as by Cauchy-Schwarz
\begin{align*}
f(0)^2 &\leq \frac{1}{n^2} \sum_{k \mid \langle \w_j^{0}, x_k\rangle > 0} y_k^2\iind{y_k<0} \sum_{k \mid \langle \w_j^{0}, x_k\rangle > 0} \langle \w_j^{0}, x_k\rangle^2 \iind{y_k<0} \\
&= \frac{1}{n^2} \left(\sum_{k \mid \langle \w_j^{0}, x_k\rangle > 0} y_k^2\iind{y_k<0}\right) \left(\|\w_j^0\|^2 - \sum_{k} \langle \w_j^{0}, x_k\rangle^2 \left(\iind{y_k>0} \iind{\langle \w_j^{0}, x_k\rangle > 0} + \iind{\langle \w_j^{0}, x_k\rangle < 0}\right) \right) \\
&< \left(a^2 + \lambda^{1-2\varepsilon}\right)\left(1 - \sum_{k} \langle \w_j^{0}, x_k\rangle^2 \left(\iind{y_k>0} \iind{\langle \w_j^{0}, x_k\rangle > 0} + \iind{\langle \w_j^{0}, x_k\rangle < 0}\right) \right) \\
&< a^2,
\end{align*}
where the two last inequalities are valid for $\lambda$ small enough. Hence $t \mapsto f(t)$ is decreasing and $f'(t) \leq -a^2 + f(0)^2$, that is if we call $b := a^2 - f(0)^2>0$, we have $f(t) \leq -b t + f(0)$, and for $\tau = f(0)/b$, we have $f(\tau)=0$. And as $\tau < t_{+1}$, we have $f(t_{+1})=0$. This concludes the proof of condition (iii) and of the lemma.
\end{proof}
Now, we show that the exact same conclusion is also valid for the neurons in $S_{-,1}$ i.e., after a sub-logarithmic time, they eventually align with $-D_-$ and deactivates with respect to positive outputs. We define here a ending time similar to \eqref{eq:t+1}:
\begin{equation}
t_{-1} = \frac{- \varepsilon r \ln(\lambda)}{\|D_+\|}. \label{eq:t-1} 
\end{equation}
We prove the following lemma analogous to \cref{lemma:phase_plus_1}.
\begin{lem}[Phase $-1$]\label{lemma:phase_moins_1}
If $\lambda\leq \lambda_*$, then we have the following inequalities,
\begin{enumerate}[itemsep=-0.5em, topsep=-1em, label=(\roman*)]
\item $\forall j\in\llbracket m \rrbracket, \|w_j^{t_{-1}} \| \leq 2c \lambda^{1-r\varepsilon}$
\item $\forall j\in  S_{-,1},\langle \w_j^{t_{-1}},-D_-\rangle \geq (1-2\lambda^{\varepsilon})\|D_-\|$.
\item For all $j \in S_{-,1}$, let $k$ such that $y_k > 0$, then $\langle w_j^{t_{-1}}, x_k \rangle = 0$.
\end{enumerate}
\end{lem}
\begin{proof}
The proof is essentially the same than the proof of \cref{lemma:phase_plus_1}. We will be short and underline solely the main differences.

\medskip

\noindent \underline{First step}: we show condition (i), i.e. that for $t\leq t_{-1}$, we have $\|w_j^t \| \leq  2c \lambda^{1-r \varepsilon}$. Indeed, let $\tau^r_\lambda:=  \inf\{ t \geq 0 \mid \exists j\in\llbracket m \rrbracket, \|w_j^t \| > 2c \lambda^{1-r\varepsilon}\}$. For all $t\leq \tau^r_\lambda$, we have $\left|h_{\theta^t}(x_k)\right| \leq 4mc^2\lambda^{2(1-r\varepsilon)}$ and similarly to the proof of \cref{lemma:phase_plus_1}, we have that for $t\leq \tau^r_\lambda$, $\|w_j^t \|\leq \|w_j^0 \|e^{\left(\|D_+ \|+4mc^2\lambda^{2(1-r\varepsilon)}\right)t}$. This shows that for $t\leq \min(\tau^r_\lambda,\, t_{-1})$,
\begin{equation}
\|w_j^t \| \leq c\lambda^{1-r\varepsilon} e^{\frac{-4\varepsilon mc^2\lambda^{2(1-r\varepsilon)}}{\|D_+\|}\ln(\lambda)} < 2c\lambda^{1-r\varepsilon},
\end{equation}
where $\lambda^*$ has been taken small enough. Hence $\tau^r_\lambda \geq t_{-1}$.

\medskip

\noindent \underline{Second step}: we show (ii), i.e. that the neurons almost align after time $t_{-1}$. Indeed, for $j \in S_{-,1}$, similarly to \cref{lemma:phase_plus_1}, we have that for all $t\leq t_{-1}$:
\begin{equation*}
\frac{\dd \langle \w_j^t, -D_- \rangle}{\dd t} \geq \|D_-\|^2 - 4mc^2 \|D_-\| \lambda^{2(1-r\varepsilon)} - \langle \w_j^t, D_- \rangle^2.
\end{equation*}
Denoting $a_r^2 = \|D_-\|^2 - 4mc^2 \|D_-\| \lambda^{2(1-r\varepsilon)}$, we have by Gr\"onwall comparison
\begin{equation*}
\langle \w_j^t, -D_- \rangle \geq a_r \tanh(a_rt).
\end{equation*}
Now, this gives $\langle \w_j^{t_{-1}}, -D_- \rangle > a_r (1-2e^{2a_r \frac{r\varepsilon}{\| D_+\|}\ln(\lambda)})$ and lower bounding $a_r$ as before, we have
\begin{equation*}
 \langle \w_j^{t_{-1}}, -D_- \rangle > \|D_-\|\left(1-2\lambda^{\varepsilon} \right),
\end{equation*}
and this shows the condition (ii).

\medskip

\noindent \underline{Third step}: we show (iii). Consider $j \in S_{-,1}$ and $k$ such that $y_k > 0$. If $\langle w_j^{0}, x_k\rangle < 0$, then by \cref{lemma:sectors}, there is nothing to prove. Otherwise, as for $t\leq t_{-1}$ we have $\|w_j^t \| \leq  2c \lambda^{1-r \varepsilon}$, it yields (as long as $\langle \w_j^{t}, x_k\rangle>0$)
\begin{align*}
\frac{\dd}{\dd t} \langle \w_j^{t}, x_k\rangle &= - \langle D_j^{t}, x_k\rangle + \langle D_j^{t}, \w_j^{t}\rangle \langle \w_j^{t}, x_k\rangle \\
&= \frac{1}{n} (h_\theta(x_k) - y_k) -\frac{1}{n}  \langle \w_j^{t}, x_k\rangle \sum_{l \mid \langle \w_j^{t}, x_l\rangle > 0} (h_{\theta}(x_l) - y_l) \langle \w_j^{t}, x_l\rangle < 0.
\end{align*}
Hence, if for some time $\tau \leq t_{-1}$, we have $\langle \w_j^{\tau}, x_k\rangle = 0$, then for all times $t \in [\tau, t_{-1}]$, this quantity will remain $0$. We now we continue to upperbound the derivative of $\langle \w_j^{\tau}, x_k\rangle$:
\begin{align*}
\frac{\dd}{\dd t} \langle \w_j^{t}, x_k\rangle&\leq \frac{1}{n} (4mc^2 \lambda^{2(1-r\varepsilon)} - y_k) -\frac{1}{n}  \langle \w_j^{t}, x_k\rangle \sum_{l \mid \langle \w_j^{t}, x_l\rangle > 0} (- 4mc^2 \lambda^{2(1-r\varepsilon)}  - y_l) \langle \w_j^{t}, x_l\rangle \\
&\leq -\frac{y_k}{n} + \frac{1}{n}  \langle \w_j^{t}, x_k\rangle \sum_{l \mid \langle \w_j^{t}, x_l\rangle > 0} y_l \langle \w_j^{t}, x_l\rangle \iind{y_l>0} + 5mc^2 \lambda^{2(1-r\varepsilon)}.
\end{align*}
From this, we sum all the $y_k> 0$, and noting $\displaystyle f(t):= \frac{1}{n}\!\!\sum_{k \mid \langle \w_j^{t}, x_k\rangle > 0}\!\!\!\! y_k \langle \w_j^{t}, x_k\rangle \iind{y_k>0} $, we have
\begin{align*}
\frac{\dd}{\dd t} f(t) \leq -\frac{1}{n^2} \sum_{k\mid y_k>0} y^2_k   + f(t)^2 + 5mc^2 \lambda^{2(1-r\varepsilon)} \frac{1}{n} \sum_{k\mid y_k>0} y_k.
\end{align*}
If we let $a^2 := \frac{1}{n^2} \sum_{k\mid y_k>0} y^2_k - 5mc^2 \lambda^{2(1-r\varepsilon)} \frac{1}{n} \sum_{k\mid y_k>0} y_k>0$, we have $f'(t) \leq -a^2 + f(t)^2$, and as by Cauchy-Schwarz
\begin{align*}
f(0)^2 &\leq \frac{1}{n^2} \sum_{k \mid \langle \w_j^{0}, x_k\rangle > 0} y_k^2\iind{y_k>0} \sum_{k \mid \langle \w_j^{0}, x_k\rangle > 0} \langle \w_j^{0}, x_k\rangle^2 \iind{y_k>0} \\
&< \frac{1}{n^2} \left(\sum_{k \mid \langle \w_j^{0}, x_k\rangle > 0} y_k^2\iind{y_k>0}\right) \left(\|\w_j^0\|^2 - \sum_{k} \langle \w_j^{0}, x_k\rangle^2 \left(\iind{y_k<0} \iind{\langle \w_j^{0}, x_k\rangle > 0} + \iind{\langle \w_j^{0}, x_k\rangle < 0}\right) \right) \\
&< \left(a^2 + \lambda^{1-2\varepsilon}\right)\left(1 - \sum_{k} \langle \w_j^{0}, x_k\rangle^2 \left(\iind{y_k<0} \iind{\langle \w_j^{0}, x_k\rangle > 0} + \iind{\langle \w_j^{0}, x_k\rangle < 0}\right) \right) \\
&< a^2,
\end{align*}
where the two last inequalities are valid for $\lambda$ small enough. Hence $t \mapsto f(t)$ is decreasing and $f'(t) \leq -a^2 + f(0)^2$, that is if we call $b := a^2 - f(0)^2>0$, we have $f(t) \leq -b t + f(0)$, and for $\tau = f(0)/b$, we have $f(\tau)=0$. And as $\tau < t_{-1}$, we have $f(t_{-1})=0$. This concludes the proof of condition (iii) and of the Lemma.
\end{proof}
We end this subsection by a lemma that shows that until $t_{\pm1}$, the neurons $w^t_j$ with $j \in S_{\pm1}$, could not have collapsed to $0$.

\begin{lem}\label{lemma:minnorm}
Define $\underline{c} = \min \|w_j^0\| / \lambda$, then, if $\lambda \leq \lambda^*$, 
\begin{enumerate}[itemsep=-0.5em, topsep=-1em, label=(\roman*)]
\item for all $t\leq t_{+1}$, $\forall j \in S_{+,1}$, $\displaystyle \|w_j^t\| > \underline{c} \lambda^{1 + \varepsilon}/2$,
\item for all $t\leq t_{-1}$, $\forall j \in S_{-,1}$, $\displaystyle \|w_j^t\| > \underline{c} \lambda^{1 + r\varepsilon}/2$.
\end{enumerate}
\end{lem}

\begin{proof}

Let us begin with the first point. As stated in the proof of Lemma~\ref{lemma:phase_plus_1}, $y_k - h_{\theta^t}(x_k)$ and $y_k$ have the same sign for any $t\leq t_{+1}$. As a consequence, for $j \in S_{+,1}$, it yields
\begin{align*}
\frac{\dd \|w_j^t \|}{\dd t} &\geq \langle D_{-,j}^t, w_j^t \rangle \geq -  \|D_{-,j}^t\| \|w_j^t \| \geq -\left(\|D_-\| + \|\frac{1}{n}\sum_{k,y_k<0}h_{\theta^t}(x_k)x_k\|\right)\|w_j^t\| \\
&\geq -\left(\|D_+\| + 4mc^2\lambda^{2(1-\varepsilon)}\right)\|w_j^t\|.
\end{align*}
Then, by Gr\"onwall's lemma, this gives for any $t\leq t_{+1}$, $\|w_j^t \|\geq \|w_j^0 \|e^{-\left(\|D_+ \|+4mc^2\lambda^{2(1-\varepsilon)}\right)t}$. This shows that, for $t\leq t_{+1}$,
\begin{equation*}
\|w_j^t \| \geq \underline{c}\lambda^{1+\varepsilon} e^{\frac{4\varepsilon mc^2\lambda^{2(1-\varepsilon)}}{\|D_+\|}\ln(\lambda)} > \underline{c}\lambda^{1+\varepsilon}/2.
\end{equation*}
This concludes the first point of the lemma.
The second point is very similar to the first one. Indeed, in this case also, for $t \leq t_{-1}$, $y_k$ rules the sign of the residual so that for $j \in S_{-,1}$,
\begin{align*}
\frac{\dd \|w_j^t \|}{\dd t} &\geq \langle D_{+,j}^t, w_j^t \rangle \geq -  \|D_{+,j}^t\| \|w_j^t \| \geq -\left(\|D_+\| + \|\frac{1}{n}\sum_{k,y_k>0}h_{\theta^t}(x_k)x_k\|\right)\|w_j^t\| \\
&\geq -\left(\|D_+\| + 4mc^2\lambda^{2(1-r\varepsilon)}\right)\|w_j^t\|.
\end{align*}
Then, by Gr\"onwall's lemma, this gives for any $t\leq t_{-1}$, $\|w_j^t \|\geq \|w_j^0 \|e^{-\left(\|D_+ \|+4mc^2\lambda^{2(1-r\varepsilon)}\right)t}$. This shows that, as  $t\leq t_{-1}$,
\begin{equation*}
\|w_j^t \| \geq \underline{c}\lambda^{1+r\varepsilon} e^{\frac{4\varepsilon r  mc^2\lambda^{2(1-r\varepsilon)}}{\|D_+\|}\ln(\lambda)} > \underline{c}\lambda^{1+r\varepsilon}/2.
\end{equation*}
This concludes the proof.
\end{proof}

\subsection{Phase 2: proof of Lemma~\ref{lemma:phase2informal}}
\label{app:sec_phase2}

During the second phase, the norm of the neurons in $S_{+,1}$ (which are aligned with $D_+$) grows until perfectly fitting all the positive labels of the training points. Meanwhile, all the other neurons do not move significantly.
We define the ending time of the second phase as
\begin{equation}\label{eq:t2}
\begin{aligned}
t_2 = \inf \Bigg\{ t\geq t_{+1} \mid & \text{ either } \exists j \not\in S_{+,1}, \|w_j^t \| \geq 2c\lambda^{\varepsilon}\\
& \text{or } \sum_{j\in S_{+,1}} \|w_j^t \|^2 \geq n\|D_+ \| - \lambda^{\frac{\varepsilon}{5}}\\
& \text{or } \exists j\in S_{+,1}, \langle \w_j^t,D_+\rangle \leq \|D_+\| - \lambda^{\frac{\varepsilon}{2}} \Bigg\}.
\end{aligned}
\end{equation}
The following lemma corresponds to \cref{lemma:phase2informal} of the main text.
\begin{lem}\label{lemma:phase2}
If $\lambda\leq \lambda_*$, then we have the following inequalities
\begin{enumerate}[itemsep=-0.5em, topsep=-1em, label=(\roman*)]
\item $t_2 \leq -\frac{1+3\varepsilon}{\|D_+ \|}\ln(\lambda),$
\item $\forall j \in \llbracket m \rrbracket\setminus S_{+,1}, \|w_j^{t_2} \| < 2c\lambda^{\varepsilon},$
\item $\forall j\in S_{+,1}, \langle \w_j^{t_2},D_+\rangle >  \|D_+\| - \lambda^{\frac{\varepsilon}{2}}.$
\end{enumerate}
\end{lem}

The first point of the lemma states that the second phase lasts a time of order $-\ln(\lambda)$. The other two points state that the first and third conditions in the definition of $t_2$ do not hold at the end of the phase. Thus, the second condition in \cref{eq:t2} holds at $t_2$, meaning that the norm of the neurons in $S_{+,1}$ have grown enough to fit the positive labels:
\begin{align*}
\sum_{j\in S_{+,1}} \|w_j^{t_2} \|^2 = n\|D_+ \| - \lambda^{\frac{\varepsilon}{5}}
\end{align*}
A direct consequence of this\footnote{This comes from the decomposition given by \cref{eq:D+t} in the proof.} is that at the end of the second phase, the training loss on the positive labels is of order $\lambda^{\frac{2\varepsilon}{5}}$ (at least).

\begin{proof}
\underline{Preliminaries}:
define for this proof $h(t) = \|D_+\|-\min_{j\in S_{+,1}}\langle \w_j^t,D_+\rangle$. By definition of the second phase, we have $h(t)\leq \lambda^{\frac{\varepsilon}{2}}$ for any $t\in[t_{+1},t_2]$.

We first show that during the whole second phase, $D_+^{t}$ is almost colinear with $D_+$. For any $j\in S_{+,1}$ and $t\in [t_{+1}, t_2]$, define $d_j^t=\w_j^t - \frac{D_+}{\|D_+ \|}$. Since $\w_j^t$ and $\frac{D_+}{\|D_+\|}$ are both of norm $1$, we have $\|d_j^t\|^2 =- 2\langle d_j^t, \frac{D_+}{\|D_+\|}\rangle$. By definition of $h(t)$,
\begin{equation*}
0\geq \langle d_j^t, D_+ \rangle \geq -h(t).
\end{equation*}
So finally, we have for any $j\in S_{+,1}$ and $t\in [t_{+1}, t_2]$ the decomposition
\begin{equation}\label{eq:decompw}
\w_j^t = \frac{D_+}{\|D_+ \|} + d_j^t \qquad \text{with} \qquad \|d_j^t\|^2 \leq 2\frac{h(t)}{\|D_+\|}.
\end{equation}
Now let any $k$ such that $y_k>0$. Using \cref{eq:decompw} and the fact that $ \|w_j^t \| \leq 2c\lambda^{\varepsilon}$ for any $j\not\in S_{+,1}$, we have for any $t\in[t_{+1},t_2]$:
\begin{align*}
h_{\theta^t}(x_k) & = \sum_{j\in S_{+,1}} \|w_j^t\|\langle w_j^t, x_k \rangle + \sum_{j\not\in S_{+,1}} \s_j\|w_j^t\|\langle w_j^t, x_k \rangle \\
& =  \sum_{j\in S_{+,1}} \frac{\|w_j^t\|^2}{\|D_+\|} \langle D_+, x_k \rangle + \underbrace{\sum_{j\in S_{+,1}} \|w_j^t\|^2 \langle d_j^t, x_k \rangle + \sum_{j\not\in S_{+,1}} \s_j\|w_j^t\|\langle w_j^t, x_k \rangle} \\
& =  \sum_{j\in S_{+,1}} \frac{\|w_j^t\|^2}{\|D_+\|} \langle D_+, x_k \rangle + \hspace{3cm} h_k(t),
\end{align*}
where $\displaystyle |h_k(t)| \leq \sum_{j\in S_{+,1}} \|w_j^t\|^2 |\langle d_j^t, x_k \rangle|  + 4mc^2\lambda^{2\varepsilon}$. 
%
It follows that
\begin{align*}
D_+^t &= -\frac{1}{n}\sum_{k, y_k>0} h_{\theta^t}(x_k)x_k + D_+ \\
&  = (1-\frac{\sum_{j\in S_{+,1}} \|w_j^t\|^2}{n\|D_+\|})D_+ -\frac{1}{n}\sum_{k, y_k>0} h_k(t)x_k.
\end{align*}
Roughly, this means that as long as $1-\frac{\sum_{j\in S_{+,1}} \|w_j^t\|^2}{n\|D_+\|}$ is large enough, $D_+^{t}$ is almost colinear with $D_+$. Precisely, we have for any $t\in [t_{+1},t_2]$ the following decomposition
\begin{gather}\label{eq:D+t}
D_+^t = \left( 1- \frac{\sum_{j\in S_{+,1}} \|w_j^t\|^2}{n\|D_+\|} + h_+(t)\right) D_+ + h_{\bot}(t), \\ \text{where}\quad \langle h_{\bot}(t) , D_+\rangle = 0 \quad\text{and}\quad |h_+(t)|\|D_+\| \lor \|h_{\bot}(t)\| \leq \frac{1}{n}\sum_{j\in S_{+,1}} \|w_j^t\|^2 \sqrt{2\frac{h(t)}{\|D_+\|}} + 4mc^2\lambda^{2\varepsilon}. \notag
\end{gather}

\medskip
\underline{First point}: 
denote $u_+(t)=\sum_{j\in S_{+,1}}\|w_j^t \|^2$. We then have by balancedness and \cref{eq:ode_a_w}
\begin{align*}
\frac{1}{2}\frac{\dd u_+(t)}{\dd t} = \sum_{j\in S_{+,1}} \| w_j^t\|\langle D_j^{\theta^t}, w_j^t \rangle.
\end{align*}
Note that $\langle D_+^t, x_k\rangle \geq 0$ for any $t\in[t_{+1}, t_2]$. We thus have $\frac{\dd}{\dd t}\langle w_j^t, x_k\rangle\leq 0$ for $j\in S_{+,1}$ and $k$ such that $y_k<0$ during this phase. Thanks to the last point of \cref{lemma:phase_plus_1}, $\langle D_j^{\theta^t}, w_j^t \rangle= \langle D_+^t, w_j^t \rangle$ for any $j\in S_{+,1}$ and $t\in[t_{+1}, t_2]$. Using \cref{eq:decompw,eq:D+t}, this yields
\begin{align*}
\frac{1}{2}\frac{\dd u_+(t)}{\dd t} &= \sum_{j\in S_{+,1}} \| w_j^t\|^2  \langle D_+^t, \w_j^t \rangle \\
& \geq \sum_{j\in S_{+,1}} \| w_j^t\|^2 \langle \left( 1- \frac{u_+(t)}{n\|D_+\|} + h_+(t)\right) D_+ + h_{\bot}(t),\w_j^t\rangle \\
& \geq \sum_{j\in S_{+,1}} \| w_j^t\|^2  \left( 1- \frac{u_+(t)}{n\|D_+\|} + h_+(t)\right)\left(\|D_+\|-\lambda^{\frac{\varepsilon}{2}}\right) - \sum_{j\in S_{+,1}} \| w_j^t\|^2 \| h_{\bot}(t) \|\|d_j^t\|.
\end{align*}
So we have the following growth comparison
\begin{equation*}
\frac{u_+'(t)}{2} \geq \left(\|D_+\|-\lambda^{\frac{\varepsilon}{2}} \right)\left( 1- \frac{u_+(t)}{n\|D_+\|} + h_+(t)\right)u_+(t) - \| h_{\bot}(t) \| \max_{j\in S_{+,1}}\|d_j^t\|u_+(t).
\end{equation*}
By definition of the second phase, $u_+(t)\leq n\|D_+\|$ for any $t\in[t_{+1},t_2]$. We chose $\lambda$ small enough, so that $\sqrt{2\|D_+\|}\lambda^{\frac{\varepsilon}{4}}\geq 4mc^2\lambda^{2\varepsilon}$. This implies that $\| h_{\bot}(t) \| \lor |h_+(t)|\|D_+\|\leq 2\sqrt{2\|D_+\|}\lambda^{\frac{\varepsilon}{4}}$. So we finally have the following growth comparison during the second phase
\begin{align}
u_+'(t) \geq 2\left(\|D_+\|-\lambda^{\frac{\varepsilon}{2}} \right)\left( 1 - \left(2\sqrt{\frac{2}{\|D_+\|}}\right)\lambda^{\frac{\varepsilon}{4}}- \frac{u_+(t)}{n\|D_+\|}\right)u_+(t) - 8 \lambda^{\frac{\varepsilon}{2}}u_+(t)\notag\\
\label{eq:growth2}u_+'(t) \geq 2\left(\|D_+\|-\lambda^{\frac{\varepsilon}{2}} \right)\left( 1 - \left(2\sqrt{\frac{2}{\|D_+\|}}\right)\lambda^{\frac{\varepsilon}{4}}- \frac{4}{\|D_+\|-\lambda^{\frac{\varepsilon}{2}}} \lambda^{\frac{\varepsilon}{2}}- \frac{u_+(t)}{n\|D_+\|}\right)u_+(t).
\end{align}
Solution of the ODE $f'(t)=af(t)-bf(t)^2$ with $f(0)\in (0, \frac{a}{b})$ are of the form $f(t)=\frac{a}{b}\frac{e^{a(t-\tau)}}{1+e^{a(t-\tau)}}$. Note in the following
\begin{equation*}
\begin{cases} a(\lambda) = 2\left(\|D_+\|-\lambda^{\frac{\varepsilon}{2}}\right)\left(1-\left(2\sqrt{\frac{2}{\|D_+\|}}\right)\lambda^{\frac{\varepsilon}{4}}-\frac{4}{\|D_+\|-\lambda^{\frac{\varepsilon}{2}}}\lambda^{\frac{\varepsilon}{2}}\right),\\
b(\lambda) = 2\frac{\|D_+\| - \lambda^{\frac{\varepsilon}{2}}}{n\|D_+\|}.\end{cases}
\end{equation*}
By Gr\"onwall comparison, for any $t\in[t_{+1},t_2]$,
\begin{equation}\label{eq:normgrowth2}
u_+(t) \geq \frac{a(\lambda)}{b(\lambda)} \frac{e^{a(\lambda)(t-\tau)}}{1+e^{a(\lambda)(t-\tau)}} \quad \text{where} \quad u_+(t_{+1}) = \frac{a(\lambda)}{b(\lambda)} \frac{e^{a(\lambda)(t_{+1}-\tau)}}{1+e^{a(\lambda)(t_{+1}-\tau)}}.
\end{equation}
Thanks to \cref{lemma:minnorm}, $u_+(t_{+1})\geq \frac{\underline{c}^2}{4}\lambda^{2(1+\varepsilon)}$. This implies that 
\begin{equation*}
\tau \leq t_{+1} - \frac{2(1+\varepsilon)}{a(\lambda)}\ln(\lambda)- \frac{1}{a(\lambda)}\ln\left(\frac{\underline{c}^2b(\lambda)}{4a(\lambda)}\right),
\end{equation*}
and so
\begin{align*}
u_+(t_2) \geq \frac{a(\lambda)}{b(\lambda)}-\frac{a(\lambda)}{b(\lambda)}\exp\left(-a(\lambda)(t_2-t_{+1})- 2(1+\varepsilon)\ln(\lambda)-\ln\left(\frac{\underline{c}^2b(\lambda)}{a(\lambda)}\right)\right).
\end{align*}
In particular, if $t_2-t_{+1}>-\frac{1+2\varepsilon}{\|D_+\|}\ln(\lambda)$, then
\begin{align*}
u_+(t_2)> \frac{a(\lambda)}{b(\lambda)}-\frac{a(\lambda)^2}{\underline{c}^2b(\lambda)^2}\lambda^{\frac{a(\lambda)}{\|D_+\|}(1+2\varepsilon)-2(1+\varepsilon)}.
\end{align*}
Note that $\frac{a(\lambda)}{b(\lambda)}=n\|D_+\|-\bigO{\lambda^{\frac{\varepsilon}{4}}}$ and $\lim_{\lambda\to 0}\frac{a(\lambda)}{\|D_+\|}(1+2\varepsilon)-2(1+\varepsilon)=2\varepsilon$. As a consequence, for $\lambda$ small enough, we have $u_+(t_2)> n\|D_+\|-\lambda^{\frac{\varepsilon}{5}}$ if $t_2-t_{+1}>-\frac{1+2\varepsilon}{\|D_+\|}\ln(\lambda)$. This would break the second condition in the definition of the second phase \cref{eq:t2}, so that $t_2-t_{+1}\leq-\frac{1+2\varepsilon}{\|D_+\|}\ln(\lambda)$ and thanks to \cref{lemma:phase_plus_1}, $t_2\leq-\frac{1+3\varepsilon}{\|D_+\|}\ln(\lambda)$.

\medskip

\underline{Second point}: let $j\in\llbracket m \rrbracket\setminus S_{+,1}$. Similarly to the proof of \cref{lemma:phase_moins_1}, we can show if $\s_j=-1$ during the second phase that
\begin{equation*}
\frac{\dd \| w_j^t\|}{\dd t}\leq \left(\|D_-\| +4mc^2\lambda^{2\varepsilon}\right)\|w_j^t\|.
\end{equation*}
If $\s_j=1$ instead, there is some $k_j$ such that $y_{k_j}>0$ and $\langle w_j^t , x_{k_j} \rangle < 0$ thanks to \cref{lemma:sectors} and the continuous initialisation.
In that case we have the following inequalities
\begin{align*}
\frac{\dd \| w_j^t\|}{\dd t}&\leq -\frac{1}{n}\sum_{k \mid y_k>0}(h_{\theta^t}(x_k)-y_k)\langle w_j^t, x_k \rangle_+ \\
&\leq \frac{1}{n}\sum_{k\mid y_k>0}y_k\langle w_j^t, x_k \rangle_+ + 4mc^2\lambda^{2\varepsilon}\|w_j^t\|\\
&\leq  \left(\frac{1}{n}\sqrt{\sum_{k\neq k_j \mid y_k>0}y_k^2}+4mc^2\lambda^{2\varepsilon}\right)\|w_j^t\|\\
&\leq \left((1-\alpha)\|D_+\|+4mc^2\lambda^{2\varepsilon}\right)\|w_j^t\|,
\end{align*}
where we recall $\alpha=\frac{\min_{k\mid y_k>0}y_k^2}{2\|D_+\|^2}>0$. The previous inequalities are also valid during the first phase. In any case, for any $j\not\in S_{+,1}$ and $t\leq t_2$:
\begin{equation*}
\frac{\dd \| w_j^t\|}{\dd t}\leq \left( \max\left((1-\alpha)\|D_+\|,\|D_-\|\right)+4mc^2\lambda^{2\varepsilon} \right)\|w_j^t\|.
\end{equation*}
Note that we chose $\varepsilon$ small enough in \cref{sec:notations}, so that $(1+3\varepsilon)\frac{\max\left((1-\alpha)\|D_+\|,\|D_-\|\right)}{\|D_+\|}\leq 1-\varepsilon$. Since $t_2\leq-\frac{1+3\varepsilon}{\|D_+\|}\ln(\lambda)$, Gr\"onwall inequality yields for any $j\not\in S_{+,1}$ and $\lambda$ small enough
\begin{align*}
\|w_j^{t_2}\|&\leq \|w_j^{0}\|e^{-(1-\varepsilon)\ln(\lambda)}e^{-(1+3\varepsilon)\frac{4mc^2\lambda^{2\varepsilon} }{\|D_+\|}\ln(\lambda)}\\
&< 2c\lambda^{\varepsilon}.
\end{align*}

\medskip 

\underline{Third point}: let $j\in S_{+,1}$. Recall that we have for any $t\in[t_{+1},t_2]$
\begin{align*}
\frac{\dd \langle\w_j^t, D_+\rangle}{\dd t} & = \langle D_j^{\theta_t}, D_+\rangle - \langle \w_j^t, D_j^{\theta_t}\rangle \langle \w_j^t, D_+\rangle \\
&=\langle D_+^{t}, D_+\rangle - \langle \w_j^t, D_+^{t}\rangle \langle \w_j^t, D_+\rangle.
\end{align*}
Thanks to \cref{eq:decompw,eq:D+t}
\begin{align}
\frac{\dd \langle\w_j^t, D_+\rangle}{\dd t} & \geq \left(1-\frac{u_+(t)}{n\|D_+\|} + h_+(t) \right) \left( \|D_+\|^2 -\langle \w_j^t, D_+ \rangle^2 \right) - \| D_+ \| \|d_j^t\| \|h_{\bot}(t)\|.\label{eq:anglegrowth2}
\end{align}
Let us now define
\begin{gather*}
t_{u} = \inf\left\{t\geq t_{+1} \mid u_+(t) \geq \frac{n\|D_+\|}{7} \text{ or } h(t)\geq 5\|D_+\|\lambda^{\varepsilon} \right\}.
\end{gather*}
For any $t\in[t_{+1}, t_u]$, we have for $\lambda$ small enough
\begin{align*}
\frac{\dd \langle\w_j^t, D_+\rangle}{\dd t} & \geq \frac{5}{7}\left( \|D_+\|^2 -\langle \w_j^t, D_+ \rangle^2 \right) - \frac{20}{7} \|D_+\|^2\lambda^{\varepsilon} .
\end{align*}
The positive solution of the ODE $f'(t)=a^2-b^2f(t)^2$ is either increasing if $f(0)\leq \frac{a}{b}$ or remains larger than $\frac{a}{b}$. The following inequality thus implies by Gr\"onwall comparison for any $t\in[t_{+1}, t_u]$ and $\lambda$ small enough
\begin{align*}
\langle\w_j^t, D_+\rangle &\geq \min\left(\|D_+\|-4\|D_+\|\lambda^{\varepsilon}, \langle\w_j^{t_{+1}}, D_+\rangle \right)\\
 &\geq \|D_+\|-4\|D_+\|\lambda^{\varepsilon}.
\end{align*}
We thus have $h(t_u)<5\|D_+\|\lambda^{\varepsilon}$. So $u_+(t_u) = \frac{n\|D_+\|}{7}$.
Let us now bound $t_2-t_u$. Similarly to \cref{eq:normgrowth2}, we actually have for any $t\in[t_u, t_2]$
\begin{equation*}
u_+(t) \geq \frac{a(\lambda)}{b(\lambda)} \frac{e^{a(\lambda)(t-\tau_u)}}{1+e^{a(\lambda)(t-\tau_u)}} \quad \text{where} \quad u_+(t_u) = \frac{a(\lambda)}{b(\lambda)} \frac{e^{a(\lambda)(t_u-\tau_u)}}{1+e^{a(\lambda)(t_u-\tau_u)}}.
\end{equation*}
We showed $u_+(t_u)=\frac{n\|D_+\|}{7}$ and so $\tau_u\leq t_u-\frac{1}{a(\lambda)}\ln(\frac{b(\lambda)}{a(\lambda)}\frac{n\|D_+\|}{7})$, i.e. for any $t\in[t_u, t_2]$:
\begin{equation*}
u_+(t) \geq \frac{a(\lambda)}{b(\lambda)}-\frac{7a(\lambda)^2}{b(\lambda)^2n\|D_+\|}e^{-a(\lambda)(t-t_u)}.
\end{equation*}
Similarly to the proof of the first point, we can then show that for $\lambda$ small enough, $t_{2}-t_u \leq -\frac{\varepsilon}{5\|D_+\|}\ln(\lambda)$.
Now note that for any $t\in[t_u,t_2]$, \cref{eq:anglegrowth2} yields:
\begin{equation*}
\frac{\dd \langle\w_j^t, D_+\rangle}{\dd t} \geq -2\|D_+\|h(t) -4\sqrt{2}\|D_+\|mc^2\lambda^{2\varepsilon}.
\end{equation*}
This directly leads to
\begin{equation*}
h'(t) \leq 2\|D_+\|h(t) +4\sqrt{2}\|D_+\|mc^2\lambda^{2\varepsilon}.
\end{equation*}
By Gr\"onwall's comparison, we thus have for $\lambda$ small enough
\begin{align*}
h(t_2) &\leq \left( h(t_u) +2\sqrt{2}mc^2\lambda^{2\varepsilon} \right) e^{2\|D_+\|(t_2-t_u)}\\
&\leq 6\|D_+\|\lambda^{\varepsilon}e^{-\frac{2\varepsilon}{5}\ln(\lambda)} < \lambda^{\frac{\varepsilon}{2}}.
\end{align*}

\end{proof}

\subsection{Phase 3: proof of Lemma~\ref{lemma:phase3informal}}
\label{app:sec_phase3}
Similarly to the second phase with the positive labels, the third phase aims at fitting the negative labels. During this third phase, the norm of the neurons in $S_{-,1}$ (which are aligned with $-D_-$) grows until perfectly fitting all the negative labels of the training points; while all the other neurons do not change significantly.
We define the ending time of the second phase as
\begin{equation}\label{eq:t3}
\begin{aligned}
t_3 = \inf \Bigg\{ t\geq t_{-1} \mid & \phantom{\text{or }} \exists j \not\in S_{+,1}\cup S_{-,1}, \|w_j^t \| \geq 3c\lambda^{\varepsilon}\\
& \text{or } \sum_{j\in S_{-,1}} \|w_j^t \|^2 \geq n\|D_- \| - \lambda^{\frac{\varepsilon}{29}}\\
& \text{or } \exists j\in S_{-,1}, \langle \w_j^t,-D_-\rangle \leq \|D_-\| - \lambda^{\frac{\varepsilon}{14}} \\
& \text{or } \left( t\geq t_2 \text{ and } \|D_+^t\| \geq \lambda^{\frac{\varepsilon}{14}} \right) \Bigg\}.
\end{aligned}
\end{equation}
The following lemma corresponds to \cref{lemma:phase3informal} of the main text.
\begin{lem}\label{lemma:phase3}
If $\lambda\leq \lambda_*$,then the following inequalities hold
\begin{enumerate}[itemsep=-0.5em, topsep=-0.1em, label=(\roman*)]
\item $t_3 \leq -\frac{1+3r\varepsilon}{\|D_- \|}\ln(\lambda),$
\item $\forall j \not\in S_{+,1}\cup S_{-,1}, \|w_j^{t_3} \| < 3c\lambda^{\varepsilon},$
\item $\forall j\in S_{-,1}, \langle \w_j^{t_3},-D_-\rangle >  \|D_-\| - \lambda^{\frac{\varepsilon}{14}},$
\item $\|D_+^{t_3}\| < \lambda^{\frac{\varepsilon}{14}}$.
\end{enumerate}
\end{lem}

The first point of the lemma states that the third phase also lasts a time of order $-\ln(\lambda)$. It actually ends after the second one ($t_3>t_2$), since the neurons in $S_{-,1}$ do not grow in norm during the second one.
The last point states that the positive labels remain fitted during the third phase. As a consequence, this also means that the neurons of $S_{+,1}$ do not change a lot after $t_2$.
The other points imply that the second condition in \cref{eq:t3} holds at $t_3$, meaning that the norm of the neurons in $S_{-,1}$ have grown enough to fit the negative labels.

\begin{proof}
\underline{Preliminaries}: the three first points of this proof share similarities with the proof of \cref{lemma:phase2}. For conciseness and clarity, parts of the proof are shortened, as they follow the same lines as the proof of \cref{lemma:phase2}.
Similarly to the proof of the second phase, we define $h(t)=\|D_-\|-\min_{j\in S_{-,1}}\langle w_j^t,-D_-\rangle$ and we have for any $j\in S_{-,1}$ the decomposition
\begin{equation}
    \w_j^t = -\frac{D_-}{\|D_-\|} + d_j^t \quad \text{with} \quad \|d_j^t\|^2 \leq \frac{2h(t)}{\|D_-\|}.
\end{equation}
We have for any $k$ such that $y_k<0$
\begin{align}
    h_{\theta^t}(x_k) & \geq \sum_{j \in S_{-,1}} -\|w_j^t\| \langle w_j^t, x_k\rangle + \sum_{\substack{j \not\in  S_{-,1}\cup S_{+,1}\\ \langle w_j^t, x_k\rangle >0}} \s_j \|w_j^t\| \langle w_j^t, x_k\rangle \notag\\
    & \geq \sum_{j \in S_{-,1}} \frac{\|w_j^t\|^2}{\|D_-\|} \langle D_-, x_k\rangle + h_k(t)\label{eq:decompo1}= \sum_{j \in S_{-,1}} \frac{\|w_j^t\|^2}{n\|D_-\|} y_k + h_k(t),
\end{align}
where
\begin{equation*}
|h_k(t)| \leq \sum_{j \in S_{-,1}} \|w_j^t\|^2|\langle d_j^t, x_k \rangle| + 9mc^2\lambda^{2\varepsilon}.
\end{equation*}
Since we chose $\lambda$ small enough, this implies that $h_{\theta^t}(x_k) \geq y_k$ for $k$ such that $y_k <0$ and $t\in[t_{-1}, t_3]$. Moreover, thanks to \cref{lemma:phase_plus_1}, $\langle w_j^{t_{+1}}, x_k\rangle=0$ for any such $k$ and $j\in S_{+,1}$. Because of this, we have for any $[t_{-1},t_3]$
\begin{equation*}
    \langle w_j^{t}, x_k\rangle\leq 0  \qquad \text{for all } j\in S_{+,1} \text{ and } k \text{ such that } y_k<0.
\end{equation*}
\cref{eq:decompo1} is thus an equality on $[t_{-1}, t_3]$. Similarly to the second phase for $D_+^t$, we can now decompose $D_-^t$ as
\begin{align*}
    D_-^t & = D_- -\frac{\sum_{j \in S_{-,1}} \|w_j^t\|^2}{n\|D_-\|}\sum_{k, y_k<0} \langle D_-, x_k\rangle x_k  -\frac{1}{n}\sum_{k, y_k<0}h_k(t)x_k\\
    & = \left( 1 - \frac{\sum_{j \in S_{-,1}} \|w_j^t\|^2}{n\|D_-\|} \right)D_-  -\frac{1}{n}\sum_{k, y_k<0}h_k(t)x_k.
\end{align*}
And then for any $t\in[t_{-1},t_3]$:
\begin{equation}\label{eq:D-t}
    D_-^t = \left( 1 - \frac{\sum_{j \in S_{-,1}} \|w_j^t\|^2}{n\|D_-\|} + h_-(t)\right) D_- + h_{\bot}(t),
\end{equation}
where $\langle h_{\bot}(t), D_-\rangle = 0$ and $|h_-(t)|\|D_-\|\lor\|h_{\bot}(t)\|\leq \frac{1}{n}\sum_{j \in S_{-,1}} \|w_j^t\|^2 \sqrt{\frac{2h(t)}{\|D_-\|}}+9mc^2\lambda^{2\varepsilon}$.

\medskip

Using the previous decompositions, we also have the following inequalities for any $j\in S_{-,1}$
\begin{align*}
    \langle D_j^{\theta^t}, \w_j^t \rangle & = \langle D_-^t, \w_j^t \rangle + \sum_{\substack{k, y_k>0 \\ \langle \w_j^t, x_k \rangle > 0}} \langle D_+^t, x_k \rangle \langle d_j^t, x_k \rangle \\
    & = \langle D_-^t, \w_j^t \rangle + g_j(t),
\end{align*}
where $|g_j(t)| \leq \|d_j^t\|\|D_+^t\|$.

\underline{First point}: with the previous decompositions, we can now prove the first point of \cref{lemma:phase3}. Define $u_-(t) = \sum_{j \in S_{-,1}} \|w_j^t\|^2$. Based on the previous inequalities, we have for any $t\in [t_{-1},t_3]$
\begin{align*}
    \frac{1}{2}\frac{\dd u_-(t)}{\dd t} & = \sum_{j \in S_{-,1}} \|w_j^t \| \langle -D_j^{\theta^t}, w_j^t \rangle \\
    & = \sum_{j\in S_{-,1}} \|w_j^t\|^2\langle -D_-^t, \w_j^t \rangle-\|w_j^t\|^2g_j(t)\\
    &\geq u_-(t) \left( 1-\frac{u_-(t)}{n\|D_-\|}+h_-(t) \right)\left( \|D_-\| - \lambda^{\frac{\varepsilon}{14}}\right) - u_-(t)\sqrt{\frac{2h(t)}{\|D_-\|}}\left( \|h_{\bot}(t)\|+\|D_+^t\|\right).
\end{align*}
From there, we can show similarly to the proof of the second phase that $t_3\leq -\frac{1+3r\varepsilon}{\|D_-\|}\ln(\lambda)$.

\medskip

\underline{Second point}: first consider $j\not\in S_{+,1}\cup S_{-,1}$ such that $\s_j=1$. Thanks to \cref{lemma:phase2}, we already have $\|w_j^{t} \| < 2c\lambda^{\varepsilon}$ for any $t\leq t_2$. For $t\in[t_2,t_3]$, we then have
\begin{align*}
    \frac{\dd \| w_j^t \|}{\dd t} & \leq \|D_+^t\| \| w_j^t \|.
\end{align*}
Gr\"onwall inequality then gives for $\lambda$ small enough: $\| w_j^{t_3}\|\leq 2c\lambda^{\varepsilon} e^{-\lambda^{\frac{\varepsilon}{14}}\frac{1+3r\varepsilon}{\|D_-\|}\ln(\lambda)}< 3c\lambda^{\varepsilon}$.

Let now $j\not\in S_{+,1}\cup S_{-,1}$ such that $\s_j=-1$. By definition, there is some $k_j$ such that $y_{k_j}<0$ and $\langle w_j^t, x_{k_j}\rangle < 0$. Similarly to the proof of \cref{lemma:phase2}, we then have for any $t\in [0,t_3]$:
\begin{align*}
    \frac{\dd \|w_j^t \|}{\dd t} & = - \langle w_j^t, D_j^{\theta^t}\rangle \\
    & \leq \left( (1-\beta)\|D_-\| +\lambda^{\frac{\varepsilon}{15}}\right) \|w_j^t\|,
\end{align*}
where we recall $\beta = \frac{\min_{k,y_k<0}y_k^2}{2\|D_-\|^2}>0$. Note that we chose $\varepsilon$ small enough, so that $(1+3r\varepsilon)(1-\beta)\leq 1-\varepsilon$. Gr\"onwall inequality then yields for $\lambda$ small enough $
    \|w_j^{t_3}\| < 3c\lambda^{\varepsilon}$.

\medskip

\underline{Third point}: let $j\in S_{-,1}$. Recall that for any $t\in [t_{-1},t_3]$, $\langle \w_j^t, D_j^{\theta^t} \rangle  =\langle \w_j^t, D_-^{t} \rangle+g_j(t)$.  This yields for any $t\in [t_{-1},t_3]$
\begin{align*}
    \frac{\dd \langle \w_j^t, -D_-\rangle}{\dd t} & = \langle D_-^{t}, D_- \rangle - \langle \w_j^t, D_-^{t} \rangle \langle \w_j^t, D_- \rangle - g_j(t)\langle \w_j^t, D_- \rangle \\
    & \geq \left( 1 - \frac{u_-(t)}{n\|D_-\|} + h_-(t) \right) \left( \|D_-\|^2 - \langle \w_j^t, -D_- \rangle^2 \right) - \sqrt{2\|D_-\|h(t)} \|h_{\bot}(t)\| - \|D_-\||g_j(t)|.
\end{align*}
Let us now define
\begin{gather*}
t_{u} = \inf\left\{t\geq t_{-1} \mid u_-(t) \geq \frac{n\|D_-\|}{7} \text{ or } h(t)\geq 5\|D_-\|\lambda^{\frac{\varepsilon}{8}} \right\}.
\end{gather*}
Similarly to the third phase, we can show for $\lambda$ small enough the following sequence of properties
\begin{itemize}
    \item $h(t_u) < 5\|D_-\|\lambda^{\frac{\varepsilon}{8}}$
    \item $u_-(t_u) = \frac{n\|D_-\|}{7}$
    \item $t_3-t_u \leq -\frac{\varepsilon}{57\|D_-\|}\ln(\lambda)$
    \item $\langle \w_j^{t_3}, -D_- \rangle > \|D_-\| - \lambda^{\frac{\varepsilon}{14}}$.
\end{itemize}

\medskip

\underline{Fourth point}: 
define for this point \begin{equation*}
    \tau = \inf\left\{ t \geq t_{-1} \mid u_-(t) \geq n\lambda^{\frac{\varepsilon}{5}} \right\} \wedge t_3.
\end{equation*}
Thanks to \cref{lemma:phase2}, $\tau\geq t_2$. For any $k$ such that $y_k>0$, we have
\begin{align*}
    \frac{\dd h_{\theta^t}(x_k)}{\dd t} & = \sum_{j, \langle w_j^t, x_k \rangle >0} \|w_j^t\|^2 \langle D_j^{\theta^t}, x_k \rangle + \langle w_j^t, D_j^{\theta^t} \rangle \langle w_j^t, x_k \rangle.
\end{align*}
And so
\begin{align*}
    \frac{\dd D_+^t}{\dd t} & = -\frac{1}{n}\sum_{k,y_k>0} \frac{\dd h_{\theta^t}(x_k)}{\dd t} x_k\\
    & = -\frac{1}{n}\sum_{k,y_k>0}\sum_{j, \langle w_j^t, x_k \rangle >0} \|w_j^t\|^2 \langle D_+^{t}, x_k \rangle x_k + \langle w_j^t, D_j^{\theta^t} \rangle \langle w_j^t, x_k \rangle x_k.
\end{align*}
Recall the for any $j\in S_{+,1}$,  $\langle w_j^t, x_k\rangle=0$ for any $k$ such that $y_k<0$. From there, we have $\langle w_j^t, D_j^{\theta^t}\rangle = \langle w_j^t, D_+^{t}\rangle$. This leads to the following inequalities
\begin{align}\label{eq:dD+t}
    \frac{1}{2}\frac{\dd \|D_+^t\|^2}{\dd t} & = -\frac{1}{n}\sum_{k,y_k>0}\sum_{j, \langle w_j^t, x_k \rangle >0} \|w_j^t\|^2 \langle D_+^{t}, x_k \rangle^2 + \langle w_j^t, D_j^{\theta^t} \rangle \langle w_j^t, x_k \rangle \langle D_+^t, x_k\rangle \\
    & \leq -\frac{\sum_{j\in S_{+,1}} \|w_j^t\|^2}{n}\sum_{k,y_k>0} \langle D_+^{t}, x_k \rangle^2 -\frac{1}{n}\sum_{j\in S_{+,1}}\langle w_j^t, D_+^t\rangle\sum_{k,y_k>0}  \langle w_j^t, x_k \rangle \langle D_+^t, x_k\rangle\notag\\
    &\phantom{=} -\frac{1}{n}\sum_{k\mid y_k>0}\sum_{\substack{j\not\in S_{+,1}\\\langle w_j^t, x_k \rangle >0}}\langle w_j^t, D_j^{\theta^t} \rangle \langle w_j^t, x_k \rangle \langle D_+^t, x_k\rangle\notag\\
        & \leq -\frac{\sum_{j\in S_{+,1}} \|w_j^t\|^2}{n}\sum_{k,y_k>0} \langle D_+^{t}, x_k \rangle^2 -\frac{1}{n}\sum_{j\in S_{+,1}}\langle w_j^t, D_+^t\rangle^2\notag\\
    &\phantom{=} -\frac{1}{n}\sum_{k\mid y_k>0}\sum_{\substack{j\not\in S_{+,1}\\\langle w_j^t, x_k \rangle >0}}\langle w_j^t, D_j^{\theta^t} \rangle \langle w_j^t, x_k \rangle \langle D_+^t, x_k\rangle.\notag
\end{align}

Given the bounds on $\|D_+^t\|$ and $t_3$, a simple Gr\"onwall argument implies that $\sum_{j\in S_{+,1}} \|w_j^t\|^2$ did not change significantly between $t_2$ and $t_3$. In particular for $t\in[t_{2},\tau]$ and $\lambda$ small enough:
\begin{equation*}
    \frac{\dd \|D_+^t\|}{\dd t} \leq - \left(\|D_+\|-\lambda^{\frac{\varepsilon}{15}}\right)  \|D_+^t\| + \left(\|D_-\|+2\lambda^{\frac{\varepsilon}{29}}\right)\lambda^{\frac{\varepsilon}{5}}.
\end{equation*}
Since $\|D_+^{t_2}\|\leq 2\|D_+\|\lambda^{\frac{\varepsilon}{5}}$ thanks to \cref{lemma:phase2}, the above inequality implies by Gr\"onwall comparison that $\|D_+^{t}\|\leq \left(1+2\|D_+\|\right)\lambda^{\frac{\varepsilon}{5}}$ for any $t\in[t_2, \tau]$ and $\lambda$ small enough.

For any $j\in S_{-,1}$, define 
\begin{equation*}
    \alpha_j(t) = \sqrt{\sum_{\substack{k \mid y_k>0 \\ \langle \w_j^t, x_k\rangle >0}}\langle \w_j^t, x_k\rangle^2}.
\end{equation*}
Since $\frac{\dd \langle \w_j^t, x_k \rangle}{\dd t} = \s_j\left(\langle D_+^t, x_k \rangle - \langle D_j^{\theta^t}, \w_j^t\rangle \langle \w_j^t, x_k \rangle\right) \iind{\langle \w_j^t, x_k \rangle>0}$, we have 
\begin{align*}
    \frac{1}{2}\frac{\dd \alpha_j(t)^2}{\dd t} &= -\sum_{\substack{k \mid y_k>0 \\ \langle \w_j^t, x_k\rangle >0}}\langle D_+^t, x_k \rangle\langle \w_j^t, x_k \rangle - \langle D_j^{\theta^t}, \w_j^t\rangle \langle \w_j^t, x_k \rangle^2 \\
    & \leq \|D_+^t\| \alpha_j(t) + \langle D_j^{\theta^t}, \w_j^t\rangle \alpha_j(t)^2.
\end{align*}
Note that for $\lambda$ small enough $\langle D_j^{\theta^t}, \w_j^t\rangle \leq -\frac{\|D_-\|}{2}$ for any $j\in S_{-,1}$ and $t\in[t_{-1}, \tau]$. We thus have for any $t\in[t_{2}, \tau]$
\begin{equation*}
    \frac{\dd \alpha_j(t)}{\dd t} \leq \left(1+2\|D_+\|\right)\lambda^{\frac{\varepsilon}{5}} - \frac{\|D_-\|}{2}\alpha_j(t).
\end{equation*}

Moreover, recall that $\langle D_+^t, x_k \rangle >0$ before $t_2$. Thanks to \cref{lemma:phase_moins_1}, this actually implies that $\alpha_j^{t_2}=0$ and by Gr\"onwall comparison: $\alpha_j^{\tau} \leq \left(4r+\frac{2}{\|D_-\|}\right)\lambda^{\frac{\varepsilon}{5}}$.

\medskip

Now note that $\langle D_j^{\theta^t}, \w_j^t\rangle \leq 0$ for any $t\in[\tau,t_3]$, which leads for any $t\in[\tau,t_3]$ to
\begin{equation*}
    \alpha_j(t) \leq \left(4r+\frac{2}{\|D_-\|}\right)\lambda^{\frac{\varepsilon}{5}} + \int_{\tau}^{t} \|D_+^s\|\dd s.
\end{equation*}
\cref{eq:dD+t} then becomes for any $t\in[\tau,t_3]$ and $\lambda$ small enough
\begin{align*}
    \frac{\dd \|D_+^t\|}{\dd t} &\leq -\left(\|D_+\|-\lambda^{\frac{\varepsilon}{15}}\right)\|D_+^t\| + \left(\|D_-\|+\lambda^{\frac{\varepsilon}{29}}\right)^2\max_{j\in S_{-,1}}\alpha_j(t) + 9mc^2 \lambda^{2\varepsilon}\\
    & \leq -\left(\|D_+\|-\lambda^{\frac{\varepsilon}{15}}\right)\|D_+^t\| + \left(\|D_-\|+\lambda^{\frac{\varepsilon}{29}}\right)^2 \int_{\tau}^{t} \|D_+^s\|\dd s+ 5\|D_-\|\left(1+\|D_+\|\right) \lambda^{\frac{\varepsilon}{5}}.
\end{align*}
Thanks to \cref{lemma:gronwall}, this implies for any $t\in[\tau,t_3]$ and $\lambda$ small enough
\begin{equation*}
    \|D_+^t\| \leq \left(2\|D_+\|+6\|D_-\|+\frac{6}{r}\right)\lambda^{\frac{\varepsilon}{5}}\left(1+e^{\frac{\left(\|D_-\|+\lambda^{\frac{\varepsilon}{29}}\right)^2}{\left(\|D_+\|-\lambda^{\frac{\varepsilon}{15}}\right)}(t-\tau)}\right).
\end{equation*}
Similarly to the first point, we can also show that $t_3-\tau\leq -\frac{\varepsilon}{8\|D_-\|}\ln(\lambda)$, which finally yields that $\|D_+^t\| \leq \lambda^{\frac{\varepsilon}{14}}$ on $[t_2, t_3]$ for $\lambda$ small enough.
\end{proof}

\subsection{Final phase: proof of Theorem~\ref{thm:main}}
\label{app:sec_phaseconvergence}

To prove~\cref{thm:main}, we need to first prove some auxiliary lemmas. \cref{lemma:theta*} shows that at time $t_3$ the neural network is in the vicinity of some identifiable (i.e. independent of $\lambda$) interpolator. \cref{lem:four_condditions_imply_almost_zero_loss,lem:PL,lem:sourav} allow to apply \citet{chatterjee2022convergence} convergence result when a local PL is satisfied. Finally, we restate the main theorem in \cref{thm:final} for the sake of clearness and prove it.

\begin{lem}\label{lemma:theta*}
For any $\lambda\leq \lambda^*$, there exists $\theta^* \in \mathrm{argmin}_{L(\theta)=0} \|\theta\|^2$ independent of $\lambda$ such that
\begin{equation*}
    \|\theta^{t_3} - \theta^*\| \leq \lambda^{\frac{\varepsilon}{31}},
\end{equation*}
where $t_3$ is defined in~\cref{eq:t3}.
\end{lem}

Lemmas~\ref{lemma:phase2} and \ref{lemma:phase3} directly imply \cref{lemma:theta*} for some $\theta^*_\lambda$ that depends on $\lambda$. Extra work is required to prove that this minimal norm interpolator does not actually depend on $\lambda$.

\begin{proof}
Lemmas~\ref{lemma:phase2} and \ref{lemma:phase3} imply the following properties
\begin{enumerate}[label=(\roman*)]
    \item for any $j\in S_{+,1}$, $\|\w_j^{t_3} -\frac{D_+}{\|D_+\|}\| \leq \lambda^{\frac{\varepsilon}{15}}$,
        \item $\left|\sum_{j\in S_{+,1}}\|w_j^t\|^2 - n\|D_+\|\right| \leq \lambda^{\frac{\varepsilon}{15}}$,
    \item for any $j\in S_{-,1}$, $\|\w_j^{t_3}+\frac{D_-}{\|D_-\|}\| \leq \sqrt{\frac{2}{\|D_-\|}}\lambda^{\frac{\varepsilon}{28}}$,
    \item $\left|\sum_{j\in S_{-,1}}\|w_j^t\|^2 - n\|D_-\|\right| \leq \lambda^{\frac{\varepsilon}{29}}$,
    \item for any $j\not\in S_{+,1}\cup S_{-,1}$, $\|w_j^t\| \leq 3c\lambda^{\varepsilon}$.
\end{enumerate}

The balancedness property along these 5 properties guarantee there exists some $\theta^*_\lambda \in \mathrm{argmin}_{L(\theta)=0} \|\theta\|^2$ such that $ \|\theta^{t_3} - \theta^*_\lambda\| \leq \lambda^{\frac{\varepsilon}{30}}$. To show that this $\theta^*_\lambda$ actually does not depend on $\lambda$, it remains to show that the norm of any individual neuron in $S_{+,1}\cup S_{-,1}$ is close to some constant (independent from $\lambda$).

Consider any pair $j,j' \in S_{+,1}$. We recall \cref{eq:ode_norm_angle}:
\begin{align*}
   \frac{\dd \rho_j^t}{\dd t} = \langle D_j^{\theta^t}, \w_j^t \rangle \quad \text{and} \quad  \frac{\dd \w_j^t}{\dd t} = D_j^{\theta^t} - \langle D_j^{\theta^t}, \w_j^t \rangle \w_j^t.
\end{align*}
We note in the following $\tilde{\rho}_j^t, \tilde{\w}_j^t$ any solutions of the ODEs
\begin{equation}\label{eq:tildeODE}
\begin{aligned}
       \frac{\dd \tilde{\rho}_j^t}{\dd t} = \langle \tilde{D}_j^t, \tilde{\w}_j^t \rangle & \quad \text{with} \quad \tilde{\rho}_j^0=\ln(\|w_j^0\|/\lambda)\\
   \frac{\dd \tilde{\w}_j^t}{\dd t} = \tilde{D}_j^{t} - \langle \tilde{D}_j^{t}, \tilde{\w}_j^t \rangle \tilde{\w}_j^t& \quad \text{with}\quad \tilde{\w}_j^0=\w_j^0,
\end{aligned}\end{equation}
where $\tilde{D}_j^t=\frac{1}{n}\sum_{k}y_k x_k \iind{\langle \tilde{\w}_j^t, x_k \rangle>0}$. Remark that the process  $(\tilde{\w}, \tilde{\rho})$ does not depend on $\lambda$ as $\|w_j^0\| = \lambda g_j$ with the $g_j$'s being standard Gaussian vectors. 
We first have
\begin{align*}
    \frac{1}{2}\frac{\dd \|\w_j^t - \tilde{\w}_j^t\|^2}{\dd t} & =\sum_{k} \left(\iind{\langle \tilde{\w}_j^t, x_k \rangle>0} - \iind{\langle \w_j^t, x_k \rangle>0}\right)\frac{y_k}{n} \langle \tilde{\w}_j^t - \w_j^t, x_k \rangle + \sum_{k}\frac{y_k+h_{\theta^t}(x_k)}{n}\langle \w_j^t, x_k \rangle_+ \langle \tilde{\w}_j^t - \w_j^t, \w_j^t \rangle \\ & \phantom{=}- \sum_k \frac{y_k}{n}\langle \tilde{\w}_j^t, x_k \rangle_+ \langle \tilde{\w}_j^t - \w_j^t, \tilde{\w}_j^t \rangle + \sum_{k} \iind{\langle \w_j^t, x_k \rangle>0}\frac{h_{\theta^t}(x_k)}{n} \langle \tilde{\w}_j^t - \w_j^t, x_k \rangle.\\
\end{align*}
For $y_k>0$, both $\iind{\langle \tilde{\w}_j^t, x_k \rangle>0}$ and $\iind{\langle \w_j^t, x_k \rangle>0}$ remain positive until $t_{+1}$. As a consequence, the first sum is non-positive, which leads to
\begin{gather*}
       \frac{1}{2}\frac{\dd \|\w_j^t - \tilde{\w}_j^t\|^2}{\dd t}  \leq  \frac{2}{n}\sqrt{\sum_{k} y_k^2} \|\w_j^t - \tilde{\w}_j^t\|^2 + \frac{2}{n}\sqrt{\sum_{k} h_{\theta^t}(x_k)^2}  \|\w_j^t - \tilde{\w}_j^t\|,\\
       \shortintertext{and then}\\
       \frac{\dd \|\w_j^t - \tilde{\w}_j^t\|}{\dd t}  \leq  \frac{2}{n}\sqrt{\sum_{k} y_k^2} \|\w_j^t - \tilde{\w}_j^t\| + \frac{2}{n}\sqrt{\sum_{k} h_{\theta^t}(x_k)^2}.
\end{gather*}
Since $|h_{\theta^t}(x_k)|\leq 4mc^2 \lambda^{2(1-\varepsilon)}$ during the first phase, Gr\"onwall lemma implies that
\begin{equation*}
    \forall t \in [0, t_{+1}], \|\w_j^t - \tilde{\w}_j^t\| \leq \frac{4mc^2 }{\sqrt{n\left(\|D_+^2 + D_-\|^2\right)}}\lambda^{2(1-\varepsilon)}\left(e^{2\sqrt{\|D_+^2 + D_-\|^2}t}-1\right).
\end{equation*}
We thus have $\|\w_j^{t_{+1}} - \tilde{\w}_j^{t_{+1}}\| \leq \frac{4mc^2 }{\sqrt{n\left(\|D_+^2 + D_-\|^2\right)}}\lambda^{2-2(1+\sqrt{2})\varepsilon}$.
From there, note that we also have
\begin{align*}
    \frac{\dd \tilde{\rho}_j^t-\rho_j^t}{\dd t} & = \sum_{k} \frac{y_k}{n}\left( \langle \tilde{\w}_j^t, x_k \rangle_+ - \langle \w_j^t, x_k \rangle_+ \right) -\sum_{k} \frac{h_{\theta^t}(x_k)}{n}\langle \w_j^t, x_k \rangle_+\\
    & \leq \sqrt{\|D_+\|^2+\|D_-\|^2}\|\tilde{\w}_j^t-\w_j^t\| + \frac{1}{n}\sqrt{\sum_k h_{\theta^t}(x_k)^2}.
\end{align*}
As a consequence:
\begin{equation*}
    \left|\tilde{\rho}_j^{t_{+1}}-\rho_j^{t_{+1}}-\ln(\lambda)\right| \leq -\frac{8mc^2\varepsilon}{\|D_+\|}\ln(\lambda)\lambda^{2-2(1+\sqrt{2})\varepsilon}.
\end{equation*}
Moreover, for any $t\in[t_{+1}, t_3]$, we have $\frac{\dd \rho_j^t - \rho_{j'}^t}{\dd t} \leq -2\sqrt{\frac{2}{\|D_+\|}}\|D_+^t\|\frac{3\varepsilon}{\|D_+\|}\ln(\lambda) \lambda^{\frac{\varepsilon}{14}}$, and so for $\lambda$ small enough:
\begin{equation}\label{eq:normratio}
    |\rho_j^{t_3} - \rho_{j'}^{t_3}| \leq |\tilde{\rho}_j^{t_{+1}} - \tilde{\rho}_{j'}^{t_{+1}}| + 2\lambda^{\frac{\varepsilon}{15}}.
\end{equation}
The quantity $|\tilde{\rho}_j^{t_{+1}} - \tilde{\rho}_{j'}^{t_{+1}}|$ only depends on $\lambda$ because of the $t_{+1}$ term. The $\tilde{\rho}$ are indeed independent of $\lambda$.

\medskip

Similarly to the proof of \cref{lemma:phase_plus_1}, we can show that after some time $t_j$, $\tilde{D}_j^t = D_+$ for all $t \geq t_j$. From there, \cref{eq:tildeODE} imply for any $t \geq t_j$
\begin{align*}
    \frac{\dd \tilde{\rho}_j^t}{\dd t} &= \langle D_+, \tilde{\w}_j^t \rangle \\
   \frac{\dd \langle\tilde{\w}_j^t, D_+ \rangle}{\dd t} &= \|D_+\|^2 - \langle D_+, \tilde{\w}_j^t \rangle^2.
\end{align*}
Solving these ODEs, we thus have for some constants $\tau_j, c_j$ and any $t\geq t_j$:
\begin{align*}
     \langle\tilde{\w}_j^t, D_+ \rangle &= \|D_+\|\tanh\left(\|D_+\|(t-\tau_j)\right)\\
     \tilde{\rho}_j^t &= \ln\left(\cosh\left(\|D_+\|(t-\tau_j)\right)\right)+c_j.
\end{align*}
For $\lambda$ small enough, $t_{+1}\geq t_j \vee t_{j'}$ and then:
\begin{align*}
\tilde{\rho}_j^{t_{+1}} - \tilde{\rho}_{j'}^{t_{+1}} & = c_j - c_{j'} + \|D_+\|\left(\tau_{j'}-\tau_j \right) + \ln\left( \frac{1+e^{2\tau_j}\lambda^{2\varepsilon}}{1+e^{2\tau_{j'}}\lambda^{2\varepsilon}}\right)\\
& = d_{j,j'} + h_{j,j'}(\lambda),
\end{align*}
where $d_{j,j'}=c_j - c_{j'} + \|D_+\|\left(\tau_{j'}-\tau_j \right)$ and $|h_{j,j'}(\lambda)|\leq e^{2 (\tau_j\vee \tau_{j'})}\lambda^{2\varepsilon}$. Using \cref{eq:normratio}, this leads for $\lambda$ small enough to
\begin{align*}
    \|w_j^{t_3}\| & = \|w_{j'}^{t_3}\| e^{\rho_j^t - \rho_{j'}^t} \\
    & = \|w_{j'}^{t_3}\| e^{d_{j,j'}} \left( 1 + g_{j,j'}(\lambda)\right),
\end{align*}
where $|g_{j,j'}(\lambda)|\leq 4\lambda^{\frac{\varepsilon}{15}}$. As the sum of the norms of the $w_j$ is known, this actually fixes the norm of each individual neuron. Precisely we have for $|f(\lambda)|\leq \lambda^{\frac{\varepsilon}{15}}$
\begin{align*}
    n\|D_+\| - f(\lambda) & = \sum_{i\in S_{+,1}} \|w_{i}^{t_3}\|^2 \\
    & = \|w_j^{t_3}\|^2 \sum_{i\in S_{+,1}} e^{2d_{i,j}}\left( 1+2g_{i,j} + g_{i,j}^2\right).
\end{align*}
And so we finally have for any $j\in S_{+,1}$, $\|w_j^{t_3}\|=\sqrt{\frac{n\|D_+\|}{\sum_{i\in S_{+,1}}e^{2d_{i,j}}}}+\bigO{\lambda^{\frac{\varepsilon}{15}}}$.
Similarly, we can show that for any $j\in S_{-,1}$, $\|w_j^{t_3}\|=\sqrt{\frac{n\|D_-\|}{\sum_{i\in S_{-,1}}e^{2d_{i,j}}}}+\bigO{\lambda^{\frac{\varepsilon}{30}}}$.
We can now define $\theta^*$ as follows:
\begin{itemize}
    \item $w_j^* = \sqrt{\frac{n\|D_+\|}{\sum_{i\in S_{+,1}}e^{2d_{i,j}}}}\frac{D_+}{\|D_+\|}$ and $a_j^* = \|w_j^*\|$ if $j \in S_{+,1}$
        \item $w_j^* = -\sqrt{\frac{n\|D_-\|}{\sum_{i\in S_{-,1}}e^{2d_{i,j}}}}\frac{D_-}{\|D_-\|}$ and $a_j^* = -\|w_j^*\|$ if $j \in S_{-,1}$
        \item $w_j^*=0$ and $a_j^*=0$ if $j\not\in S_{+,1}\cup S_{-,1}$.
\end{itemize}
By construction, $\theta^*$ does not depend on $\lambda$ and $\|\theta^{t_3}-\theta^*\|\leq \lambda^{\frac{\varepsilon}{31}}$ for $\lambda$ small enough. Moreover, thanks to \cref{prop:KKT}, $\theta^* \in \mathrm{argmin}_{L(\theta)=0} \|\theta\|^2$.
\end{proof}

For this final phase, we know that at time $t_3$, given~\cref{lemma:phase3}, the neural net is arrived at a point $\theta^{t_3}$ satisfying the following: for $\varepsilon' = \frac{\varepsilon}{30}$,
\begin{enumerate}[label=(\roman*)]
\item For all $t>0$, $\theta^t \in \Theta:=\{\theta = (a_j, w_j)_{j \leq m}, \textrm{ such that } \forall j \in \llbracket m \rrbracket,\   |a_j|^2 = \|w_j\|^2 \}$.
\item For all $\displaystyle j \in \llbracket m \rrbracket \setminus \left(S_{+,1} \cup S_{-,1}\right), \|w_j^{t_3}\| \leq \lambda^{\varepsilon'}$.  
\item We have $ \left| \frac{1}{n} \sum_{j \in S_{+,1}} \|w_j^{t_3}\|^2 - \|D_+\| \right| \leq \lambda^{\epsilon'} $ and $ \left| \frac{1}{n} \sum_{j \in S_{-,1}} \|w_j^{t_3}\|^2 - \|D_-\| \right| \leq \lambda^{\epsilon'} $.
\item For all $j \in S_{+,1}$, $\left|\langle \w_j^{t_3},D_+\rangle - \|D_+\|\right| \leq \lambda^{\epsilon'}$ and $j \in S_{-,1}$, $\left|\langle \w_j^{t_3},D_-\rangle - \|D_-\|\right| \leq \lambda^{\epsilon'}$.
\end{enumerate}
Let us fist show an auxiliary Lemma that states that when these four conditions are satisfied, then the loss is almost zero.
\begin{lem}
\label{lem:four_condditions_imply_almost_zero_loss}
For $\theta$ such that conditions (i), (ii), (iii) and (iv) are satisfied then, for $\lambda$ small enough,
\begin{equation}
L(\theta)\leq \lambda^{\varepsilon'/2}.
\end{equation}
\end{lem}
\begin{proof}
Assume that $\theta$ satisfies conditions (i), (ii), (iii) and (iv). Then, for $k$ such that $y_k>0$, 
\begin{align*}
\left|h_\theta(x_k) - y_k\right| &= \left| \sum_{j \mid \langle w_j , x_k\rangle > 0} \s_j \|w_j\| \langle w_j , x_k\rangle - y_k  \right|  \\ 
&\leq \left| \sum_{j \in S_{+,1}} \s_j \|w_j\| \langle w_j , x_k\rangle - y_k  \right| + \left| \sum_{j \notin S_{+,1} \cup S_{1-}} \s_j \|w_j\| \langle w_j , x_k\rangle  \right| \\ 
&\leq \left| \sum_{j \in S_{+,1}} \|w_j\|^2 \langle \mathsf{w}_j , x_k\rangle - y_k  \right| + \sum_{j \notin S_{+,1} \cup S_{-,1}} \|w_j\|^2.
\end{align*} 
Using the (ii) property, on the one hand, the second term is upper bounded by $m \lambda^{2\varepsilon'}$, and, on the other hand, as $\|\mathsf{w}_j - \frac{D_+}{\|D_+\|}\|^2 \leq 2 \frac{\lambda^{\varepsilon'}}{\|D_+\|}$, we have by adding and subtracting $\frac{D_+}{\|D_+\|}$ in the inner product of the first term
\begin{align*}
\left|h_\theta(x_k) - y_k\right| &\leq \left| \sum_{j \in S_{+,1}} \frac{\|w_j\|^2}{\|D_+\|} \langle D_+ , x_k\rangle - y_k  \right| + \sum_{j \in S_{+,1}} \|w_j\|^2 \sqrt{2 \frac{\lambda^{\varepsilon'}}{\|D_+\|}}  + m \lambda^{2\varepsilon'} \\ 
&\leq \left| \frac{1}{n}\sum_{j \in S_{+,1}} \frac{\|w_j\|^2}{\|D_+\|} - 1  \right| y_k + \sum_{j \in S_{+,1}} \|w_j\|^2 \sqrt{2 \frac{\lambda^{\varepsilon'}}{\|D_+\|}}  + m \lambda^{2\varepsilon'} \\
&\leq \frac{\lambda^{\varepsilon'}}{\|D_+\|} y_k + 2 n \sqrt{\|D_+\|} \lambda^{\varepsilon'/2} +   \frac{n \lambda^{3\varepsilon'/2}}{\sqrt{\|D_+\|}}  + m \lambda^{2\varepsilon'} \\
&\leq 2 \lambda^{\varepsilon'/4},
\end{align*} 
for $\lambda$ small enough. And the same goes similarly for $y_k<0$. Hence, 
\begin{equation*}
L(\theta) = \frac{1}{2n} \sum_k \left(h_\theta(x_k) - y_k\right)^2 \leq \lambda^{\varepsilon'/2}.
\end{equation*}
This concludes the proof of the lemma.
\end{proof}
We show a local estimate of the PL inequality when balancedness, i.e. (i), is assumed.
\begin{lem}
\label{lem:PL}
For all $\theta \in \Theta$, we have
\begin{align}
\label{eq:PL}
\|\nabla L (\theta)\|^2 \geq 2 L(\theta) \cdot \min \left\{ \frac{1}{n} \sum_{j \in S_{-,1}} \|w_j\|^2, \frac{1}{n} \sum_{j \in S_{+,1}} \|w_j\|^2\right\}.
\end{align}
\end{lem}
\begin{proof}
Indeed, thanks to  the balancedness property we have the following calculation
\begin{align*}
\|\nabla L (\theta)\|^2 &= \sum_{j=1}^m \langle D_j, w_j\rangle^2 + \sum_{j=1}^m \|D_j\|^2 \|w_j\|^2 \\
&\geq \sum_{j \in S_{+,1}}^m \|D_j\|^2 \|w_j\|^2 + \sum_{j \in S_{-,1}}^m \|D_j\|^2 \|w_j\|^2 \\ 
&\geq n \min_{j \in S_{+,1}}\|D_j\|^2\cdot \frac{1}{n}\sum_{j \in S_{+,1}}^m  \|w_j\|^2 + n \min_{j \in S_{-,1}}\|D_j\|^2\cdot \frac{1}{n}\sum_{j \in S_{-,1}}^m  \|w_j\|^2 \\ 
&\geq \left( n \min_{j \in S_{+,1}}\|D_j\|^2 + n \min_{j \in S_{-,1}}\|D_j\|^2 \right) \cdot \min \left\{ \frac{1}{n} \sum_{j \in S_{-,1}} \|w_j\|^2, \frac{1}{n} \sum_{j \in S_{+,1}} \|w_j\|^2\right\}.
\end{align*}
Furthermore for all $j \in S_{+,1}$,
\begin{equation*}
n \|D_j\|^2 = \frac{1}{n} \sum_{k | \langle w_j, x_k \rangle > 0} \left( h_\theta(x_k) - y_k\right)^2 \geq \frac{1}{n} \sum_{k | y_k > 0} \left( h_\theta(x_k) - y_k\right)^2, 
\end{equation*}
where the last inequality is implied by the definition of the set $S_{+,1}$. As the same goes for $S_{-,1}$ replacing the sum over the positive $(y_k)$'s by the negative ones, we have that 
\begin{align*}
\left( n \min_{j \in S_{+,1}}\|D_j\|^2 + n \min_{j \in S_{-,1}}\|D_j\|^2 \right) &\geq \frac{1}{n} \sum_{k | y_k > 0} \left( h_\theta(x_k) - y_k\right)^2 + \frac{1}{n} \sum_{k | y_k < 0} \left( h_\theta(x_k) - y_k\right)^2 \\
&= 2 L(\theta).
\end{align*}
This concludes the proof of the claimed inequality.
\end{proof}
Second we show that on a neighbourhood of $\theta^{t_3}$ intersected with $\Theta$, the local PL constant can be lower bounded, where we recall $\Theta$ is the set of balanced parameters.  
\begin{lem}
\label{lem:sourav}
For $\lambda$ small enough, we have the following lower bound on the PL constant
\begin{equation}
\label{eq:PL_sourav}
\inf_{\theta \in B(\theta^{t_3},\ \lambda^{\frac{\varepsilon'}{8}}) \cap \Theta} \frac{\|\nabla L (\theta)\|^2}{L(\theta)} \geq \min \left\{ \|D_+\|, \|D_-\|\right\}.
\end{equation}
\end{lem}
\begin{proof}
Indeed, fix any $r>0$ and take $\theta \in B(\theta^{t_3}, r) \cap \Theta$. Let us denote by $(w_j^{t_3})_j$ the components of the hidden layer of $\theta^{t_3}$. We have 
\begin{equation*}
\|w_j\|^2 = \|w_j - w_j^{t_3} + w_j^{t_3}\|^2 \geq \|w_j^{t_3}\|^2 - 2 \|w_j - w_j^{t_3}\|\|w_j^{t_3}\| \geq \|w_j^{t_3}\|^2 - 2 r \|w_j^{t_3}\|.
\end{equation*}
Furthermore, 
\begin{equation*}
\frac{1}{n}\sum_{j \in S_{+,1}}\|w^{t_3}_j\| \leq \frac{\sqrt{m}}{n} \sqrt{\sum_{j \in S_{+,1}}\|w^{t_3}_j\|^2} \leq \sqrt{\frac{m}{n}} \sqrt{\sum_{j \in S_{+,1}}\|w^{t_3}_j\|^2} \leq  \sqrt{\frac{m}{n}} \left(\sqrt{\|D_+\|} + \lambda^{\varepsilon'/2}\right).
\end{equation*}
Hence, 
\begin{align*}
\frac{1}{n}\sum_{j \in S_{+,1}}\|w_j\|^2 &\geq \frac{1}{n}\sum_{j \in S_{+,1}}\|w^{t_3}_j\|^2 - 2 r \frac{1}{n}\sum_{j \in S_{+,1}}\|w^{t_3}_j\| \\
&\geq \|D_+\| - \lambda^{\varepsilon'} - 2 r \sqrt{\frac{m}{n}} \left(\sqrt{\|D_+\|} + \lambda^{\varepsilon'/2}\right),
\end{align*}
and for $r = \lambda^{\varepsilon'/8}$ and $\lambda$ small enough such that $\lambda^{\varepsilon'} + 2 \lambda^{\varepsilon'/8} \sqrt{\frac{m}{n}} \left(\sqrt{\|D_+\|} + \lambda^{\varepsilon'/2}\right) \leq \|D_+\|/2 $, we have
\begin{equation*}
\frac{1}{n}\sum_{j \in S_{+,1}}\|w_j\|^2 \geq \|D_+\|/2,
\end{equation*}
and the exact same inequality stands for the sum over $j \in S_{-,1}$ with alignment vector $D_-$.
\end{proof}
Thanks to the lower bound given by Lemma~\ref{lem:sourav}, we can now conclude that the gradient flow will not go out $B(\theta^{t_3},\ \lambda^{\frac{\varepsilon'}{8}})$ and will converge exponentially fast to some $\theta^\infty_\lambda$. Then we can take the limit $\lambda \to 0$ to characterise in this low initialisation regime the limit of the gradient flow.
\begin{theo}\label{thm:final}
For $\lambda$ small enough, 
\begin{itemize}
\item The gradient flow $(\theta^t)_{t>0}$ converges to some $\theta^\infty_\lambda$ of zero training loss , i.e $L(\theta^\infty_\lambda)= 0$.
\item There exists $\displaystyle \theta^* $ such that we have the following limit:
\begin{equation}
\lim_{\lambda \to 0} \lim_{t \to \infty} \theta^t = \theta^* \in \underset{L(\theta) = 0}{\mathrm{argmin}} \ \|\theta\|^2,
\end{equation}
and in its last phase, the gradient flow $(\theta^t)_{t\geq t_3}$ stays in $B(\theta^*,\, \lambda^{\frac{\varepsilon}{240}})\, \cap\, \Theta$ for which the convergence is exponential.
\end{itemize}
\end{theo}
\begin{proof}

From Lemma~\ref{lem:sourav}, \cref{eq:PL_sourav}, as $\varepsilon' = \varepsilon / 30$ we have
\begin{equation*}
\inf_{\theta \in B(\theta^{t_3},\ \lambda^{\frac{\varepsilon}{240}}) \cap \Theta} \frac{\|\nabla L (\theta)\|^2}{L(\theta)} \geq \min \left\{ \|D_+\|, \|D_-\|\right\},
\end{equation*}
and for $\lambda$ small enough, $L(\theta^{t_3}) / \lambda^{\varepsilon/120} \leq \lambda^{\varepsilon/60} / \lambda^{\varepsilon/120} = \lambda^{\varepsilon/120} \leq \min \left\{ \|D_+\|, \|D_-\|\right\} $. Hence for $r = \lambda^{\varepsilon/240}$, 
\begin{equation*}
\inf_{\theta \in B(\theta^{t_3},\,r) \cap \Theta} \frac{\|\nabla L (\theta)\|^2}{L(\theta)} \geq \frac{L(\theta^{t_3})}{r^2}.
\end{equation*}
Then, Theorem 2.1 of \citet{chatterjee2022convergence} applies (at least a benign modification of it restricting the flow to $\Theta$) and this shows that the gradient flow $(\theta^t)_{t\geq {t_3}}$ stays in $B(\theta^{t_3},\ \lambda^{\frac{\varepsilon}{240}}) \cap \Theta$, and converges towards some $\theta^\infty_\lambda$ of zero loss at exponential speed. Furthermore, from \cref{lemma:theta*}, there exists $\theta^*$, independent of $\lambda$, and belonging to $\textrm{argmin}_{L(\theta) = 0} \|\theta\|^2$, such that $\theta^{t_3} \in B(\theta^*,\ \lambda^{\frac{\varepsilon}{31}}) \cap \Theta$. Hence, $(\theta^t)_{t\geq {t_3}} \in B(\theta^*,\ 2\lambda^{\frac{\varepsilon}{240}}) \cap \Theta$ and finally: $ \lim_{\lambda \to 0} \lim_{t \to \infty} \theta^t = \lim_{\lambda \to 0}  \theta^\infty_\lambda = \theta^*$.
\end{proof}

\subsection{Auxiliary Lemmas}

This section states the auxiliary \cref{lemma:gronwall} used in the proof of \cref{lemma:phase3}.

\begin{lem}\label{lemma:gronwall}
Suppose a non-negative function $f$ verifies for non-negative constants $a,b$ and $c$
\begin{equation*}
    \forall t \in \mathbb{R}_+, \qquad f'(t) \leq -af(t)+b\int_0^t f(s)\dd s + \frac{c}{a},
\end{equation*}
then $f$ is bounded as
\begin{equation*}
    \forall t \in \mathbb{R}_+,  \qquad f(t) \leq \left(f(0)+c \right)\left(1+e^{\frac{b}{a}t}\right).
\end{equation*}
\begin{proof}
For $g(t)=-af(t)+b\int_0^t f(s)\dd s + c$, we have
\begin{align*}
    g'(t) = -af(t)+bf(t) \geq -af'(t) \geq -ag(t).
\end{align*}
By Gr\"onwall lemma, we then have $g(t) \geq -af(0)$ on $\mathbb{R}_+$. It then follows
\begin{align*}
    f(t)& = -\frac{g(t)}{a}+\frac{b}{a}\int_0^t f(s)\dd s + \frac{c}{a}\\
    & \leq f(0)+\frac{c}{a} + \frac{b}{a}\int_0^t f(s) \dd s.
\end{align*}
For $F(t)=\int_0^t f(s) \dd s$, Gr\"onwall comparison yields
\begin{align*}
    F(t) & \leq \frac{a}{b}e^{\frac{b}{a}t}\left(1-e^{-\frac{b}{a}t} \right)\left(f(0)+\frac{c}{a}\right)  \\
    & \leq \left(f(0)+\frac{c}{a}\right)\frac{a}{b}e^{\frac{b}{a}t}.
\end{align*}
The previous inequality $f(t) \leq f(0) + \frac{c}{a} + \frac{b}{a}F(t)$ yields the lemma.
\end{proof}
\end{lem}

\section{On global solutions of the minimum norm problem}
\label{app:KKT}

In this section, we study the following optimisation problem:
\begin{equation}
\label{eq:minimum_norm}
\theta^* \in \underset{L(\theta) = 0} {\mathrm{argmin}}\|\theta\|_2^2.
\end{equation}
An important remark is that the optimisation problem in \cref{eq:minimum_norm} is not convex because of the constraint set. Hence, the KKT conditions stated below are only necessary conditions: there exist real numbers $(\lambda_k)_{k \in \llbracket n \rrbracket }$ such that
\begin{equation}
\label{eq:KKT_minimum_norm}
\theta^* = \sum_{k = 1}^n \lambda_k   \nabla_\theta h_{\theta^*}(x_k) \quad \textrm{and for all $k \in  \llbracket n \rrbracket$,} \quad  h_{\theta^*}(x_k) = y_k.
\end{equation}
In terms of $(a^*, W^*)$, it implies 
\begin{align}
\label{eq:KKT_a_minimum_norm}
a^*_j &= \sum_{\substack{\langle w^*_j,\, x_k\rangle  > 0}} \lambda_k \langle x_k, w^*_j\rangle \\
\label{eq:KKT_w_minimum_norm}
w^*_j &= \sum_{\substack{\langle w^*_j,\, x_k\rangle  > 0}} \lambda_k x_k\, a^*_j.
\end{align}
In the orthonormal case, it is possible to solve explicitly the non-convex optimisation problem defined in \eqref{eq:minimum_norm}. Recall the definition of balanced networks: $\Theta = \{ (a,W) \text{ such that } \forall j \in \llbracket m \rrbracket, |a_j| = \|w_j\| \}$. For $\theta \in \Theta$, this means that there exists $(\s_j)_{j\in \llbracket m\rrbracket} \in \{-1,1\}^m$ such that $a_j = \s_j \|w_j\|$. Let us call $\mathsf{S}^\theta_0 = \{j \in \llbracket m \rrbracket \, |\,  w_j = 0\}$,  $\mathsf{S}^\theta_+ = (\mathsf{S}^\theta_0)^c \cap \{j \in \llbracket m \rrbracket \, |\,  \s_j = +1\}$ and $\mathsf{S}^\theta_- = (\mathsf{S}^\theta_0)^c \cap \{j \in \llbracket m \rrbracket \, |\,  \s_j = -1\}$, where $(\mathsf{S}^\theta_0)^c := \llbracket m \rrbracket \setminus \mathsf{S}^\theta_0$. Note that the family $\mathsf{S}^\theta_0$, $\mathsf{S}^\theta_+$ and $\mathsf{S}^\theta_-$ form a partition of $\llbracket m \rrbracket$.

We have the following proposition.
\begin{prop}\label{prop:KKT}
The global minimum of the problem~\eqref{eq:minimum_norm} is reached at objective value $\displaystyle 2n(\|D_+\| + \|D_-\|)$. Moreover, we have the following description of the global minimisers of \cref{eq:minimum_norm}:
\begin{align*}
\label{eq:set_of_global_minima}
\underset{L(\theta) = 0} {\mathrm{argmin}}\,\|\theta\|_2^2 = \{ \theta \in \Theta \text{ such that } &\forall j\in \mathsf{S}^\theta_+, \langle w_j, {\color{white}-}D_+ \rangle = \|w_j\|\|D_+\| \text{ and } \sum_{j \in \mathsf{S}_+} \|w_j\|^2 =  n \|D_+\|\\
&\forall j\in \mathsf{S}^\theta_-, \langle w_j, -D_- \rangle = \|w_j\|\|D_-\| \text{ and } \sum_{j \in \mathsf{S}_-} \|w_j\|^2 =  n \|D_-\|\}.
\end{align*}

\end{prop}
\begin{proof}
First the constraint set is non-empty if $m\geq2$: one can consider the hidden weights defined as $W = (\sqrt{\frac{n}{\|D_+\|}}D_+, -\sqrt{\frac{n}{\|D_-\|}}D_-, 0, \hdots, 0)$ and the outputs $a = (\sqrt{n\|D_+\|}, -\sqrt{n\|D_-\|}, 0, \hdots, 0)$ so that for all $k \in \llbracket m \rrbracket$, $h_{(a,W)}(x_k) = \langle a, \sigma(W x_k)\rangle = y_k$. Hence the minimum is attained in the closed ball centred in the origin and of radius $\|(a, W)\|$. Moreover, the set $C:=\{\theta | \forall k \leq n,\  h_\theta (x_k) = y_k\} = \cap_{k\leq n} C_k$, where  each $C_k$ is a closed subset as it is the pre-image of $y_k$ by the continuous function: $\theta \mapsto h_{\theta}(x_k)$. Hence the minimum is to be found in the intersection of a compact and a close subset, that is, a compact set overall. By continuity of the norm to be minimised over this compact set, there exists a global minimum. 

Moreover, let $\theta^* = (a^*, W^*)$ be a global minimiser of the optimisation problem. Note that if there exists $j \in \llbracket m \rrbracket$ such that $|a^*_j|^2 \neq \|w^*_j\|^2$, then, as for $c = |a^*_j| / \|w^*_j\|$, $|a^*_j|^2/c + c\|w^*_j\|^2 <  |a^*_j|^2 + \|w^*_j\|^2$, without changing the constraint set, we have found a strictly better minimum. Hence, for all $j \in \llbracket m \rrbracket$ we have: $|a^*_j|^2 = \|w^*_j\|^2$. 

Suppose that $w^*_j \neq 0$ for all $j\in\llbracket m \rrbracket$. Otherwise the analysis can be restricted to the indices of non-zero $w^*_j$, without changing neither the objective nor the constraint set. 
Set $j \in \llbracket m \rrbracket$, we have that $w^*_j \in \mathrm{span}(x_1,\hdots,x_n)$, otherwise, this would simply add weights in the objective without changing $(h_\theta^*(x_k))_{k\leq n}$. Hence, $w^*_j = \sum_k \langle w^*_j, x_k \rangle x_k$. Then, defining the Lagrange multipliers $(\lambda_k)_{k\leq n}$ as in the KKT condition stated above, we have that, for all $k \leq n$
\begin{align*}
\textrm{If } &\langle w^*_j, x_k\rangle > 0, \textrm{ then } \lambda_k = \s_j \langle \frac{w^*_j}{\|w^*_j\|}, x_k \rangle, \\
\textrm{else if } &\langle w^*_j, x_k\rangle \leq 0, \textrm{ then } \langle w^*_j, x_k\rangle = 0.
\end{align*}
On the other side, for all $k \leq n$, $h_{\theta^*}(x_k) = y_k$, and thus
\begin{align*}
\sum_{j | \langle w^*_j,  x_k \rangle > 0} \s_j \|w^*_j\| \langle w^*_j, x_k \rangle = y_k \quad \Longleftrightarrow \quad \lambda_k \left(\sum_{j | \langle w^*_j,  x_k \rangle > 0}  \|w_j\|^2 \right) = y_k.
\end{align*}
From there, we deduce that $\lambda_k$ and $y_k$ have the same sign and define
\begin{align*}
A_k:=\{j \in \llbracket m \rrbracket \, |\,  \langle w^*_j, x_k\rangle > 0 \} .
\end{align*}
The previous observation implies that $A_k \subseteq \mathsf{S}^\theta_+$ if $y_k>0$ and $A_k \subseteq \mathsf{S}^\theta_-$ if $y_k<0$.

From there, we define an alternative network $\tilde{\theta}=(\tilde{a},\tilde{W})$ such that
\begin{gather*}
    \forall j>2, (\tilde{a}_j,\tilde{w}_j) =\mathbf{0}\\
        \tilde{w}_{1} = \frac{\sum_{j\in\mathsf{S}^\theta_+} \|w_j^*\|w_j^*}{\sqrt{\left\|\sum_{j\in\mathsf{S}^\theta_+} \|w_j^*\|w_j^*\ \right\|_2}} \quad \text{and} \quad  \tilde{w}_{2} = \frac{\sum_{j\in\mathsf{S}^\theta_-} \|w_j^*\|w_j^*}{\sqrt{\left\|\sum_{j\in\mathsf{S}^\theta_-} \|w_j^*\|w_j^*\ \right\|_2}},\\
        \tilde{a}_{1} = \|\tilde{w}_1\| \quad \text{and} \quad 
                          \tilde{a}_{2} = -\|\tilde{w}_2\|.
\end{gather*}
Note then than $h_{\theta^*}(x_k)=h_{\tilde{\theta}}(x_k)$ for any $k\in[n]$. Indeed if $y_k>0$,
\begin{align*}
    h_{\theta^*}(x_k) & = \sum_{j\in A_k} \|w_j^*\| \langle w_j^*, x_k\rangle \\
    & = \sum_{j\in\mathsf{S}^\theta_+} \|w_j^*\| \langle w_j^*, x_k\rangle \\
    & = \tilde{a}_1 \langle \tilde{w}_1, x_k \rangle_+\\
    & = h_{\tilde{\theta}}(x_k).    
\end{align*}
The second equality comes from the fact that $A_k\subseteq \mathsf{S}^\theta_+$, and that for any $j\in\mathsf{S}^\theta_+\setminus A_k$, $\langle w_j^*,x_k \rangle=0$. A similar analysis holds if $y_k<0$. As a consequence, $L(\tilde{\theta})=0$.

Moreover, the norm of $\tilde{\theta}$ is necessary smaller:
\begin{align*}
    \frac{1}{2}\| \tilde{\theta}\|_2^2 & =  \|\tilde{W}\|_2^2 \\
    & = \|\tilde{w}_1\|_2^2 + \|\tilde{w}_2\|_2^2 \\
    & = \left\|\sum_{j\in\mathsf{S}^\theta_+} \|w_j^*\|w_j^*\ \right\|_2 + \left\|\sum_{j\in\mathsf{S}^\theta_-} \|w_j^*\|w_j^*\ \right\|_2 \\
    &\leq \sum_{j\in\mathsf{S}^\theta_+} \|w_j^*\|^2 + \sum_{j\in\mathsf{S}^\theta_-} \|w_j^*\|^2\\
    & = \sum_{j\in[m]} \|w_j^*\|_2^2 \\
    & = \frac{1}{2}\| \theta^*\|_2^2.
\end{align*}
The inequality is here a consequence of the triangle inequality. Since $\theta^*$ is a global minimum of \cref{eq:minimum_norm}, so that it is an equality. The equality case of the triangle inequality then implies that all $w_j^*$ such that $j\in \mathsf{S}^\theta_+$ are aligned, i.e.,
\begin{gather*}
    \forall j,j' \in \mathsf{S}^\theta_+, \frac{w_j^*}{\|w_j^*\|}=\frac{w_{j'}^*}{\|w_{j'}^*\|}\\
        \forall j,j' \in \mathsf{S}^\theta_-, \frac{w_j^*}{\|w_j^*\|}=\frac{w_{j'}^*}{\|w_{j'}^*\|}.
\end{gather*}
As a consequence, by defining $D_1 = \sum_{j\in\mathsf{S}^\theta_+} \|w_j^*\| w_j^*$, we have for any $j^*\in \mathsf{S}^\theta_+$
\begin{equation*}
     \frac{w_j^*}{\|w_j^*\|}= \frac{D_1}{\|D_1\|}.
\end{equation*}
Moreover by interpolation and orthogonality, $D_1 = \sum_{k, y_k>0} y_k x_k$, i.e., $D_1= n D_+$. A similar analysis holds for $\mathsf{S}^\theta_-$, which yields the first inclusion:
\begin{align*}
\label{eq:set_of_global_minima}
\underset{L(\theta) = 0} {\mathrm{argmin}}\,\|\theta\|_2^2 \subseteq \{ \theta \in \Theta \text{ such that } &\forall j\in \mathsf{S}^\theta_+, \langle w_j, {\color{white}-}D_+ \rangle = \|w_j\|\|D_+\| \text{ and } \sum_{j \in \mathsf{S}_+} \|w_j\|^2 =  n \|D_+\|\\
&\forall j\in \mathsf{S}^\theta_-, \langle w_j, -D_- \rangle = \|w_j\|\|D_-\| \text{ and } \sum_{j \in \mathsf{S}_-} \|w_j\|^2 =  n \|D_-\|\}.
\end{align*}
The reverse inclusion is automatic, by checking the for such a $\theta$ we indeed have $L(\theta)=0$ and $\|\theta\|_2^2 = 2n(\|D_+\| + \|D_-\|)$.
\end{proof}

\section{On the existence of gradient flows}\label{app:subgradient_flows}

This section shows that a global solution of the ODE followed by the gradient flow only exists for the choice of subdifferential $\sigma'(0)=0$. Precisely, the gradient flow follows the following ODE
\begin{equation}\label{eq:inclusionODE}
    \frac{\dd \theta^t}{\dd t} \in -\partial L(\theta^t),
\end{equation}
where $\partial f$ is the Clarke subdifferential of $f$. The subdifferential of the loss is uniquely defined up to the choice of the subdifferential of the activation function at $0$. If we consistently choose a fixed value $\sigma_0 \in [0,1]$ for the latter, we then have
\begin{equation}\label{eq:subdiff}\begin{gathered}
    \frac{\dd \theta^t}{\dd t} = \Bigg(\left(\frac{1}{n}\sum_{k=1}^n (y_k-h_{\theta^t}(x_k)) \sigma(\langle w_j^t, x_k \rangle)\right)_j, \left(\frac{a_j^t}{n}\sum_{k=1}^n (y_k-h_{\theta^t}(x_k)) \sigma'(\langle w_j^t, x_k \rangle)x_k\right)_j \Bigg),\\
    \text{where} \qquad    \sigma'(z) = \begin{cases} 0 \text{ if } z<0,\\
    \sigma_0 \text{ if } z=0,\\
    1 \text{ if }z>0.
    \end{cases}\end{gathered}
\end{equation}
Note that an empirical study of the influence of $\sigma_0$ on the dynamics has been conducted by \citet{bertoin2021numerical}. We have the following proposition, justifying the choice $\sigma'(0)=0$.
\begin{prop}\label{prop:GFexistence}
The ODE~\eqref{eq:subdiff} admits (at least) one solution on $[0,\infty)$ if and only if $\sigma_0=0$.
\end{prop}
\begin{proof}
First of all note that $\langle w_j^0, x_k \rangle\neq 0$ for all $j$ and $k$. The Peano theorem then implies there exists a local solution of this ODE, i.e., there exists a time $t_0$ such that \cref{eq:subdiff} admits a continuous solution $\theta^t$ on $[0, t_0)$. Assume now that $t_0=\infty$. As long as $\langle w_j^0, x_k \rangle\neq 0$ for all $j$ and $k$, the analysis does not depend on the choice of $\sigma_0$. We can thus show similarly to the proof of the first phase that for some time $\tau$ and some $j,k$: $\langle w_j^{\tau}, x_k \rangle= 0$. Moreover, still following the lines of the first phase, we have some $\delta>0$ such that $|h_{\theta^t}(x_k)|<\frac{|y_k|}{2}$ on $[0, \tau+\delta]$. Recall that
\begin{equation*}
    \frac{\dd \langle w_j^t,x_k\rangle}{\dd t} = -\frac{y_k-h_{\theta^t}(x_k)}{n} a_j^t \sigma'(\langle w_j^t, x_k \rangle).
\end{equation*}

As a consequence, $\langle w_j^{t}, x_k \rangle$ is monotone on $[0,\tau+\delta]$, and thus decreasing. In particular,  $\langle w_j^{t}, x_k \rangle \leq 0$ on $[\tau, \tau+\delta]$. Moreover, note that $\langle w_j^{t}, x_k \rangle$ can not become (strictly) negative, as its derivative is $0$ as soon as it becomes negative. Indeed, if $ t_-\coloneqq\inf\{t \mid \langle w_j^{t}, x_k \rangle<0\}<\tau+\delta$, then the scalar product has a zero derivative on $(t_-, \tau+\delta)$ and is thus constant, equal to $0$ by continuity, on this interval. So we finally have $\langle w_j^{t}, x_k \rangle = 0$ on $[\tau, \tau+\delta]$ and $\theta^t$ is then a solution of \cref{eq:subdiff} (almost everywhere) only if $\sigma_0=0$.

\medskip

Conversely, let $\sigma_0=0$ now. Assume we have a maximal solution of the ODE in finite time, i.e., let $t_0 \in \R_+$ and $\theta_*^t$ defined on $[0,t_0)$ such that $\theta_*^t$ verifies \cref{eq:subdiff} almost everywhere on $[0,t_0)$ and there exists no $\delta>0$ and $\tilde{\theta}$ such that $\tilde{\theta}^0=\theta_*^0$ and $\tilde{\theta}^t$ verifies \cref{eq:subdiff} almost everywhere on $[0,t_0+\delta)$. 

By definition of the gradient flow, the training loss is decreasing. As a consequence, $|y_k-h_{\theta_*^t}(x_k)|$ is uniformly bounded in time. The ODE then leads for some constant $L>0$ to $\frac{\dd \| \theta_*^t\|}{\dd t} \leq L \| \theta_*^t\|$. By Gr\"onwall argument, $\|\theta_*^t\|$ is bounded on $[0,t_0)$. This implies that $\theta_*^t$ is Lipschitz with time on $[0,t_0)$. In particular, $\theta_*^t$ admits a limit in $t_0$. Now consider the alternative ODE
\begin{equation}\label{eq:altODE}
\begin{gathered}
    \frac{\dd \theta^t}{\dd t} = \Bigg(\left(\frac{1}{n}\sum_{k=1}^n (y_k-h_{\theta^t}(x_k)) \sigma(\langle w_j^t, x_k \rangle)\right)_j, \left(\frac{a_j^t}{n}\sum_{k=1}^n (y_k-h_{\theta^t}(x_k)) \iind{\langle w_{*,j}^{t_0}, x_k \rangle>0}x_k\right)_j \Bigg),\\
    \theta^{t_0} = \theta_*^{t_0}.
\end{gathered}\end{equation}
We replaced $\iind{\langle w_{j}^{t}, x_k \rangle>0}$ in the original ODE by $\iind{\langle w_{*,j}^{t_0}, x_k \rangle>0}$, which makes it Lipschitz in $\theta^t$. Cauchy-Lipschitz theorem then implies that \cref{eq:altODE} admits a (unique) local solution $\tilde{\theta}^t$ on $[t_0-\delta, t_0+\delta]$ for some $\delta>0$. Moreover, $\tilde{\theta^t}=0$ on the whole interval if $\langle w_{*,j}^{t_0}, x_k \rangle=0$. By continuity, we can thus choose $\delta>0$ small enough so that $\iind{\langle w_{*,j}^{t_0}, x_k \rangle>0} = \iind{\langle \tilde{w}_{j}^{t}, x_k \rangle>0}$ on $[t_0,t_0+\delta]$. $\theta_*$ can then be extended by $\tilde{\theta}$ on $[t_0, t_0+\delta]$ and still verifies \cref{eq:subdiff} on $(t_0, t_0+\delta]$, contradicting its maximality. As a consequence, there exists a solution of the ODE on $\R_+$. \end{proof}

\end{document}